\documentclass[fleqn,10pt]{article}
\usepackage{inputenc}
\usepackage[T1]{fontenc}
\usepackage[noblocks]{authblk}
\usepackage{geometry}
\usepackage{float}
\usepackage{booktabs}
\usepackage{makecell}
\usepackage{caption}
\usepackage{dsfont}
\RequirePackage{amsthm,amsmath,amsfonts,amssymb,mathrsfs}
\usepackage{mathtools}
\usepackage[square,sort,comma,numbers]{natbib}
\RequirePackage{graphicx}
\RequirePackage{extarrows}

\usepackage{enumerate}
\usepackage{amsmath}
\usepackage{extarrows}
\usepackage[noend]{algpseudocode}
\usepackage{algorithmicx,algorithm}
\usepackage{mathrsfs}

\newtheorem{theorem}{Theorem}[section]
\newtheorem{lemma}{Lemma}[section]

\newtheorem{ass}{Assumption}[section]

\newtheorem{col}{Corollary}[section]

\newcommand{\pr}{\mathrm{P}}
\newcommand{\cov}{\operatorname{cov}}
\newcommand{\E}{\mathrm{E}}

\newcommand{\X}{\mathbf{X}}
\newcommand{\x}{\mathbf{x}}
\newcommand{\vthe}{\boldsymbol{\theta}}
\newcommand{\vbeta}{\boldsymbol{\beta}}

\newcommand{\Y}{\mathbf{Y}}
\newcommand{\bm}[1]{\mathbf{#1}}
\newcommand{\y}{\mathbf{y}}
\newcommand{\e}{\mathrm{e}}
\newcommand{\var}{\operatorname{var}}
\newcommand{\fdr}{\operatorname{FDR}}
\newcommand{\fdp}{\operatorname{FDP}}

\makeatletter
\DeclareRobustCommand\widecheck[1]{{\mathpalette\@widecheck{#1}}}
\def\@widecheck#1#2{%
    \setbox\z@\hbox{\m@th$#1#2$}%
    \setbox\tw@\hbox{\m@th$#1%
       \widehat{%
          \vrule\@width\z@\@height\ht\z@
          \vrule\@height\z@\@width\wd\z@}$}%
    \dp\tw@-\ht\z@
    \@tempdima\ht\z@ \advance\@tempdima2\ht\tw@ \divide\@tempdima\thr@@
    \setbox\tw@\hbox{%
       \raise\@tempdima\hbox{\scalebox{1}[-1]{\lower\@tempdima\box
\tw@}}}%
    {\ooalign{\box\tw@ \cr \box\z@}}}
\makeatother

\usepackage[noend]{algpseudocode}
\usepackage{algorithmicx,algorithm}
\newcommand{\algorithmicinput}{{Input:}}

\newcommand{\INPUT}{\item[\algorithmicinput]}

\def\aconstskip{\par\hrule height 0pt \nobreak\vskip -\baselineskip
   \noindent \hbox to \hsize{\hfill \vrule width 0pt depth \baselineskip}}
\makeatletter

\title{
High-dimensional Inference and FDR Control for Simulated Markov Random Fields}

\author[1,+]{Haoyu Wei}
\affil[1]{Department of Economics, University of California San Diego, La Jolla, USA.}
\author[2,+]{Xiaoyu Lei}
\affil[2]{Department of Statistics, University of Wisconsin–Madison, Madison, USA.}
\author[3]{Yixin Han}
\affil[3]{Department of Mathematical and Statistical Sciences, University of Alberta, Canada.}
\author[4,5*]{Huiming Zhang}
\affil[4]{Institute of Artificial Intelligence, Beihang University, Beijing, China.}

\affil[5]{Zhuhai UM Science \& Technology Research Institute, Zhuhai, China.}
\affil[*]{zhanghuiming@buaa.edu.cn}
\affil[+]{these authors contributed equally to this work}

\begin{document}

\flushbottom
\maketitle
\begin{abstract}

     Identifying important features linked to a response variable is a fundamental task in various scientific domains. 
     This article explores statistical inference for simulated Markov random fields in high-dimensional settings. We introduce a methodology based on Markov Chain Monte Carlo Maximum Likelihood Estimation (MCMC-MLE) with Elastic-net regularization. Under mild conditions on the MCMC method, our penalized MCMC-MLE method achieves $\ell_{1}$-consistency. We propose a decorrelated score test, establishing both its asymptotic normality and that of a one-step estimator, along with the associated confidence interval. 
     Furthermore, we construct two false discovery rate control procedures via the asymptotic behaviors for both p-values and e-values. Comprehensive numerical simulations confirm the theoretical validity of the proposed methods.
    
\end{abstract}

%
%
\thispagestyle{empty}

\section{Introduction}\label{sec_ind}
\subsection{Backgrounds and Related Work}

The probabilistic graphical model is a framework comprising various probability distributions that decompose based on the architecture of an associated graph \cite{wainwright2008graphical}. This model adeptly encapsulates the intricate interdependencies among random variables, enabling the construction of extensive multivariate statistical models. These models have found extensive applications across diverse research domains. Notably, they are utilized in hierarchical Bayesian models \cite{gyftodimos2002hierarchical}, and in the analysis of contingency tables \cite{edwards1983analysis,wermuth1982graphical}, a key aspect of categorical data analysis \cite{agresti2003categorical,fienberg2000contingency}. Within these graphical models, undirected graphical models stand out, characterized by a probability distribution that factorizes based on functions defined over the graph's cliques. These undirected graphical models typically play a critical role in constraint satisfaction problems \cite{cook1971complexity,chandrasekaran2012complexity}, and have significant applications in language and speech processing \cite{blei2003latent,brown1968}. Their utility extends to image processing \cite{4767341,4767596,hassner1981use} and more broadly in the realm of spatial statistics \cite{besag1974spatial}, demonstrating their versatility and importance in contemporary research. These models are widely employed in diverse fields such as statistical physics \cite{ising1925beitrag}, natural language processing \cite{manning1999foundations}, image analysis \cite{woods1978markov}, and spatial statistics \cite{ripley1984spatial}. 

Our focus narrows specifically to undirected graphical models conceptualized as exponential families. This broad category of probability distributions has been extensively examined in statistical literature \cite{barndorff2014information,efron1978geometry}. The characteristics of exponential families forge insightful links between inference methodologies and convex analysis \cite{borwein2010convex,hiriart2013convex}. Notably, many renowned models are perceived as exponential families within undirected graphical models, including the Ising model \cite{ising1925beitrag,baxter2016exactly}, Gaussian random fields \cite{speed1986gaussian}, and latent Dirichlet allocation \cite{blei2003latent}. These models are all special instances of Markov Random Fields (MRFs) framed within exponential families \cite{rue2005gaussian}. Consequently, our paper will concentrate on this particular aspect of MRFs.

In the realm of graphical models, two crucial aspects are structure learning and parameter estimation \cite{wainwright2008graphical,jalali2011learning,miasojedow2018sparse,chen2023systematic}. The significance of Markov Random Fields (MRFs) with exponential families necessitates effective solutions tailored to this specific framework. A primary challenge in learning the structure of such models is the computationally daunting normalizing constant. In a scenario where a graphical model comprises $d$ vertices, with each variable at a vertex assuming values from a discrete state space of $r$ distinct elements, the normalizing constant escalates to a sum of $d^{r}$ terms, rendering computation impractical. To address this challenge, numerous methods have been proposed. A prominent approach is the pseudo-likelihood method \cite{besag1974spatial}, which substitutes the likelihood involving the normalizing constant with the product of conditional probabilities that do not include this constant. 
However, its effectiveness is contingent on the pseudo-likelihood being a close approximation to the actual likelihood, which typically occurs in graphs with simpler structures. Another notable strategy to tackle the intractable normalizing constant is the Markov Chain Monte Carlo (MCMC) method \cite{gilks2005m,robert2013monte}. Here, the normalizing constant is approximated through a path integral over a Markov chain. A significant advantage of the MCMC method is that its computational cost does not depend on the complexity of the graph being estimated. This implies that sufficiently large sample sizes can effectively approximate the normalizing constant, offering a viable solution for handling more complex graphical structures.

In the evolving landscape of graphical models, a notable trend is their increasing dimensionality, often resulting in high-dimensional settings \cite{meinshausen2006high,hastie2015statistical,kurtz2015sparse}. In such scenarios, where the number of parameters ($p$) exceeds the number of independent samples ($n$), the maximum likelihood estimation becomes problematic. To address this, the Lasso method, a $\ell_1$ penalty, is commonly applied \cite{tibshirani1996regression}, effectively constraining parameters to prevent overfitting. While Lasso can zero out certain parameters, ridge regression ensures that no parameter is entirely excluded. The Elastic-net method, a combination of Lasso ($\ell_1$-)penalty and ridge ($\ell_2$-)penalty \cite{zou2005regularization}, will be used in our study. This approach, often outperforming Lasso in practice, can select multiple correlated variables simultaneously, demonstrating a group effect and handling multicollinearity effectively. Another challenge in a high-dimensional graphical model is discerning the most influential factors, crucial in fields like genomics for identifying genes related to specific diseases. The false discovery rate (FDR), as proposed by \cite{benjamini1995controlling}, is preferred over the family-wise error rate (FWER) in these scenarios due to its suitability for large-scale hypothesis testing. The Benjamini-Hochberg (BH) procedure \cite{benjamini1995controlling} is commonly used for controlling the FDR, especially effective with asymptotically normal estimators in high dimensions. Additionally, \textit{e-values} \cite{vovk2021values} have emerged as a novel and mathematically convenient tool for multiple testing, gaining prominence in statistical analysis \cite{vovk2019true,cui2021directional}.

In high-dimensional graphical models, a significant hurdle is the estimation of the intractable normalizing constant from the Markov Chain Monte Carlo (MCMC) component. In low-dimensional cases, the normalizing constant is often linked to MCMC-MLE, the minimization of the MCMC approximation of the likelihood, as discussed in various studies \cite{geyer1992markov, geyer1994convergence, banerjee2008model}. However, the high-dimensional scenario introduces the complexity of numerous nuisance parameters, rendering traditional partial-likelihood-based inferences impractical. To address this, some researchers have explored penalized MCMC approximations of the likelihood. For example, \cite{miasojedow2018sparse} applied an $\ell_1$ penalty to the MCMC approximation for the Ising model, while \cite{geng2018stochastic} implemented an $\ell_1$ penalty for discrete MRFs. Additionally, \cite{bogdandiscussion} discussed the potential of this approach in Bayesian logic regression using arbitrary priors from MCMC-MLE. Despite these advancements, there remains a big gap in the literature regarding the analysis of general MRFs within the exponential family for high-dimensional settings.

\subsection{Our Contributions}
~~~\\
In our study, we bridge a significant gap by integrating $\ell_1$- and $\ell_2$-penalties in the MCMC-MLE for general MRFs within the exponential family. Our contributions are three-fold. First, we derive an oracle inequality for the MCMC-MLE under Elastic-net penalty, establishing the $\ell_{1}$-consistency of the estimator under conditions such as the compatibility condition \cite{van2008high,buhlmann2013statistical}; see Appendix \ref{app_ConIneq_MC}.  


To address the issue of high-dimensional nuisance parameters, we construct a decorrelated score function by projecting the score function of the parameter onto the space spanned by the nuisance parameters. This approach yields a corresponding decorrelated score test statistic \cite{ning2017general,shi2021statistical,chernozhukov2023high} and facilitates the proposal of an asymptotically unbiased one-step estimator. 
We establish asymptotic normality and confidence intervals in our work, aligning with high-dimensional inferential statistics. with a stronger assumption ${\log p}/{\sqrt{n}}= o(1)$, where $n$ denotes the sample size, and $p$ signifies the dimension of the model parameter. Achieving the desired asymptotic normality with the introduction of the Monte Carlo method and $m$ simulated samples necessitates the acceleration of its convergence. This requires the assumption $\frac{m}{n}\asymp\log p$, marking a theoretical breakthrough and a notable departure from current literature.

Finally, leveraging this one-step estimator, we propose algorithms for controlling the FDR in both e-value \cite{vovk2021values,2020False} and p-value frameworks, which enhances the applicability of our approach to high-dimensional MRFs. 



\subsection{Outline and Notations}

The paper is outlined as follows. 
The paper is outlined as follows.
Section \ref{sec-2} provides background: it gives an overview of graphical models, the MCMC-MLE method, Elastic-net penalized estimation for high-dimensional data, and concentration inequalities for Markov chains.
Section \ref{sec-4} constructs the decorrelated score function and corresponding test statistic, proving their asymptotic normality. This section also introduces an asymptotically unbiased one-step estimator used to formulate a confidence interval for a single parameter of interest.
Section \ref{sec-6} presents two distinct procedures for controlling the FDR, utilizing p-values and e-values. Finally, Section \ref{sec_sim} conducts simulations to validate the theoretical aspects discussed.

Throughout the paper, we define the following notations. Let $X_{n}$ be a series of random variables and $a_{n}$ a series of constants. The notation $X_{n}=o_{p}(a_{n})$ is used to denote that $\frac{X_{n}}{a_{n}}$ converges to zero in probability as $n$ tends to infinity. Conversely, $X_{n}=O_{p}(a_{n})$ indicates that the set of random variables $\{\frac{X_{n}}{a_{n}}\}_{n\ge 1}$ is stochastically bounded. For two series of constants $a_{n}$ and $b_{n}$, the notation $a_{n}\lesssim b_{n}$ implies the existence of a constant $c > 0$ such that $a_{n} \le c b_{n}$ for all $n$. Analogously, $a_{n}\gtrsim b_{n}$ indicates the lower bound. Let $\rightsquigarrow$ denote weak convergence. The notation $a_{n}\asymp b_{n}$ is used when both $a_{n} \lesssim b_{n}$ and $a_{n}\gtrsim b_{n}$ hold true.  

\section{Preliminaries}\label{sec-2}

\subsection{Graphical Models within Exponential Families}

Define graph model $G=(V, E)$, where $V=\{1,2,\cdots,d\}$ is a set of vertices and $E\subset V\times V$ is a set of edges. 
For each vertex $s\in V$, assign a random variable $X_{s}$
To establish a probability distribution associated with the graph. 
In our study, we focus on the finite discrete state space $\mathcal{X}_{s}=\{1,2,\cdots,r\}$. We denote the Cartesian product of the state spaces $\{\mathcal{X}_{s} : s\in V\}$—where the random vector $\X=(X_{s},s\in V)$ takes values—by $\prod_{s\in V}\mathcal{X}{s}$. The notation $x{s}$ represents a specific element in $\mathcal{X}_{s}$, and $\x=(x_{1},\cdots,x_{d})^{\top}$ signifies a specific value within the space $\prod_{s\in V}\mathcal{X}_{s}$. For any subset $A$ of the vertex set $V$, we define $\X_{A}:=(X_{s},s\in A)$ as the sub-vector of the random vector and $\x_{A}:=(x_{s},s\in A)$ as a specific element of the random sub-vector.
 
The undirected graphical models are factorized according to a function defined over the graph's cliques. A clique $C$ is identified as a fully connected subset of the vertex set, satisfying $(s,t)\in E$ for every $s,t\in C$. For each clique $C$, we introduce a compatibility function $\psi_{C}:\prod_{s\in C}\mathcal{X}_{s} \rightarrow \mathbb{R}_{+}$. Let $\mathcal{C}$ represent a set of the graph's cliques. The undirected graphical model also referred to as the Markov random field (MRF, \cite{kindermann1980markov}), is expressed as
\[
    p(\x):=\frac{1}{Z}\prod_{C\in \mathcal{C}}\psi_{C}(\x_{C}),\quad\text{where}\quad Z:=\sum_{\x\in \prod_{s\in V}\mathcal{X}_{s}}\prod_{C\in \mathcal{C}}\psi_{C}(\x_{C}),
\]
with $Z$ is the constant chosen to normalize the distribution. 
Concentrating on the canonical exponential family.  
the probability mass function for $(\mathcal{X}_{s},s\in V)$ is given by
\[ 
    p (\x \mid \vthe) = \frac{1}{C(\vthe)} \e^{\vthe^{\top} \varphi(\x)}, ~~\text{ where }~~ C(\vthe) = \sum_{\x \in \prod_{s \in V} \mathcal{X}_s} \e^{\vthe^{\top} \varphi(\x)},
\]
and $\vthe=(\theta_{1},\cdots,\theta_{p})^{\top}$ is the unknown parameters of interest. Denote ${\mathcal{X}}=\prod_{s\in V}\mathcal{X}_{s}$. 
Let $\nabla \log C(\vthe)$ and $\nabla^{2} \log C(\vthe)$ be the expectation and covariance of the random vector $\varphi(\X)$, respectively.
Specifically, by taking $\varphi_i (\x) = \varphi_{(j, k)}(\x) = x_j x_k$ where $x_j \in \{-1, +1\}$, the model evolves into the Ising model in \cite{miasojedow2018sparse}. Further, constraining $\mathcal{X}_s =\{0, 1 \}$ for any $s \in V$ transforms the model into the discrete MRFs, also discussed in \cite{geng2018stochastic}. Consequently, our model is considerably more general than those in the existing literature.

\subsection{Penalized Markov Chain Monte Carlo Likelihood }

Suppose $\x_{1},\cdots,\x_{n} \in \mathbb{R}^d$ are independent and identically distributed from distribution $p(\x \mid \vthe^*)$, where $\vthe^*$ is the true parameter. We construct the negative log-likelihood function as follows:
\begin{equation}\label{2-4}
    \mathcal{L}_{n}(\vthe;\x_{1},\cdots,\x_{n})=
    -\frac{1}{n}\sum_{i=1}^{n}\log(p(\x_{i} \mid \vthe))=
    -\frac{1}{n}\sum_{i=1}^{n}\vthe^{\top}\varphi(\x_{i})+\log(C(\vthe)).
\end{equation}
We then employ the MCMC method for approximation of the normalized constant $C(\vthe)$. Consider $\Y_{1},\cdots,\Y_{m}$ as a Markov chain with its stationary distribution on the product space $\bm{\mathcal{X}}$ being $h(\Y)$. By the ergodic property of the Markov chain, $C(\vthe)$ can be approximated through a path integral
\begin{equation}\label{2-5}
    C(\vthe)=\operatorname{E}_{\Y\sim h(\Y)}\frac{\e^{\vthe^{\top}\varphi(\Y)}}{h(\Y)}
    \approx \frac{1}{m}\sum_{i=1}^{m}\frac{\e^{\vthe^{\top}\varphi(\Y_{i})}}{h(\Y_{i})}.
\end{equation}
Combining \eqref{2-4} and \eqref{2-5}, the MCMC log-likelihood function $\mathcal{L}_{n}^{m}$ is
\begin{equation*}
    \mathcal{L}_{n}^{m}(\vthe;\x_{1},\cdots,\x_{n})=
    -\frac{1}{n}\sum_{i=1}^{n}\vthe^{\top}\varphi(\x_{i})+\log\left(\frac{1}{m}\sum_{i=1}^{m}\frac{\e^{\vthe^{\top}\varphi(\Y_{i})}}{h(\Y_{i})}\right).
\end{equation*}
Denote
$w_{i}(\vthe):=\frac{\e^{\vthe^{\top}\varphi(\Y_{i})}}{h(\Y_{i})} $ and $\overline{\varphi}(\vthe):=\frac{\sum_{i=1}^{m}w_{i}(\vthe)\varphi(\Y_{i})}{\sum_{i=1}^{m}w_{i}(\vthe)}.$ Then the gradient $\nabla \mathcal{L}_{n}^{m}$ and the Hessian $\nabla^{2} \mathcal{L}_{n}^{m}$ can be re-expressed as
\begin{equation}\label{2-19}
    \nabla \mathcal{L}_{n}^{m}(\vthe)=-\frac{1}{n}\sum_{i=1}^{n}\varphi(\x_{i}) + \frac{\sum_{i=1}^{m}w_{i}(\vthe)\varphi(\Y_{i})}{\sum_{i=1}^{m}w_{i}(\vthe)}=-\frac{1}{n}\sum_{i=1}^{n}\varphi(\x_{i})+\overline{\varphi}(\vthe),
\end{equation}
\begin{equation}\label{2-20}
    \nabla^{2} \mathcal{L}_{n}^{m}(\vthe)=\frac{\sum_{i=1}^{m}w_{i}(\vthe)\varphi(\Y_{i})^{\otimes 2}}{\sum_{i=1}^{m}w_{i}(\vthe)}-\left(\frac{\sum_{i=1}^{m}w_{i}(\vthe)\varphi(\Y_{i})}{\sum_{i=1}^{m}w_{i}(\vthe)}\right)^{\otimes 2}=\frac{\sum_{i=1}^{m}w_{i}(\vthe)(\varphi(\Y_{i})-\overline{\varphi}(\vthe))^{\otimes 2}}{\sum_{i=1}^{m}w_{i}(\vthe)}.
\end{equation}
It is clear that $\nabla^{2} \mathcal{L}_{n}^{m}(\vthe)$ is convex. Denote $H^{*} := \cov_{\vthe^{*}}(\varphi(\bm{X}))$ as the covariance matrix of the function $\varphi(\X)$ under the true value $\vthe^*$.

In high-dimensional cases where the number of parameters significantly exceeds the number of samples ($p \gg n$), the minimum of $\mathcal{L}_{n}^{m}(\vthe)$ is not unique, resulting in an ill-posed minimizer of MCMC log-likelihood. To address this issue, we consider deploying a penalty to the likelihood function in this article. We employ the Elastic-net penalty, which combines the benefits of both the Lasso method and ridge regression. The Elastic-net-penalized MCMC-MLE we propose is formulated as
\begin{equation}\label{2-8}
    \widehat{\vthe} := \widehat{\vthe}_{n}^{m}(\lambda_{1},\lambda_{2})=\mathop{\arg\min}_{\vthe \in \mathbb{R}^d} \mathcal{L}_{n}^{m}(\vthe)+\lambda_{1}\left\|\vthe\right\|_{1}+\lambda_{2}\left\|\vthe\right\|_{2}^{2},
\end{equation}
where $\lambda_{1},\lambda_{2} > 0$ are two tuning parameters. The terms $\left\|\vthe\right\|_{1}:=\sum_{i=1}^{p}|\theta_{i}|$ and $\left\|\vthe\right\|_{2}:=(\sum_{i=1}^{p}\theta_{i}^{2})^{\frac{1}{2}}$ represent the $\ell_{1}$-norm and $\ell_{2}$-norm of the vector $\vthe$, respectively. 

Under certain regularity conditions, we establish the non-asymptotic oracle inequality for $\widehat{\vthe}$ in Theorem~\ref{thm-1} in Appendix and it leads to convergence rate for  $\widehat{\vthe}$ that can be served for deriving decorrelated score test and constructing confidence interval for the component of $\widehat{\vthe}$; see Theorem \ref{5-2}.

\section{Decorrelated Score Test and Confidence Interval}\label{sec-4}


In high-dimensional inference, testing a specific element $\alpha^*$ within the true vector $\vthe^*$ presents some challenges. Unlike low-dimensional scenarios, the presence of a high-dimensional nuisance parameter set $\vthe^* \setminus \{ \alpha^* \}$ complicates the development of valid inferential methods. Traditional approaches like partial-likelihood-based inference \cite{ning2017general,fang2017testing,shi2021statistical}, become impractical in this context due to the intractability of the limiting distribution and the complexity introduced by numerous nuisance parameters.
Motivated from \cite{fang2017testing}, we propose a decorrelated score test to address the hypothesis testing of $H_{0}:\alpha^{}=\alpha_{0}$ versus $H_{1}:\alpha^{}\neq\alpha_{0}$. This approach is analogous to the traditional score test but modified to suit high-dimensional contexts. In such settings, a direct extension of the profile partial score test becomes infeasible due to the intractable limiting distribution caused by the large number of nuisance parameters. We address this challenge by adopting a decorrelated method inspired by a projection technique, which helps to mitigate the influence of these nuisance parameters.


Without loss of generality, consider $\vthe$ as $(\alpha,\vbeta^{\top})^{\top}$, with $\alpha$ being the first scalar component. The gradients and Hessian are partitioned as $ \nabla \mathcal{L}_{n}^{m}=(\nabla_{\alpha} \mathcal{L}_{n}^{m},(\nabla_{\vbeta}\mathcal{L}_{n}^{m})^{\top})^{\top} $ and $\nabla^{2} \mathcal{L}_{n}^{m}= \left(
    \begin{array}{cc}
  \nabla^{2}_{\alpha\alpha} \mathcal{L}_{n}^{m}   & (\nabla^{2}_{\alpha\vbeta} \mathcal{L}_{n}^{m})^{\top} \\
   \nabla^{2}_{\alpha\vbeta} \mathcal{L}_{n}^{m}  & \nabla^{2}_{\vbeta\vbeta} \mathcal{L}_{n}^{m}
\end{array}
\right)$. Similarly, we denote the blocks of $H^{*} = \cov_{\vthe^{*}} \big(\varphi(\x) \big)$ as
$
H^{*}=
\left(
\begin{array}{cc}
  H^{*}_{\alpha\alpha}   & (H^{*}_{\alpha\vbeta})^{\top} \\
    H^{*}_{\alpha\vbeta}  & H^{*}_{\vbeta\vbeta}
\end{array}
\right)$. To test the hypothesis $H_{0}:\alpha^{*}=\alpha_{0}$, we first estimate the nuisance parameters $\vbeta$ by $\widehat{\vbeta}$ using the Elastic-net-penalized MCMC-MLE defined in \eqref{2-8}. Then, we approximate $\nabla_{\alpha}\mathcal{L}_{n}^m(\vthe^{})$ using the partial score function $\nabla_{\vbeta}\mathcal{L}_{n}^m (\vthe^{*})$ in terms of expectation:
\[
    \bm{w}^{*}=\mathop{\arg\min}_{\bm{w}\in\mathbb{R}^{p-1}}\operatorname{E} \big[\nabla_{\alpha}\mathcal{L}_{n}(\vthe^{*})-\bm{w}^{\top}\nabla_{\vbeta}\mathcal{L}_{n}(\vthe^{*}) \big]^{2} = (H^{*}_{\vbeta\vbeta})^{-1}H^{*}_{\alpha\vbeta},
\]
where $\bm{w}^{\top}\nabla_{\vbeta}\mathcal{L}_{n}^m(\vthe^{*})$ is the projection of $\nabla_{\alpha}\mathcal{L}_{n}^m(\vthe^{*})$ onto the linear space spanned by $\nabla_{\vbeta}\mathcal{L}_{n}^m(\vthe^{*})$ elements. In high-dimension case, direct computation of $\bm{w}^{*}$ from sample data is problematic, so we estimate $\bm{w}^{*}$ using a lasso-type estimator $\widehat{\bm{w}}$:
\begin{equation*}
    \widehat{\bm{w}}=\mathop{\arg\min}_{\bm{w}\in\mathbb{R}^{p-1}}\frac{1}{2}\bm{w}^{\top}\nabla^{2}_{{\vbeta\vbeta}}\mathcal{L}_{n}^{m}(\widehat{\vthe})\bm{w}-\bm{w}^{\top}\nabla^{2}_{\alpha\vbeta}\mathcal{L}_{n}^{m}(\widehat{\vthe})+\lambda^{\prime}\|\bm{w}\|_{1},
\end{equation*}
where $\lambda^{\prime} > 0$ is another tuning parameter. Building upon this, we propose a decorrelated score function:
\begin{equation}\label{4-7}
    \widehat{U}(\alpha,\widehat{\vbeta}):= \nabla_{\alpha}\mathcal{L}_{n}^{m}(\alpha,\widehat{\vbeta}) - \widehat{\bm{w}}^{\top}\nabla_{\vbeta}\mathcal{L}_{n}^{m}(\alpha,\widehat{\vbeta}).
\end{equation}
Intuitively, this decorrelated score function removes the effects of the high-dimensional nuisance parameters \cite{ning2017general}.
Then to establish the limiting distribution of $\widehat{U}(\alpha,\widehat{\vbeta})$ under the null hypothesis $H_{0}:\alpha^{*}=\alpha_{0}$, the following assumptions regarding the properties of the covariance matrix $H^{*}$ and the weight vector $\bm{w}^{*}$ are required.
\begin{ass}\label{4-5}
    The eigenvalues of the covariance matrix $H^{*}$ are lower and upper bounded: $\lambda_{\min} \le \lambda_{\min}(H^{}) \le \lambda_{\max}(H^{*}) \le \lambda_{\max}$ for some $0 < \lambda_{\min} < \lambda_{\max}$.
\end{ass}
\begin{ass}\label{4-6}
    The $\ell_{\infty}$-norm of $\bm{w}^{*}$ is bounded: $\|\bm{w}^{*}\|_{\infty} \le D$ for some $D > 0$.
\end{ass}
Define the $\ell_0$-norm as $\| \bm{v} \|_0 := \# \{ j \in [p] : v_j \neq 0\} $ for any vector $\bm{v} \in \mathbb{R}^p$, and let $s :=\|\vthe^{*}\|_{0} $ and  $s^{\prime}:=\|\bm{w}^{*}\|_{0}$. The following lemmas establish the asymptotic normality of $\nabla\mathcal{L}_{n}^{m}(\vthe^{*})$ and the consistency of the estimator $\widehat {\bm w}$.
\begin{lemma}\label{4-1}
Suppose Assumptions \ref{3-4} and \ref{4-5}
hold. If $\frac{m}{n} \gtrsim \log p$ ,  for any vector $\bm{v}\in\mathbb{R}^{p}$ with $\|\bm{v}\|_{0} = O(1)$, it holds that
\begin{equation*}
    \frac{\sqrt{n}\bm{v}^{\top}\nabla\mathcal{L}_{n}^{m}(\vthe^{*})}{\sqrt{\bm{v}^{\top}H^{*}\bm{v}}} \, \rightsquigarrow \, \mathcal{N}(0,1).
\end{equation*}
\end{lemma}



\begin{lemma}\label{4-4}
    Suppose Assumptions \ref{3-4}-\ref{3-6}, \ref{4-5}- \ref{4-6} hold. Given $\lambda_{1} \asymp \bigg( \sqrt{\frac{\log p }{n}} + \sqrt{\frac{\log p}{m}} + \frac{\log p}{m} \bigg)$, $\lambda_{1} \asymp \lambda_{2} \asymp \lambda^{\prime} = o(1)$, and $s = s^{\prime} = O(1)$, we have
    \begin{equation*}
        \|\widehat{\bm{w}}-\bm{w}^{*}\|_{1}=O_{p}\left((s+s^{\prime})\left(\sqrt{\frac{\log p}{n}}+\sqrt{\frac{\log p}{m}} + \frac{\log p}{m}\right)\right).
    \end{equation*}
\end{lemma}

\noindent Lemma \ref{4-4} implies that the sparsity parameters of $\bm{w}^{*}$ 
are essential for the asymptotic normality of $\nabla\mathcal{L}_{n}^{m}(\vthe^{*})$ and the consistency of $\widehat{\bm{w}}$. Additionally, the condition $\lambda^{\prime}\asymp\lambda_{1}$ is required to ensure the consistency of $\widehat{\bm{w}}$. We now present the main result of the proposed method, which demonstrates the asymptotic normality of the decorrelated score function $\widehat{U}(\alpha,\widehat{\vbeta})$ under the null.



\begin{theorem}\label{4-11}
    Suppose Assumptions \ref{3-4}-\ref{3-6}, \ref{4-5}-\ref{4-6} hold. Let $\lambda_{1}\asymp\bigg(\sqrt{\frac{\log p}{n}}+\sqrt{\frac{\log p}{m}} + \frac{\log p}{m} \bigg)$, $\lambda_{1}\asymp\lambda_{2}\asymp\lambda^{\prime}=o(1)$, and $s=s^{\prime}=O(1)$. Under the null hypothesis $\alpha^{*}=\alpha_{0}$ and the conditions ${\frac{\log p}{\sqrt{n}}+\frac{\log p}{\sqrt{m}} + \frac{\log^{3 / 2} p}{m} = o(1)}$ and ${\frac{m}{n}\gtrsim \log p}$, it follows that  the decorrelated score test $\widehat{U}(\alpha, \widehat{\vbeta})$ as defined in \eqref{4-7} satisfies
    \begin{equation*}
        \sqrt{n}\widehat{U}(\alpha_{0},\widehat{\vbeta}) \, \rightsquigarrow \, \mathcal{N} (0,H^{*}_{\alpha|\vbeta}),
    \end{equation*}
    with variance $H^{*}_{\alpha|\vbeta}= H^{*}_{\alpha\alpha}-\bm{w}^{*\top}H^{*}_{\alpha\vbeta} = H^{*}_{\alpha\alpha} - H^{* \top}_{\alpha\vbeta} H^{*-1}_{\vbeta\vbeta}H^{*}_{\alpha\vbeta}$.
\end{theorem}
\noindent Theorem \ref{4-11} indicates that to achieve the asymptotic normality of the decorrelated score function, we require the condition $\frac{\log p}{\sqrt{n}}+\frac{\log p}{\sqrt{m}} + \frac{\log^{3 / 2} p}{m} = o(1)$, which is a stronger assumption than $\sqrt{\frac{\log p}{n}}+\sqrt{\frac{\log p}{m}} +\frac{\log p}{m} = o(1)$, necessary for the $\ell_1$-consistency of $\widehat{\vthe}$. This stronger assumption aligns with existing literature for proportional hazards models \cite{fang2017testing} and broader statistical frameworks \cite{ning2017general}. Note that the asymptotic variance $H_{\alpha | \beta}^*$ is unknown. Thus, we estimate the variance of the limiting normal distribution with $\widehat{H}_{\alpha|\vbeta}=\nabla^{2}_{\alpha\alpha}\mathcal{L}_{n}^{m}(\widehat{\vthe}) - \widehat{\bm{w}}^{\top}\nabla^{2}_{\alpha\vbeta}\mathcal{L}_{n}^{m}(\widehat{\vthe})$.

\begin{lemma}\label{4-10}
    Suppose Assumptions \ref{3-4}-\ref{3-6}, \ref{4-5}-\ref{4-6} hold. Given $\lambda_{1}\asymp \bigg(\sqrt{\frac{\log p}{n}}+\sqrt{\frac{\log p}{m}} + \frac{\log p}{m}\bigg)$, $\lambda_{1}\asymp\lambda_{2}\asymp\lambda^{\prime}=o(1)$, and $s=s^{\prime}=O(1)$, we have
    \begin{equation*}
       \big| \widehat{H}_{\alpha|\vbeta}-H^{*}_{\alpha|\vbeta} \big|=O_{p}\left((s+s^{\prime})\left(\sqrt{\frac{\log p}{n}}+\sqrt{\frac{\log p}{m}} + \frac{\log p}{m}\right)\right).
    \end{equation*}
\end{lemma}
\noindent  And thus, we define the decorrelated score test statistic as
\begin{equation*}
    \widehat{S}_{n} := \left\{ \begin{array}{ll}
        \sqrt{{n}{\widehat{H}_{\alpha|\vbeta}^{-1}}} \widehat{U}(\alpha_{0},\widehat{\vbeta}), & \widehat{H}_{\alpha|\vbeta} > 0, \\
        0, & \widehat{H}_{\alpha|\vbeta} \leq 0
    \end{array}. \right.
\end{equation*}
by Theorem \ref{4-11} and Lemma \ref{4-10}, we know that the asymptotic distribution of $\widehat{S}_{n}$ is a standard normal distribution under $H_0$. 
Next, we focus on developing a confidence interval for the true parameter $\alpha^{*}$. This approach addresses the limitation noted earlier, where the decorrelated score function did not directly yield a confidence interval for $\alpha^{*}$.

The key idea to derive a confidence interval for $\alpha^{*}$ is based on the decorrelated score function $\widehat{U}(\alpha, \widehat{\vbeta})$. Drawing from Theorem \ref{4-11} and leveraging the properties of Z-estimators as outlined in \cite{van2000asymptotic}, we find that the solution to the equation $\widehat{U}(\alpha, \widehat{\vbeta})=0$ closely approximates $\alpha^{}$. However, directly solving the estimation equation is computationally challenging. To circumvent this, we employ a method that linearizes $\widehat{U}(\alpha, \widehat{\vbeta})$ around the penalized estimator $\widehat{\alpha}$, leading to the introduction of the following one-step estimator $\widetilde{\alpha}$:
\begin{equation}\label{one_step_est}
    \widetilde{\alpha} := \widehat{\alpha} - \left[ \left. \frac{\partial  \widehat{U}(\alpha,\widehat{\vbeta})}{\partial\alpha} \right|_{\alpha = \widehat{\alpha}} \right]^{-1} \widehat{U}(\widehat{\alpha},\widehat{\vbeta}),
\end{equation}
where $\widehat{\vthe} = (\widehat{\alpha}, \widehat{\vbeta}^{\top})^{\top}$ is derived from the Elastic-net-penalized MCMC-MLE in \eqref{2-8}, and $\widehat{U}(\alpha,\widehat{\vbeta})$ is defined as the decorrelated score test in \eqref{4-7}. Notably, the formulation of $\widetilde{\alpha}$ parallels the application of a one-step iteration at the initial point $\alpha=\widehat{\alpha}$, akin to Newton's method, for solving $\widehat{U}(\alpha,\widehat{\vbeta})=0$.

A key advantage of the one-step estimator in high-dimensional settings is its asymptotic normality, a property not shared by traditional penalty estimators such as $\widehat{\theta}$; this is detailed in the theorem below.
\begin{theorem}\label{5-2}
    Suppose Assumptions \ref{3-4}-\ref{3-6}, \ref{4-5}-\ref{4-6},  $\lambda_{1}\asymp \left(\sqrt{\frac{\log p}{n}}+\sqrt{\frac{\log p}{m}} + \frac{\log p}{m}\right)$, $\lambda_{1}\asymp\lambda_{2}\asymp\lambda^{\prime}=o(1)$, and $s=s^{\prime}=O(1)$ hold. 
    If $\frac{\log p}{\sqrt{n}}+\frac{\log p}{\sqrt{m}} + \frac{\log^{3 / 2} p}{m} = o(1)$ and $\frac{m}{n}\gtrsim \log p$, the  decorrelated one-step estimator $\widetilde{\alpha}$ satisfies
    $
        \sqrt{n}(\widetilde{\alpha}-\alpha^{*}) \, \rightsquigarrow \, \mathcal{N}(0,H^{*-1}_{\alpha|\vbeta})$.
\end{theorem}

 Theorem \ref{5-2} and Lemma \ref{4-10} indicate that
\begin{equation}\label{convergence-1}
    \sqrt{n \widehat{H}_{\alpha|\vbeta}}(\widetilde{\alpha}-\alpha^{*}) \, \rightsquigarrow \, \mathcal{N} (0,1).
\end{equation}
Consequently, it is straightforward to construct a $100(1-\eta)\%$ confidence interval for $\alpha^{*}$ with $\widehat{H}_{\alpha \mid \vbeta} > 0$
\begin{equation}\label{confidence_interval}
    \left[  \widetilde{\alpha} - \frac{\Phi^{-1}(1-\frac{\eta}{2})}{\sqrt{n \widehat{H}_{\alpha|\vbeta}}} , \,  \widetilde{\alpha}+\frac{\Phi^{-1}(1-\frac{\eta}{2})}{\sqrt{n \widehat{H}_{\alpha|\vbeta}}} \right].
\end{equation}

\section{FDR Control}\label{sec-6}

In this section, we develop two effective false discovery rate (FDR) control procedures using p-values and e-values. 
Specifically, we define a specific ``metric" value to assess the confidence level associated with each component of $\vthe^*$ via data splitting for the p-value based approach; and further introduce the e-value based approach to facilitate more streamlined computations. 




Formally speaking, FDR control is for the multiple hypothesis testing when addressing a series of hypotheses, such as $H_{k}: \theta_{k}^* = \theta^{(0)}_{k}$ for each $k$ in $[p]$. Here, $\theta^{(0)}_{k}$ signifies the pre-established value against which the true value $\theta_k^*$ is tested \cite{benjamini1995controlling,benjamini2001control,storey2002direct}.
Consider $\mathscr{N} \subseteq [p]$ as the set encompassing unknown true null hypotheses. The FDR control process, based on certain statistical measures, decides on the acceptance or rejection of each hypothesis $H_{k}$. The subset $\mathscr{D} \subseteq [p]$, identified as discoveries in the literature, represents those hypotheses rejected. 
The intersection $\mathscr{N} \, \cap \,  \mathscr{D}$, known as false discoveries, includes the true null hypotheses erroneously rejected. The false discovery proportion (FDP), defined as $\fdp=|\mathscr{N}\cap \mathscr{D}| / |\mathscr{D}|$, is a critical metric indicating the efficacy of the FDR control procedure. Notably, $\fdp$ is a random variable, but the emphasis is on its expected value under the data's generating distribution, referred to as FDR  represented as
\begin{equation*}
    \fdr=\E (\fdp) = \E \left[ \frac{|\mathscr{N}\cap \mathscr{D}|}{ |\mathscr{D}|} \right].
\end{equation*}
Crafting an effective FDR control strategy is a fundamental task in multiple hypothesis testing to control the type-I errors \cite{storey2002direct,vovk2019true,2020False}.

\subsection{FDR control via p-values}\label{sec-6-1}


 False discovery rate is a profound p-value based criterion for controlling uncertainty and avoiding spurious discoveries in multiple testing \cite{benjamini1995controlling}.
Nevertheless, the challenge with classical p-values arises from their inter-correlation, which prevents straightforward summation for statistical analysis. To address the dependence, data splitting is effective in constructing data-driven mirror statistics \cite{dai2022false,du2023false}
\begin{equation*}
    M_j = \operatorname{sgn} \big( T_j^{(1)} T_j^{(2)}\big) f \big( |T_j^{(1)}|, |T_j^{(2)}|\big),
\end{equation*}
where $f(u, v)$ is a function that is non-negative, symmetric, and exhibits monotonic increase for both $u$ and $v$ and $\operatorname{sgn}(\cdot)$ is the sign function. The term $T_j^{(\ell)}$ represents the normalized estimates for $\theta_j$ in the $\ell$-th partition of the entire dataset. Specifically, we employ
\begin{equation*}
    T_j^{(\ell)} = \big( \widetilde{\theta}_j - \theta_j^{*} \big) \sqrt{n \widehat{H}_{j \, | \, -j} / 2} ,
\end{equation*}
where $\widetilde{\theta}_j$ denotes the one-step estimator of $\theta_j$ in \eqref{one_step_est} and $\widehat{H}_{j \, | \, -j} := \widehat{H}_{\theta_j \mid \vthe_{-j}}$. Thus, we first split the data into two disjoint parts, and then run Algorithm \ref{alg1} analogous to that in \cite{dai2022false} to asymptotically control the FDR.

\begin{algorithm}[H]
\caption{FDR control via single data split}\label{alg1}
\begin{algorithmic}[1] 
    \vspace{1ex}
    \INPUT The data and designated FDR control level $q$;
    \vspace{1ex}
    
    \State For each $j$, calculate the mirror statistic $M_j$.
    
    \State Given a designated level $q \in (0, 1)$, the cutoff is 
    \begin{equation*}
        \tau_{q}=\inf \left\{t>0: \widehat{\mathrm{FDP}}(t) :=\frac{\#\left\{j: M_{j}<-t\right\}}{\#\left\{j: M_{j}>t\right\}} \leq q\right\}.
    \end{equation*}
    
    \State Set $\widehat{\mathscr{S}} = \{j : M_j > \tau_q \}$.
\end{algorithmic}
\end{algorithm}
We necessitate an additional assumption to theoretically substantiate our method for controlling the FDR.
\begin{ass}\label{ass6}
    The indices ${j \in S_0}$ are exchangeable. Specifically, for the population version $\X$, we have that ${X}_{j} \xlongequal{d.} {X}_{k}$ and ${X}_{j} | \bm{X}_{-j} \xlongequal{d.} {X}_{k} | \bm{X}_{-k}$.
\end{ass}

\noindent Assumption \ref{ass6} is commonly encountered in the literature on multiple hypothesis testing, underpinning the concept of `knockoff filtering' \cite{barber2015controlling,romano2020deep,huang2020relaxing,chen2023systematic}. This assumption is satisfied by a broad spectrum of models within the exponential family of MRFs, including Ising model and the Gaussian graphical model. 

\begin{theorem}\label{thm6-1}
    Suppose the conditions in Theorem \ref{5-2} hold. Assume that $\widehat{H}_{j \mid -j}$ is uniformly integrable for any $j \in \{ j \in [p] : \theta_j^* = 0\}$. Under Assumption \ref{ass6}, we have
    \begin{equation*}
        \limsup_{n,p \rightarrow \infty} \fdr := \mathop{\lim \sup}_{n, p \rightarrow \infty} \mathrm{E} \left[ \frac{\# \{j : j \in S_0, \, j \in \widehat{S}_{\tau_q}\}}{\#\{j \in \widehat{S}_{\tau_q}\}} \right] \leq q.
    \end{equation*}
\end{theorem}

\noindent Theorem \ref{thm6-1} validates that single data split (Algorithm \ref{alg1}) yields effective FDR at the target level. However, 
potential loss of power or instability \cite{dai2022false} are involved due to the split randomness. To mitigate these issues, we consider multiple data splits, akin to the approach described in \cite{dai2022false}. Define
\begin{equation*}
    I_j = \mathrm{E} \left[ \frac{\mathds{1}(j \in \widehat{\mathscr{S}})}{|\widehat{\mathscr{S}}|}\right], \qquad \widehat{I}_j = \frac{1}{m} \sum_{k = 1}^m \frac{\mathds{1}(j \in \widehat{\mathscr{S}}^{(k)})}{|\widehat{\mathscr{S}}^{(k)}|},
\end{equation*}
where $\widehat{\mathscr{S}}^{(k)}$ denotes the set of selected features in the $k$-th data split as determined by Algorithm \ref{alg1}. Set $I_j = 0$ if $|\widehat{\mathscr{S}}| = 0$. The term $I_j$ represents the inclusion rate, reflecting the extent to which information about $\theta_j$ is captured within the selected set $\widehat{\mathscr{S}}$. 

\begin{algorithm}[H]
\caption{FDR control via multiple data split}\label{alg2}
\begin{algorithmic}[1] 
    \vspace{1ex}
    \INPUT The data and designated FDR control level $q$;
    \vspace{1ex}
    
    \State Sort $\widehat{I}_j$ by $0 \leq \widehat{I}_{(1)} \leq \widehat{I}_{(2)} \leq \cdots \leq \widehat{I}_{(p)}$.
    
    \State Find the largest $\ell$ such that $\widehat{I}_{(1)} + \cdots + \widehat{I}_{(\ell)} \leq q$.
    
    \State Select the features $\widehat{\mathscr{S}}^{\text{mul}} = \{ j : \widehat{I}_j > \widehat{I}_{(\ell)}\}$.
\end{algorithmic}
\end{algorithm}

\begin{theorem}\label{thm6-2}
Suppose the conditions in Theorem \ref{thm6-1} hold. The $\widehat{\mathscr{S}}^{\text{mul}}$ selected by Algorithm \ref{alg2} satisfies
    \begin{equation*}
        \limsup_{n,p \rightarrow \infty} \fdp := \mathop{\lim \sup}_{n, p \rightarrow \infty} \frac{\# \{j \, : \, j \in S_0, j \in \widehat{\mathscr{S}}^{\text{mul}}\}}{\# \{j \, : \, j \in \widehat{\mathscr{S}}^{\text{mul}}\}} \leq q.
    \end{equation*}
\end{theorem}
\noindent Theorem \ref{thm6-2} confirms that Algorithm \ref{alg2} is capable of efficiently controlling the FDR. Moreover, given that
\begin{equation*}
    \operatorname{var} \bigg( \frac{\mathds{1} (j \in \widehat{\mathscr{S}})}{|\widehat{\mathscr{S}}|}\bigg) \geq \operatorname{var} \bigg( \mathrm{E} \bigg[ \frac{\mathds{1} (j \in \widehat{\mathscr{S}})}{|\widehat{\mathscr{S}}|} \, \bigg| \, \text{data} \bigg]\bigg),
\end{equation*}
the implementation of multiple data splits in Algorithm \ref{alg2} can be considered a Rao-Blackwell improvement over the single data split approach delineated in Algorithm \ref{alg1}.

\subsection{FDR control via e-values}

Controlling FDR typically aims to restrict the expected proportion of Type I errors, which involve the incorrect rejection of a true null hypothesis. The p-value based approaches are typically adept at closely controlling the FDR at a pre-specified level, implying a concurrent tendency to manage Type II errors – the probability of erroneously retaining a false null hypothesis \cite{dai2022false, Dai2020A}
However, in some scenarios, the focus on Type II error is less critical, allowing for the possibility of controlling the FDR at levels significantly below the designated threshold. This approach may be overly stringent, but if such strictness is not a concern, it can be an effective strategy. In these instances, the e-BH procedure proves beneficial \cite{vovk2019true,vovk2021values,wang2020false}. The e-values are characterized by their additivity, independent of the correlation among covariates. This unique attribute simplifies and streamlines the FDR control process using e-values, compared to approaches based on p-values.
In this section, we will design the e-BH procedure in \cite{vovk2019true,vovk2021values,wang2020false}, and then analyze the asymptotic properties of the proposed methods.

The key to a valid e-value procedure is the establishment of an e-variable $E$ such that $\E E = 1$ under the null hypothesis \cite{vovk2021values}. When the hypothesis $H_{k}$ is true, by \eqref{convergence-1} we observe that $\sqrt{2 / \pi} \cdot E_{k} \, \rightsquigarrow \, \big| \mathcal{N} (0,1) \big|$, leading to $\lim_{n \rightarrow \infty} \E E_k = 1$. Based on this, we propose the following Algorithm \ref{alg3}.


\begin{algorithm}[H]
\caption{FDR control via e-values}\label{alg3}
\begin{algorithmic}[1] 
    \vspace{1ex}
    \INPUT The data and designated FDR control level $q$;
    \vspace{1ex}
    
    \State For each $k\in \mathcal{P}$, compute the e-values $e_{k}$ as one observation of $E_k$ corresponding to each hypothesis $H_{k}$ by
    \begin{equation}\label{6-2}
        E_{k}=\sqrt{\pi / 2}\cdot \sqrt{n \widehat{H}_{\theta_{k}|\vthe_{-k}}}\left|\widetilde{\theta}_{k} - \theta^{(0)}_{k}\right|.
    \end{equation}
    
    \State\label{step-2} For each $k\in \mathcal{P}$, let $e_{(k)}$ be the $k$-th order statistics of $e_{1},\cdots,e_{p}$ from the largest to the smallest.

\State Define $k^{*}$ to be
\begin{equation*}
    k^{*}:=\max \left\{ k\in\mathcal{K}: \frac{k e_{(k)}}{p} \ge \frac{1}{q} \right\}\,(\max \{ \varnothing \}=0),
\end{equation*}
reject the largest $k^{*}$ e-values (hypotheses).
\end{algorithmic}
\end{algorithm}

For the purpose of avoiding overly complex structures in the index set, we introduce the following assumption.

\begin{ass}\label{6-1}
For any $\mathscr{S}_{n} \subseteq [p]$,  
$
    \limsup_{n,p \rightarrow \infty} \frac{1}{|\mathscr{S}_{n}|}\sum_{k\in\mathscr{S}_{n}}\E (E_{k}) \le 1 .
$
\end{ass}
\noindent Assumption \ref{6-1} is pivotal in ensuring average convergence on sets with diverging numbers of elements. However, this assumption can be relaxed with additional constraints on the data, as discussed in Section 4 of \cite{chakrabortty2018inference}. Under this assumption, the following theorem addressees the control of FDR in an asymptotic framework.



\begin{theorem}\label{6-3}
Under Assumptions \ref{3-4} to \ref{6-1}, Algorithm \ref{alg3} follows that $\limsup_{n,p \rightarrow \infty} \fdr \le q$.
\end{theorem}

As we can see, when incorporating Assumption \ref{6-1}, Theorem \ref{6-3} ensures that Algorithm \ref{alg3} asymptotically controls the FDR at the desired level. Therefore, our e-value procedure provides theoretically valid FDR control, similar to the previous p-value procedure.

\section{Simulation Studies}\label{sec_sim}

In this section, we conduct extensive numerical studies to assess the finite sample performance of the proposed procedure.
We generate $\Y_1, \ldots, \Y_m$ i.i.d. from $\mathcal N(0, I_d)$ with $d = 1$. 
The following scenarios for $\varphi$ are investigated:
\begin{itemize}
\item $\varphi^{(1)} (x) = \big( \cos (x \pi), \cos (2 x \pi), \ldots, \cos (p x \pi) \big)^{\top}$.
\item $ \varphi^{(2)} (x) = \big( \arctan (x), \arctan (2 x), \ldots, \log (p x + 1) \big)^{\top}$.
\item $\varphi^{(3)} (x) = \left(\frac{1}{1 + x}, \frac{1}{1 + x^2}, \ldots, \frac{1}{1 + x^p}\right)^{\top}$.
\end{itemize}
Utilizing the Metropolis sampling method \cite{hastings1970monte}, we generate samples $\X_1, \ldots, \X_n$, treating $\{ \X_i\}_{i = 1}^n$ as an i.i.d. sample. In each simulation, cross-validation is employed to select the tuning parameters $(\lambda_1, \lambda_2, \lambda') \in \mathbb{R}_+^3$. The elements of the true parameter $\vthe^* = (\theta_1^*, \theta_2^*, \ldots, \theta_p^*)$ are generated from $U \times \mathds{1} ( U' < 0.1)$, where $U$ and $U'$ are independently uniformly distributed over $(0, 1)$. We set the simulated sample size as $m = n$. All the simulation results are based on $N=100$ replications.
\begin{table}[htbp]
  \centering
  \caption{The average $\ell_1$ error $p^{-1} \|\widehat{\vthe} - \vthe^* \|_1$ under $N = 100$ duplications.}
  \resizebox{0.5\paperheight}{!}{\begin{tabular}{ccccccccccc}
    \toprule
    \toprule
    \multicolumn{11}{c}{$\varphi^{(1)}(x)$} \\
    \midrule
       \diaghead(-2,1){aaaaaaaaaa}{$\quad n$}{$p$}   & 50    & 100   & 150   & 200   & 250   & 300   & 350   & 400   & 450   & 500 \\
    \midrule
    100   & 0.0583  & 0.1232  & 0.2042  & 0.5108  & 1.8834  & 1.5548  & 1.5819  & 2.0679  & 2.9404  & 3.7422  \\
    \midrule
    200   & 0.0422  & 0.0751  & 0.0903  & 0.1596  & 0.2024  & 0.2187  & 0.4313  & 0.5091  & 0.5176  & 1.1599  \\
    \midrule
    300   & 0.0354  & 0.0541  & 0.0909  & 0.1052  & 0.1183  & 0.1741  & 0.2058  & 0.2232  & 0.2828  & 0.3410  \\
    \midrule
    400   & 0.0326  & 0.0614  & 0.0965  & 0.1044  & 0.0863  & 0.1089  & 0.1622  & 0.1689  & 0.1952  & 0.2074  \\
    \midrule
    500   & 0.0762  & 0.0853  & 0.0550  & 0.0783  & 0.0930  & 0.1220  & 0.1221  & 0.1227  & 0.1615  & 0.1903  \\
    \midrule
    600   & 0.0219  & 0.0565  & 0.0453  & 0.0696  & 0.0695  & 0.0821  & 0.1344  & 0.1116  & 0.1513  & 0.1390  \\
    \midrule
    700   & 0.0281  & 0.0377  & 0.0476  & 0.0612  & 0.0714  & 0.0920  & 0.0914  & 0.1051  & 0.1102  & 0.1218  \\
    \midrule
    800   & 0.0230  & 0.0550  & 0.0373  & 0.1000  & 0.0755  & 0.0710  & 0.1041  & 0.0957  & 0.1244  & 0.1294  \\
    \midrule
    900   & 0.0218  & 0.0351  & 0.0667  & 0.0617  & 0.0651  & 0.0614  & 0.0788  & 0.0911  & 0.0888  & 0.1185  \\
    \midrule
    1000  & 0.0521  & 0.0452  & 0.0854  & 0.0495  & 0.0620  & 0.0737  & 0.0794  & 0.0795  & 0.0998  & 0.1097  \\
    \midrule
    \multicolumn{11}{c}{$\varphi^{(2)}(x)$} \\
    \midrule
       \diaghead(-2,1){aaaaaaaaaa}{$\quad n$}{$p$}   & 50    & 100   & 150   & 200   & 250   & 300   & 350   & 400   & 450   & 500 \\
    \midrule
    100   & 0.2420  & 0.9483  & 0.0920  & 3.7035  & 5.7790  & 0.9768  & 11.1498  & 14.4463  & 18.1982  & 19.6261  \\
    \midrule
    200   & 0.1179  & 0.4703  & 0.1048  & 0.1615  & 2.8785  & 0.2173  & 5.5507  & 7.1994  & 9.3011  & 11.4753  \\
    \midrule
    300   & 0.0790  & 0.3118  & 0.0252  & 1.2319  & 0.0555  & 0.0748  & 3.7021  & 4.7937  & 6.1796  & 8.6076  \\
    \midrule
    400   & 0.0596  & 0.2335  & 0.0174  & 0.9251  & 1.4300  & 0.2248  & 2.7787  & 3.6098  & 4.6258  & 5.7292  \\
    \midrule
    500   & 0.0475  & 0.1853  & 0.4183  & 0.0309  & 1.1425  & 0.1636  & 0.1216  & 2.8813  & 3.6366  & 4.5797  \\
    \midrule
    600   & 0.0393  & 0.1545  & 0.0189  & 0.0264  & 0.9494  & 0.1365  & 0.1844  & 2.3881  & 0.2575  & 0.4251  \\
    \midrule
    700   & 0.0332  & 0.1329  & 0.0110  & 0.0621  & 0.8180  & 0.0577  & 0.1580  & 2.0542  & 2.6547  & 0.3268  \\
    \midrule
    800   & 0.0297  & 0.1164  & 0.2591  & 0.0239  & 0.7172  & 0.0361  & 0.1384  & 0.1570  & 0.1404  & 0.2866  \\
    \midrule
    900   & 0.0261  & 0.1033  & 0.0066  & 0.4086  & 0.0246  & 0.0352  & 0.1229  & 0.1583  & 0.1321  & 0.2545  \\
    \midrule
    1000  & 0.0236  & 0.0024  & 0.0208  & 0.0120  & 0.0573  & 0.0818  & 0.1105  & 0.1441  & 0.1205  & 0.1126  \\
    \midrule
    \multicolumn{11}{c}{$\varphi^{(3)}(x)$} \\
    \midrule
       \diaghead(-2,1){aaaaaaaaaa}{$\quad n$}{$p$}   & 50    & 100   & 150   & 200   & 250   & 300   & 350   & 400   & 450   & 500 \\
    \midrule
    100   & 0.0744  & 0.0802  & 0.3834  & 0.1714  & 0.6114  & 0.8035  & 0.5839  & 0.4769  & 2.0614  & 6.2039  \\
    \midrule
    200   & 0.0724  & 0.1124  & 0.0550  & 0.1127  & 0.3119  & 1.8504  & 0.2266  & 0.2357  & 0.6231  & 0.8614  \\
    \midrule
    300   & 0.0114  & 0.0465  & 0.0596  & 0.0549  & 0.2676  & 0.1700  & 0.1451  & 0.1181  & 0.4000  & 0.1436  \\
    \midrule
    400   & 0.0095  & 0.0171  & 0.0335  & 0.0427  & 0.5303  & 0.0639  & 0.3037  & 0.1764  & 0.2361  & 1.4182  \\
    \midrule
    500   & 0.0101  & 0.0142  & 0.0235  & 0.0672  & 0.1142  & 0.0819  & 0.8170  & 1.0434  & 0.7713  & 0.1647  \\
    \midrule
    600   & 0.0057  & 0.0135  & 0.0469  & 0.0450  & 0.3342  & 0.2689  & 0.6022  & 0.8204  & 0.6905  & 0.1272  \\
    \midrule
    700   & 0.0060  & 0.0100  & 0.0185  & 0.0258  & 0.0399  & 0.0605  & 0.0472  & 0.8851  & 0.8351  & 0.0566  \\
    \midrule
    800   & 0.0042  & 0.0104  & 0.0210  & 0.0197  & 0.0346  & 0.0485  & 0.0449  & 0.6599  & 0.4567  & 0.0721  \\
    \midrule
    900   & 0.0039  & 0.0086  & 0.0136  & 0.0217  & 0.0261  & 0.0411  & 0.0538  & 0.4903  & 0.0550  & 0.0577  \\
    \midrule
    1000  & 0.0035  & 0.0084  & 0.0308  & 0.0273  & 0.0201  & 0.0267  & 0.0392  & 0.0309  & 0.0695  & 0.0572  \\
    \bottomrule
    \bottomrule
    \end{tabular}}%
  \label{tab_L_1}%
\end{table}%

To assess the estimation performance of the proposed methods, we calculate the $\ell_1$-error of the Elastic-net estimator $\widehat{\vthe}$ from $p^{-1} \|\widehat{\vthe} - \vthe^* \|_1$.
Table \ref{tab_L_1} shows our estimator demonstrating consistency under both low and high-dimensional cases when the tuning parameters $\lambda_1$ and $\lambda_2$ are appropriately selected. This empirically confirms the theoretical consistency guarantee stated in Theorem \ref{thm-1}.

Next, we investigate the one-step estimator with the empirical rejection rate $\operatorname{rate} \big(\widetilde{\theta}_1\big) =  \frac{1}{N} \sum_{\ell = 1}^N 1\big( \widetilde{\theta}_1^{(\ell)} \in CI^{(\ell)}\big)$, where $N = 50$ represents the number of replications, and $\widetilde{\theta}_1^{(\ell)}$ and $CI^{(\ell)}$ correspond to the values obtained in the $\ell$-th simulation.  For the sake of simplicity, we construct the $95\%$ confidence interval for $\widetilde{\theta}_1$ defined \eqref{confidence_interval}.  The results of this analysis are depicted in Figure \ref{fig_rej}.
As shown in Figure \ref{fig_rej}, the one-step estimators exhibit the asymptotic normality stated in Theorem \ref{5-2}, without requiring extremely large sample sizes. The empirical rejection rates match the designed confidence level when $n$ is around 300. This demonstrates the asymptotic normality holding reasonably well for moderately sized samples.

\begin{figure}[t!]
    \minipage{0.33\textwidth}
         \includegraphics[width=\linewidth]{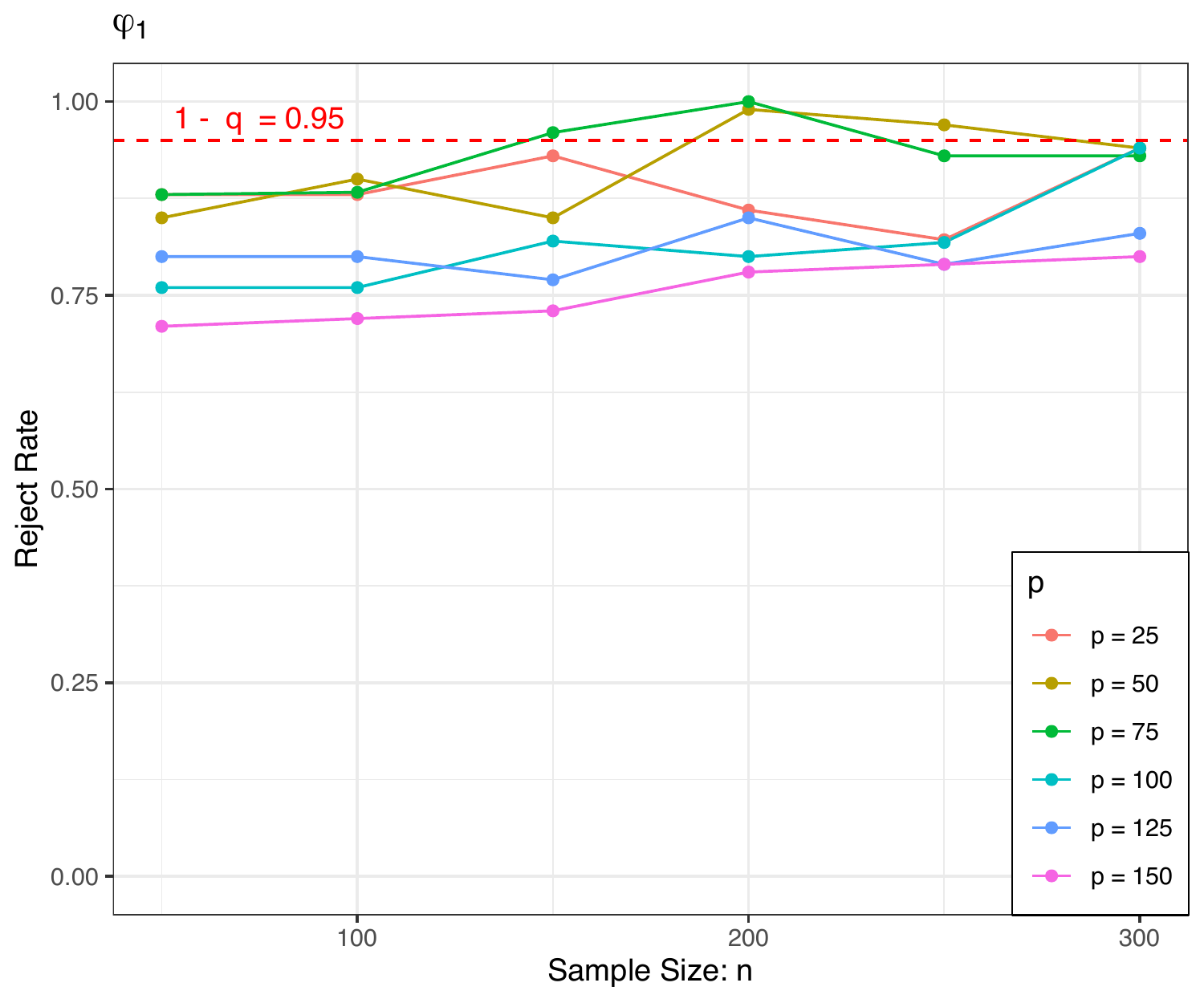}
    \endminipage\hfill
    \minipage{0.33\textwidth}
        \includegraphics[width=\linewidth]{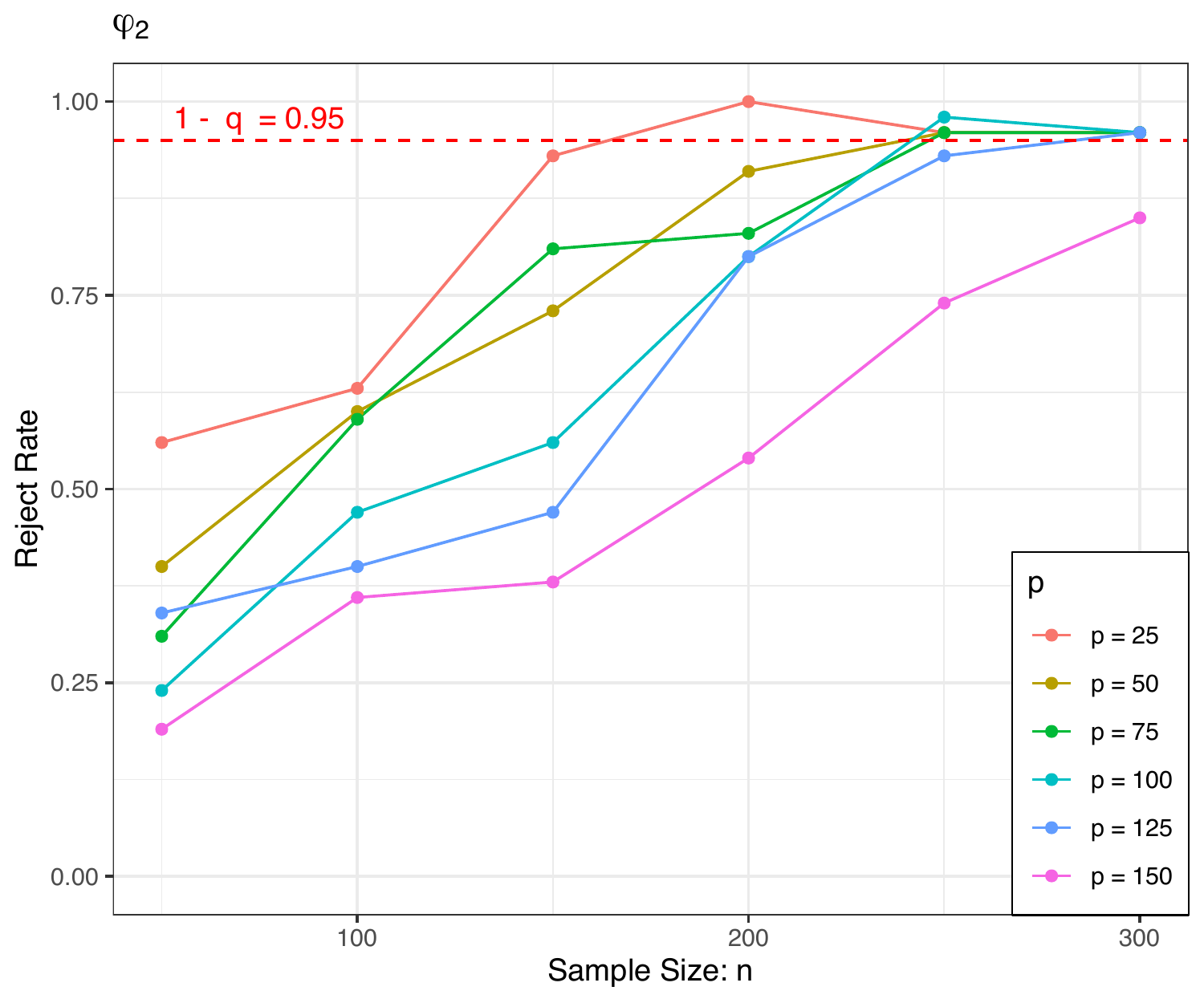}
    \endminipage
    \minipage{0.33\textwidth}
        \includegraphics[width=\linewidth]{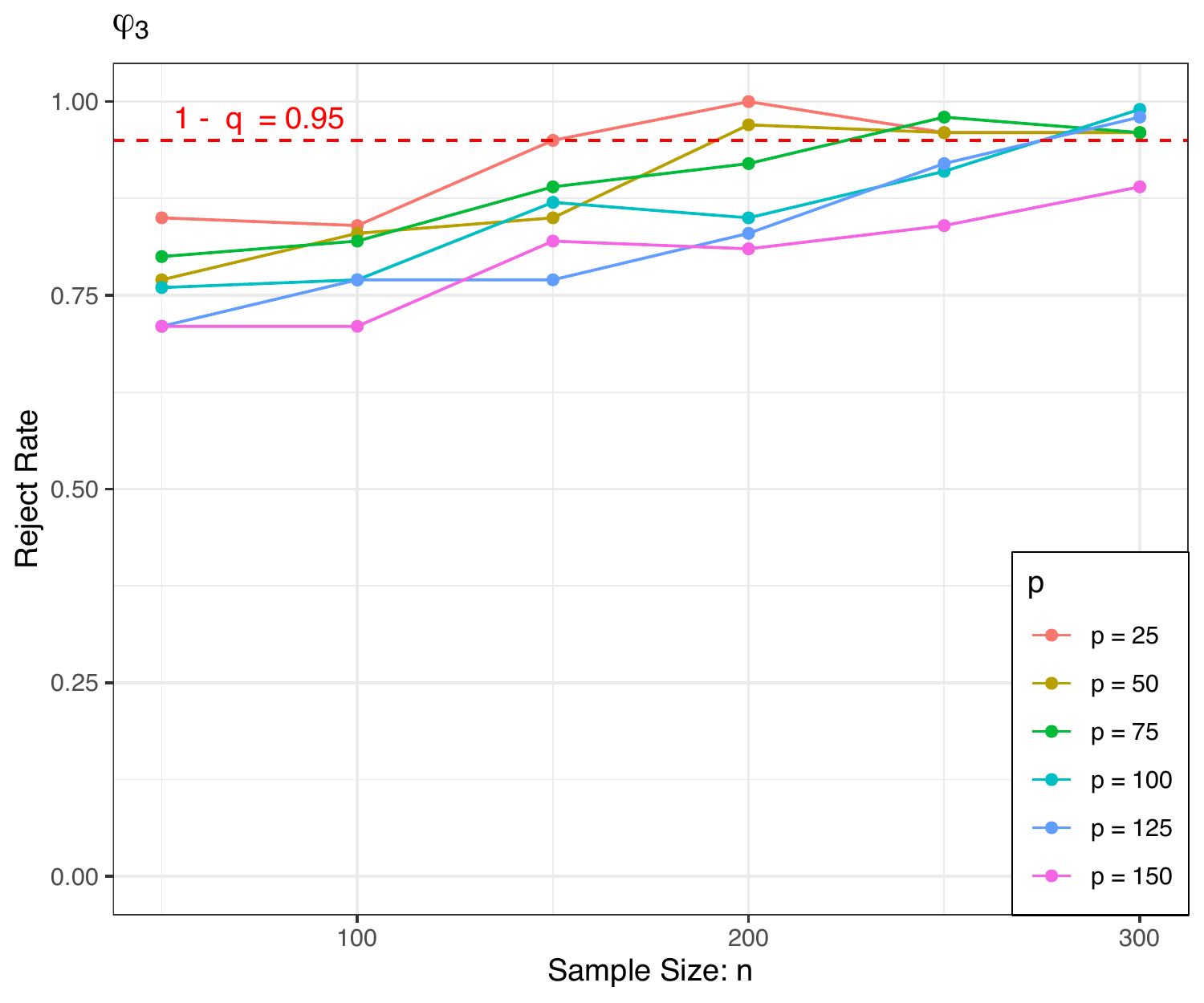}
    \endminipage
    \caption{Empirical rejection rate under $N = 100$ replications.}
    \label{fig_rej}
\end{figure}

As for FDR control, we study the
traditional false discovery rate control procedures via p-values (Algorithm \ref{alg1}) and the more novel approach utilizing e-values (Algorithm \ref{alg3}) with a nominal level of $q = 0.05$. 
Table \ref{tab_FDR} demonstrates that both the p-value and e-value based procedures effectively maintain the desired level of FDR control. Notably, the e-value based procedure (Algorithm \ref{alg3}) tends to produce a more conservative FDP level. Moreover, the e-value procedure is more computationally efficient than the p-value procedure.


\begin{table}[h!]
  \centering
  \caption{The FDP levels under both p-value procedure and e-value procedure with $q=0.05$. In  each cell, $(FDP_{\text{alg 1}}, FDP_{\text{alg 3}})$ represents the average FDP levels of Algorithm \ref{alg1} and Algorithm \ref{alg3} under $N = 50$ independent repeated simulations. }
    \resizebox{\columnwidth}{!}{\begin{tabular}{ccccccccccc}
    \toprule
    \toprule
    \multicolumn{11}{c}{$\varphi^{(1)}(x)$} \\
    \midrule
      \diaghead(-2,1){aaaaaaaaaa}{$\quad n$}{$p$}    & 50    & 100   & 150   & 200   & 250   & 300   & 350   & 400   & 450   & 500 \\
    \midrule
    100   & 0.0598, 0.02 & 0.0972, 0 & 0.1234, 0 & 0.1338, 0 & 0.1001, 0 & 0.1387, 0 & 0.1084, 0.02 & 0.1059, 0.06 & 0.0963, 0.04 & 0.0957, 0.102 \\
    \midrule
    200   & 0.0060, 0 & 0.0062, 0 & 0.0325, 0 & 0.0671, 0 & 0.1001, 0 & 0.0706, 0 & 0.1084, 0 & 0.1059, 0 & 0.0963, 0 & 0.0957, 0 \\
    \midrule
    300   & 0.0030, 0 & 0, 0  & 0.0011, 0 & 0.0018, 0 & 0.0148, 0 & 0.0435, 0 & 0.0724, 0 & 0.0928, 0 & 0.0803, 0 & 0.1103, 0 \\
    \midrule
    400   & 0, 0  & 0, 0  & 0, 0  & 0, 0  & 0, 0  & 0.0042, 0 & 0.0116, 0 & 0.0287, 0 & 0.0623, 0 & 0.0672, 0 \\
    \midrule
    500   & 0, 0  & 0, 0  & 0, 0  & 0, 0  & 0, 0  & 0, 0  & 0, 0  & 0, 0  & 0.0066, 0 & 0.0303, 0 \\
    \midrule
    \multicolumn{11}{c}{$\varphi^{(2)}(x)$} \\
    \midrule
       \diaghead(-2,1){aaaaaaaaaa}{$\quad n$}{$p$}   & 50    & 100   & 150   & 200   & 250   & 300   & 350   & 400   & 450   & 500 \\
    \midrule
    100   & 0.0162, 0.02 & 0.0584, 0 & 0.0737, 0 & 0.1408, 0 & 0.1289, 0.02 & 0.1105, 0.02 & 0.1171, 0.04 & 0.0647, 0.04 & 0.069, 0.08 & 0.1121, 0.12 \\
    \midrule
    200   & 0, 0  & 0, 0  & 0.0068, 0 & 0.0399, 0 & 0.0723, 0 & 0.1105, 0 & 0.1171, 0 & 0.0073, 0 & 0.0337, 0 & 0.0636, 0 \\
    \midrule
    300   & 0, 0  & 0, 0  & 0, 0  & 0, 0  & 0.0017, 0 & 0.0173, 0 & 0.0348, 0 & 0.0647, 0 & 0.0690, 0 & 0.1121, 0 \\
    \midrule
    400   & 0, 0  & 0, 0  & 0, 0  & 0, 0  & 0, 0  & 0, 0  & 0, 0  & 0.0011, 0 & 0.0254, 0 & 0.0496, 0 \\
    \midrule
    500   & 0, 0  & 0, 0  & 0, 0  & 0, 0  & 0, 0  & 0, 0  & 0, 0  & 0.0005, 0 & 0.0010, 0 & 0.0023, 0 \\
    \midrule
    \multicolumn{11}{c}{$\varphi^{(3)}(x)$} \\
    \midrule
       \diaghead(-2,1){aaaaaaaaaa}{$\quad n$}{$p$}   & 50    & 100   & 150   & 200   & 250   & 300   & 350   & 400   & 450   & 500 \\
    \midrule
    100   & 0.0988, 0.08 & 0.1421, 0.02 & 0.1145, 0 & 0.1024, 0.04 & 0.1408, 0.04 & 0.0954, 0.1 & 0.1132, 0.12 & 0.0935, 0.12 & 0.0836, 0.08 & 0.0992, 0.12 \\
    \midrule
    200   & 0.0771, 0.002 & 0.1342, 0 & 0.1145, 0.02 & 0.1024, 0 & 0.1408, 0.02 & 0.0954, 0 & 0.0814, 0.02 & 0.1111, 0.02 & 0.1067, 0.02 & 0.1060, 0.06 \\
    \midrule
    300   & 0.0988, 0 & 0.0408, 0 & 0.0617, 0 & 0.0685, 0 & 0.1138, 0 & 0.1096, 0 & 0.1132, 0 & 0.0935, 0 & 0.0836, 0 & 0.0992, 0 \\
    \midrule
    400   & 0.0697, 0 & 0.0238, 0 & 0.0251, 0 & 0.028, 0 & 0.0346, 0 & 0.0505, 0 & 0.0696, 0 & 0.0885, 0 & 0.0997, 0 & 0.0996, 0 \\
    \midrule
    500   & 0.0405, 0 & 0.0232, 0 & 0.0122, 0 & 0.0076, 0 & 0.0102, 0 & 0.0183, 0 & 0.0223, 0 & 0.0368, 0 & 0.0428, 0 & 0.0556, 0 \\
    \bottomrule
    \bottomrule
    \end{tabular}}%
  \label{tab_FDR}%
\end{table}%

\section{Concluding remarks}

This paper investigated the estimation and inference of parameters for Markov Random Fields in canonical exponential families. We propose a penalized Markov Chain Monte Carlo Maximum Likelihood method to estimate the intractable normalizing constant in high-dimensional settings. To facilitate statistical inference on the obtained estimates, we subsequently devise a one-step estimator with decorrelated scores and thus present two FDR controlling procedures. Our extensive numerical investigations demonstrate that the proposed methods provide satisfactory performances in estimation and inference. Given our methodology's reliance on mild assumptions, its broad potential applicability, warrants further research. 




\section*{Acknowledgments}
The research of H. Zhang is supported in part by NSFC Grant No.12101630.

\section*{Author contributions statement}

H.Z. was the driving force behind the motivation for this study and developed the main ideas. H.W., X.L. and Y.H. were responsible for the writing and rigorous completion of all proofs. H.W. took the lead in implementing the experiments. All authors participated in writing and reviewing the manuscript.

\bibliographystyle{chicago}
\bibliography{ref}
\newpage
\begin{appendix}

\section*{Appendix}

\section{Concentration Inequalities for Markov Chain}\label{app_ConIneq_MC}
Concentration inequalities provide essential non-asymptotic error bounds for sums of random variables, crucial for high-dimensional statistical inference. They are notably used in analyzing how closely these sums approximate their expected values, as demonstrated in applications like Oracle inequalities in linear models and testing in high-dimensional regression models \cite{bickel2009simultaneous,ning2017general,zhang2020concentration}.

In MCMC-MLE, the Markov chain's trajectory average approximates the normalizing constant $C(\vthe)$. To evaluate how this average aligns with $C(\vthe)$, we require a concentration inequality for the Markov chain. Our setup involves a Markov chain $\{\Y_{i},i\ge 1\}$ on state space ${\mathcal{X}}$, with transition kernel $P(\y,\x)$, initial distribution $q(\Y)$, and stationary distribution $h(\Y)$. We assume the chain is irreducible and aperiodic, typical in MCMC, ensuring a unique stationary distribution and ergodicity. Key quantities are defined as
\begin{equation}\label{2-13}
    M_{1}:=\operatorname{E}_{\Y\sim h}\left( \frac{p(\Y|\vthe^{*})}{h(\Y)} \right)^{2}=\frac{1}{C^{2}(\theta^{*})}\operatorname{E}_{\Y\sim h}\left( \frac{e^{\vthe^{*\top}\varphi(\Y)}}{h(\Y)} \right)^{2} \quad\text{and}\quad M_{2}:=\max_{\Y\in\bm{\mathcal{X}}}\frac{p(\Y|\vthe^{*})}{h(\Y)},
\end{equation}
where $M_{1}$ and $M_{2}$ gauge the deviation between the stationary distribution $h(\y)$ and the true distribution $p(\y \mid \vthe^{*})$.

Denote the Hilbert space $L^{2}(h)$ with the inner product $\langle f,g \rangle = \sum_{\y\in\bm{\mathcal{X}}}f(\y)g(\y)h(\y)$. A linear operator $P$, connected to the transition kernel $P(\y, \x)$, is defined as $(Pf)(\y):=\sum_{\x\in {\mathcal{X}}}f(\x)P(\y, \x)$.
The spectral gap of the Markov chain, denoted as $1-\kappa$, is characterized by $\kappa=\sup\{|\rho|:\rho\in  \operatorname{Spec} (P)\}$,
where $\operatorname{Spec}(P)$ represents the spectrum of $P$. We define $\beta_{1}$ and $\beta_{2}$ to describe the convergence rate of a trajectory's average to its expectation:
\begin{equation}\label{2-16}
    \beta_{1}:=\frac{1+\kappa}{1-\kappa} \quad\text{and}\quad \beta_{2}:=
    \begin{cases}
    \frac{1}{3}, & \text{if}\quad\kappa=0 \\
    \frac{5}{1-\kappa}, & \text{if}\quad\kappa\in(0,1)
    \end{cases}.
\end{equation}
These definitions set the stage for a crucial concentration inequality for Markov chains, pivotal in the proof of this article. 
\begin{lemma}[Theorem 1 in \cite{jiang2018bernstein}]\label{2-17}
Let $\{ \Y_{1},\cdots,\Y_{m} \}$ be a stationary Markov chain with invariant distribution $h(\Y)$ and non-zero absolute spectral gap $1-\kappa>0$. Suppose $g_{i}:\bm{\mathcal{X}}\longrightarrow [-M,+M]$ is a sequence of functions with $\operatorname{E}_{\Y\sim h}g_{i}(\Y)=0$. Let $\sigma^{2}=\frac{1}{m}\sum_{i=1}^{m}\operatorname{E}_{\Y\sim h}g_{i}^{2}(Y)$. Then for any $t>0$, we have
\begin{equation*}
    \pr \left( \frac{1}{m}\sum_{i=1}^{m}g_{i}(\Y_{i}) \ge t \right)\le \exp\left( -\frac{mt^{2}}{2(\beta_{1}\sigma^{2}+\beta_{2}Mt)} \right).
\end{equation*}
\end{lemma}

\section{$\ell_{1}$-Consistency}\label{sec-3}

In this part, we will address the oracle inequality and demonstrate the $\ell_{1}$-consistency of MCMC-MLE penalized by Elastic-net under certain regularity conditions.

To study the Elastic-net penalty, consider a more generalized framework with the concept of symmetric Bregman divergence
\[
    D^{s}_{g}(\widehat{\vthe},\vthe):=(\widehat{\vthe}-\vthe)^{\top} \Big[ \nabla \mathcal{L}(\widehat{\vthe})-\nabla \mathcal{L}(\vthe)+\lambda_{2}\big(\nabla g(\widehat{\vthe})-\nabla g(\vthe) \big) \Big],
\]
where $\widehat{\vthe}(\lambda_{1},\lambda_{2})=\mathop{\arg\min}_{\vthe}\mathcal{L}(\vthe)+\lambda_{1}\left\|\vthe\right\|+\lambda_{2}g(\vthe)$ is the generalized Lasso-type convex penalty (GLCP). In our case, we have $g(\cdot) = \| \cdot \|_2$. The following lemma provides an upper bound for the symmetric Bregman divergence \cite{nielsen2009sided,yu2010high,zhang2017elastic}.

\begin{lemma}\label{2-11}
    For GLCP maximum likelihood estimation, we have for any $T \subseteq [p]$
    \[
        D^{s}_{g}(\widehat{\vthe},\vthe^{*})\le (\lambda_{1}+z^{*})\|\widehat{\vthe}_{T}-\vthe^{*}_{T}\|_{1} - (\lambda_{1}-z^{*})\|\widehat{\vthe}_{T^{c}} - {\vthe}_{T^{c}}^*\|_{1},
    \]
    where  $z^{*}:=\left\|\nabla \mathcal{L}(\vthe^{*})+\lambda_{2}\nabla g(\vthe^{*})\right\|_{\infty}$.
\end{lemma}

We define the cone set, $\mathcal{C}(\zeta,T) := \big\{\vthe:\|\vthe_{T^{c}}\|_{1}\le \zeta\|\vthe_{T}\|_{1} \big\}$ with parameter $\zeta$ and index set $T$.
If constant $\zeta > 1$ satifies that $z^{*} \le \frac{\zeta - 1}{\zeta + 1}\lambda_{1}$, we have $\lambda_{1} - z^{*} \ge \frac{2}{\zeta + 1}\lambda_{1}$ and $\lambda_{1} + z^{*} \le \frac{2\zeta}{\zeta + 1}\lambda_{1}$. Hence,
\begin{equation}\label{3-2}
    \frac{2}{\zeta+1}\lambda_{1}\|\widehat{\vthe}_{T^{c}} - {\vthe}_{T^{c}}^*\|_{1}\le D^{s}_{g}(\widehat{\vthe},\vthe^{*})+\frac{2}{\zeta+1}\lambda_{1}\|\widehat{\vthe}_{T^{c}} - {\vthe}_{T^{c}}^*\|_{1}\le
    \frac{2\zeta}{\zeta+1}\lambda_{1}\|\widehat{\vthe}_{T} - {\vthe}_{T}^*\|_{1},
\end{equation}
which implies that $\widehat{\vthe} - \vthe^*$ is within the cone $\mathcal{C}(\zeta,T)$ on the event $\Omega_{1} := {z^{*} \le \frac{\zeta - 1}{\zeta + 1}\lambda_{1}}$ \cite{van2008high,meier2009high,buhlmann2013statistical}. Another essential aspect for deriving $\ell_{1}$-consistency is the compatibility factor \cite{bickel2009simultaneous,van2009conditions,raskutti2010restricted}
\begin{equation}\label{3-3}
F(\zeta,T,\Sigma):=\inf_{\bm{0}\neq\vthe\in\mathcal{C}(\zeta,T)}\frac{s(\vthe^{\top}\Sigma\vthe)}{\|\vthe_{T}\|_{1}^{2}},
\end{equation}
where $\Sigma$ is a $p \times p$ non-negative-definite matrix. Denote 
$H^{*}:= \cov_{\vthe^{*}}(\varphi(\X))$ as the variance of $\varphi(\X)$ under the true parameter $\vthe^*$, $\overline{F}(\zeta,T) := F\big(\zeta,T,\nabla^{2}\mathcal{L}_{n}^{m}(\vthe^{*}) \big)$, and $F(\zeta,T):=F(\zeta,T,H^{*})$.
To demonstrate the desired consistency, we need the following technical assumptions.

\begin{ass}\label{3-4}
     The link function $\varphi(\cdot)$  is continuous and bounded, with the existence of a constant $K > 0$ such that $\| \varphi(\y) \|_{\infty}\leq K$ for any $\y \in \mathcal{X}$.
\end{ass}
\begin{ass}\label{3-5}
    $F(\zeta,T,H^{*})$ has a uniform positive lower bound such that 
    $
        \inf_{T \subseteq [p]} \inf_{\zeta > 1} F(\zeta,T,H^{*})\ge C_{\min}
    $
    for some $C_{\min}>0$.
\end{ass}
\begin{ass}\label{3-6}
    The $\ell_{\infty}$-norm of $\vthe^{*}$ is bounded: $\|\vthe^{*}\|_{\infty}\le B$ for some $B>0$.
\end{ass}


Lemma \ref{3-8} controls $\|\widecheck{\vthe}\|_{1} = \| \widehat{\vthe}-\vthe^{*} \|$ by $\overline{F}(\zeta,T)$.
\begin{lemma}\label{3-8}
    Suppose Assumption \ref{3-4} hold. Define the event $
        \mathcal{P}=\left\{\frac{(\zeta+1)s\lambda_{1}}{2\overline{F}(\zeta,T)}\le\frac{1}{4K \e} \right\}.$
    On the event $\mathcal{P} \, \cap \, \Omega_{1}$, it holds that
    \begin{equation}\label{3-8-1}
        \| \widehat{\vthe}-\vthe^{*} \|\le \frac{\e(\zeta+1)s\lambda_{1}}{2\overline{F}(\zeta,T)}.
    \end{equation}
\end{lemma}

We need the following results of $\nabla\mathcal{L}_{n}^{m}(\vthe^{*})$ and $\nabla^{2}\mathcal{L}_{n}^{m}(\vthe^{*})$ to prove consistency.


\begin{lemma}\label{3-9}
     Suppose Assumption \ref{3-4} holds. For any $t \geq 0$, we have
    \begin{equation*}
        \pr \big(\|\nabla\mathcal{L}_{n}^{m}(\vthe^{*} \big)\|_{\infty}\le t)\ge 1-2p \e^{-\frac{n t^{2}}{8K^{2}}} - 2p \e^{-\frac{mt^{2}}{16K(8K\beta_{1}M_{1}+t\beta_{2}M_{2})}} - \e^{-\frac{m}{4(2\beta_{1}M_{1}+\beta_{2}M_{2})}},
    \end{equation*}
    and
    \begin{equation*}
        \pr \big(\|\nabla^{2}\mathcal{L}_{n}^{m}(\vthe^{*})-H^{*}\|_{\infty}\le t \big)\ge 1-2\exp\left\{-\frac{m}{4(2\beta_{1}M_{1}+\beta_{2}M_{2})}\right\}-2p^{2}\e^{-\frac{m t^{2}}{16 K^{2}( 8K^{2}\beta_{1}M_{1}+ t\beta_{2}M_{2} )}}-2p \e^{-\frac{mt^{2}}{64K^{3}(8K\beta_{1}M_{1}+t\beta_{2}M_{2})}}.
    \end{equation*}
\end{lemma}
\noindent The above lemma implies that $\big\|\nabla\mathcal{L}_{n}^{m}(\vthe^{*})\big\|_{\infty}$ and $\big\|\nabla^2\mathcal{L}_{n}^{m}(\vthe^{*})-H^{*}\big\|_{\infty}$ are stochastically bounded, i.e.,
\begin{equation}\label{3-11}
    \big\|\nabla\mathcal{L}_{n}^{m}(\vthe^{*})\big\|_{\infty}=O_{p}\left(\sqrt{\frac{\log p}{n}}+\sqrt{\frac{\log p}{m}} + \frac{\log p}{m}\right) \text{ and } \big\|\nabla^2\mathcal{L}_{n}^{m}(\vthe^{*})-H^{*}\big\|_{\infty}=O_{p}\left(\sqrt{\frac{\log p}{m}} + \frac{\log p}{m}\right).
\end{equation}


\begin{theorem}\label{thm-1}
Suppose Assumptions \ref{3-4}-\ref{3-6} hold. Let $\tau_{1}=\frac{\zeta-1}{\zeta+1}\lambda_{1}-2\lambda_{2}B$ and $\tau_{2}=\frac{C_{\min}}{2s(\zeta+1)^{2}}$ with constant $\zeta>1$ satisfying $\frac{\zeta-1}{\zeta+1}\lambda_{1}-2\lambda_{2}B>0$ and $\frac{(\zeta+1)s\lambda_{1}}{C_{\min}}\le \frac{1}{4K \e}$.
We have
\begin{equation*}
    \big\|\widehat{\vthe}-\vthe^{*}\big\|_{1}\le \e C_{\min}^{-1}(\zeta+1)s\lambda_{1},
\end{equation*}
with a probability of at least $1-\delta_{1}-\delta_{2}$, where
$
    \delta_{1} = 2p \e^{-\frac{n \tau_{1}^{2}}{8K^{2}}}+2p \e^{-\frac{m\tau_{1}^{2}}{16K(8K\beta_{1}M_{1}+\tau_{1}\beta_{2}M_{2})}} + \e^{-\frac{m}{4(2\beta_{1}M_{1}+\beta_{2}M_{2})}}$
and
$
    \delta_{2} = 2\e^{-\frac{m}{4(2\beta_{1}M_{1}+\beta_{2}M_{2})}}+2p^{2} \e^{-\frac{m \tau_{2}^{2}}{16 K^{2}( 8K^{2}\beta_{1}M_{1}+ \tau_{2}\beta_{2}M_{2} )}}+2p \e^{-\frac{m\tau_{2}^{2}}{64K^{3}(8K\beta_{1}M_{1}+\tau_{2}\beta_{2}M_{2})}}$.
\end{theorem}

\noindent Theorem \ref{thm-1} confirms the consistency of the estimator $\widehat{\theta}$.Set $\lambda_{2}=\frac{1}{4B}\frac{\zeta-1}{\zeta+1}\lambda_{1}$, which yields $\tau_{1}=\frac{\zeta-1}{2(\zeta+1)}\lambda_{1}$. By choosing a constant $r_{1}>1$ and setting
\begin{equation}\label{lam-condi-1}
    \lambda_{1} > \frac{4K(\zeta+1)}{\zeta-1}\sqrt{\frac{2r_{1}\log p}{n}},
\end{equation}
it follows that $\e^{\frac{n\tau_{1}^{2}}{8K^{2}}}>p^{r_{1}}$, and thus we have $2p \e^{-\frac{n \tau_{1}^{2}}{8K^{2}}}\rightarrow 0$. Note that 
\begin{equation}\label{lam-condi-2}
    \max\left\{ \frac{32K(\zeta+1)}{\zeta-1}\sqrt{\frac{\beta_{1}M_{1}r_{1}\log p}{m}}, \frac{4K(\zeta+1)}{\zeta-1}\sqrt{\frac{2r_{1}\log p}{n}}\right\} <  \lambda_{1} < \frac{16K(\zeta+1)}{\zeta-1}\frac{\beta_{1}M_{1}}{\beta_{2}M_{2}}
 \end{equation}
ensures $8K\beta_{1}M_{1}>\tau_{1}\beta_{2}M_{2}$, and then we have 
$
    \e^{\frac{m\tau_{1}^{2}}{16K(8K\beta_{1}M_{1}+\tau_{1}\beta_{2}M_{2})}}> \e^{\frac{m\tau_{1}^{2}}{256K^{2}\beta_{1}M_{1}}}> p^{r_{1}},
$
and  $2p \e^{-\frac{m\tau_{1}^{2}}{16K(8K\beta_{1}M_{1}+\tau_{1}\beta_{2}M_{2})}}\to 0$. Furthermore, when we have
\begin{equation}\label{m-p}
    m>\max\{ 16K^{2}(8K^{2}\beta_{1}M_{1}+\tau_{2}\beta_{2}M_{2}) , 64K^{3}(8K\beta_{1}M_{1}+\tau_{2}\beta_{2}M_{2})\} \cdot r_{2} \cdot \log p
\end{equation}
for some constant $r_{2}>2$, it holds
$
    2p^{2} \e^{-\frac{m \tau_{2}^{2}}{16 K^{2}( 8K^{2}\beta_{1}M_{1}+ \tau_{2}\beta_{2}M_{2} )}}<2p^{2-r_{2}}\to 0
    \quad\text{and}\quad
    2p \e^{-\frac{m\tau_{2}^{2}}{64K^{3}(8K\beta_{1}M_{1}+\tau_{2}\beta_{2}M_{2})}}<2p^{1-r_{2}}\to 0.
$
Therefore, the other components of $\delta_{1}$ and $\delta_{2}$ also converge to $0$.

Indeed, if we set aside the probability term $1 - \delta_1 - \delta_2$ in Theorem \ref{thm-1}, the theorem implies that $\|\widehat{\vthe}-\vthe^{*}\|_{1}\lesssim s\lambda_{1}$. For $\ell_{1}$-consistency, it's necessary that $\lambda_{1}=o(1)$, which translates to $\sqrt{\frac{\log p}{n}}=o(1)$ and $\sqrt{\frac{\log p}{m}}=o(1)$ as per conditions \eqref{lam-condi-1} and \eqref{lam-condi-2}.
Given that $\sqrt{\frac{\log p}{m}}=o(1)$, the condition in \eqref{m-p} is naturally satisfied. Consequently, in this scenario, both $\delta_{1}$ and $\delta_{2}$ approach zero. This result confirms the consistency of the estimator $\widehat{\vthe}$ under these conditions, indicating that as the sample size increases (for both $n$ and $m$), the probability of the estimator deviating significantly from the true value $\vthe^{*}$ diminishes. We state this fact in the following corollary.



\begin{col}\label{3-14}
    Under the assumption that $s=O(1)$ and the parameters satisfy $\lambda_{1} \asymp \bigg(\sqrt{\frac{\log p}{n}}+\sqrt{\frac{\log p}{m}} + \frac{\log p}{m}\bigg)$, $\lambda_{1}\asymp\lambda_{2}$, and $\lambda_{1}=o(1)$, and given that Assumptions \ref{3-4}, \ref{3-5}, and \ref{3-6} hold, the $\ell_1$-error $\|\widehat{\vthe}-\vthe^{*}\|_{1}$ is stochastically bounded as follows:
    \begin{equation*}
       \|\widehat{\vthe}-\vthe^{*}\|_{1} =O_{p}\left(s\left(\sqrt{\frac{\log p}{n}}+\sqrt{\frac{\log p}{m}} + \frac{\log p}{m} \right)\right) = o_p(1).
    \end{equation*}
\end{col}

\noindent Corollary \ref{3-14} shows that under the assumption $\lambda_{1}\asymp\lambda_{2}=o(1)$, the $\ell_1$-error of the bias $\widehat{\vthe}$ is asymptotically negligible. Specifically, we have $|\widehat{\vthe} - \vthe^*|{1}=o{p}(1).$ This result implies that consistency for high-dimensional estimators requires sparsity of $\vthe^{*}$ and the standard assumption $\sqrt{\frac{\log p}{n}}=o(1).$ Additionally, when using Monte Carlo methods, ensuring the Markov chain convergence rate satisfies $\sqrt{\frac{\log p}{m}}=o(1)$ is crucial.


\begin{lemma}\label{3-13}
    Given $\overline{F}(\zeta,T)$ and $F(\zeta,T)$ 
    it holds
    \begin{equation*}
        \overline{F}(\zeta,T)\ge F(\zeta,T)-s(\zeta+1)^{2} \big\|\nabla^{2}\mathcal{L}_{n}^{m} (\vthe^*)-H^{*}\big\|_{\infty}.
    \end{equation*}
\end{lemma}

\section{Proofs of Theorem and Lemmas in Section \ref{sec-3}}\label{App-A}
\noindent \textbf{Proof of Lemma \ref{3-8}:}\label{App-A.1}
\begin{proof}

Define $\widecheck{\vthe} = \widehat{\vthe} - \vthe^*$, $\vthe^{\dagger}=\frac{\widecheck{\vthe}}{\|\widecheck{\vthe}\|_{1}}$,  and 
$
    g(t)=(\vthe^{\dagger})^{\top}\big[\nabla\mathcal{L}_{n}^{m}(\vthe^{*}+t\vthe^{\dagger})-\nabla\mathcal{L}_{n}^{m}(\vthe^{*}) \big].
$
Note that $\vthe^{\dagger}\in\mathcal{C}(\zeta,T)$, $\|\vthe^{\dagger}\|_{1}=1$, and $g(t)$ is non-decreasing for $t \geq 0$ as $\nabla^{2}\mathcal{L}_{n}^{m}$ is convex. By Lemma \ref{2-11} and \eqref{3-2}, for $t\in(0,\|\widecheck{\vthe}\|_{1})$, we have
\begin{equation}\label{A.1-2}
\begin{aligned}
g(t) \le g(\|\widecheck{\vthe}\|_{1})
&=(\vthe^{\dagger})^{\top} \big[ \nabla\mathcal{L}_{n}^{m} (\vthe^{*}+\|\widetilde{\vthe}\|_{1}\vthe^{\dagger})-\nabla\mathcal{L}_{n}^{m}(\vthe^{*}) \big] \\
&=\|\widecheck{\vthe}\|_{1}^{-1} \Big[ (\widehat{\vthe}-\vthe^{*})^{\top}\big( \nabla\mathcal{L}_{n}^{m}(\widehat{\vthe})-\nabla\mathcal{L}_{n}^{m}(\vthe^{*}) \big) \Big] \\
& \le \|\widecheck{\vthe}\|_{1}^{-1}D^{s}(\widehat{\vthe},\vthe^{*}) \\
&\le \|\widecheck{\vthe}\|_{1}^{-1}\left(\frac{2\zeta}{\zeta+1}\lambda_{1}\|\widecheck{\vthe}_{T}\|_{1}-\frac{2}{\zeta+1}\lambda_{1}\|\widecheck{\vthe}_{T^{c}}\|_{1} \right) \\
&=\frac{2\zeta}{\zeta+1}\lambda_{1}\|\vthe^{\dagger}_{T}\|_{1}-\frac{2}{\zeta+1}\lambda_{1}\|\vthe^{\dagger}_{T^{c}}\|_{1},
\end{aligned}
\end{equation}
where $D^s (\cdot, \cdot)= D^s_{\|\cdot \|_2} (\cdot, \cdot)$.
Let $\widetilde{t}$ be the maximal value satsfiyig 
$
    g(t)\le\frac{2\zeta}{\zeta+1}\lambda_{1}\|\vthe^{\dagger}_{T}\|_{1}-\frac{2}{\zeta+1}\lambda_{1}\|\vthe^{\dagger}_{T^{c}}\|_{1}.
$
Note that
$
    w_{i}(\vthe^{*}+t\vthe^{\dagger})=\frac{\e^{(\vthe^{*}+t\vthe^{\dagger})^{\top} \varphi\left(\Y_{i}\right)}}{h\left(\Y_{i}\right)} = w_{i}(\vthe^{*})\e^{t(\vthe^{\dagger})^{\top} \varphi\left(\Y_{i}\right)}.
$
According to \eqref{2-19}, we have
\begin{align*}
\nabla\mathcal{L}_{n}^{m}(\vthe^{*}+t\vthe^{\dagger})-\nabla\mathcal{L}_{n}^{m}(\vthe^{*})
&=\frac{\sum_{i=i}^{m}w_{i}(\vthe^{*}+t\vthe^{\dagger})\varphi(\Y_{i})}{\sum_{i=i}^{m}w_{i}(\vthe^{*}+t\vthe^{\dagger})}-\frac{\sum_{i=i}^{m}w_{i}(\vthe^{*})\varphi(\Y_{i})}{\sum_{i=i}^{m}w_{i}(\vthe^{*})} \notag\\
&=\frac{\sum_{i=i}^{m}\e^{t(\vthe^{\dagger})^{\top} \varphi\left(\Y_{i}\right)}w_{i}(\vthe^{*})\varphi(\Y_{i})}{\sum_{i=i}^{m}\e^{t(\vthe^{\dagger})^{\top} \varphi\left(\Y_{i}\right)}w_{i}(\vthe^{*})}-\frac{\sum_{i=i}^{m}w_{i}(\vthe^{*})\varphi(\Y_{i})}{\sum_{i=i}^{m}w_{i}(\vthe^{*})}.
\end{align*}
Let $a_{i}:=t [\vthe^{\dagger} ]^{\top}(\varphi(\Y_{i})-\overline{\varphi}(\vthe^{*}))$ and $w_{i}=w_{i}(\vthe^{*})$, we have
\begin{align*}
t\vthe^{\dagger} \big[ \nabla\mathcal{L}_{n}^{m}(\vthe^{*}+t\vthe^{\dagger})-\nabla\mathcal{L}_{n}^{m}(\vthe^{*}) \big]
&=\frac{\sum_{i=i}^{m}w_{i}(\vthe^{*}) \e^{t(\vthe^{\dagger})^{\top} \varphi\left(\Y_{i}\right)}t(\vthe^{\dagger})^{\top}\varphi(\Y_{i})}{\sum_{i=i}^{m}w_{i}(\vthe^{*})\e^{t(\vthe^{\dagger})^{\top} \varphi\left(\Y_{i}\right)}}-\frac{\sum_{i=i}^{m}w_{i}(\vthe^{*})t(\vthe^{\dagger})^{\top}\varphi(\Y_{i})}{\sum_{i=i}^{m}w_{i}(\vthe^{*})}  \notag\\
&=\frac{\sum_{i=i}^{m}w_{i}\e^{a_{i}}(a_{i}+t(\vthe^{\dagger})^{\top}\overline{\varphi}(\vthe^{*}))}{\sum_{i=i}^{m}w_{i}\e^{a_{i}}}-\frac{\sum_{i=i}^{m}w_{i}(a_{i}+t(\vthe^{\dagger})^{\top}\overline{\varphi}(\vthe^{*}))}{\sum_{i=i}^{m}w_{i}} \\
&=\frac{\sum_{i=i}^{m}w_{i}\e^{a_{i}}a_{i}}{\sum_{i=i}^{m}w_{i}\e^{a_{i}}}-\frac{\sum_{i=i}^{m}w_{i}a_{i}}{\sum_{i=i}^{m}w_{i}} \\
&=\frac{\sum_{1\le i,j\le m}w_{i}w_{j}a_{i}(\e^{a_{i}}-\e^{a_{j}})}{\sum_{1\le i,j\le m}w_{i}w_{j} \e^{a_{i}}} \notag\\
&=\frac{\sum_{1\le i,j\le m}w_{i}w_{j}(a_{i}-a_{j})(\e^{a_{i}}-\e^{a_{j}})}{2\sum_{1\le i,j\le m}w_{i}w_{j}\e^{a_{i}}}.
\end{align*}
By Assumption \ref{3-4} and $\sum_{1\le i,j\le m}w_{i}w_{j}(a_{i}-a_{j})^{2}=2\sum_{1\le i\le m}w_{i}\sum_{1\le j\le m}w_{j}a_{j}^{2}$
due to $\sum_{1\le i\le m}w_{i}a_{i}=0$, we have  $|a_{i}|\le t\|\vthe^{\dagger}\|_{1}\|\varphi(\Y_{i})-\overline{\varphi}(\vthe^{*})\|_{\infty}\le 2Kt$ and 
\begin{align*}
\frac{\sum_{1\le i,j\le m}w_{i}w_{j}(a_{i}-a_{j})(\e^{a_{i}} - \e^{a_{j}})}{2\sum_{1\le i,j\le m}w_{i}w_{j}e^{a_{i}}}
&\ge \e^{-4Kt}\frac{\sum_{1\le i,j\le m}w_{i}w_{j}(a_{i}-a_{j})^{2}}{2\sum_{1\le i,j\le m}w_{i}w_{j}}\notag\\
&=\e^{-4Kt}\frac{2\sum_{1\le i\le m}w_{i}\sum_{1\le j\le m}w_{j}a_{j}^{2}}{2\sum_{1\le i\le m}w_{i}\sum_{1\le j\le m}w_{j}} \notag \\
&=\e^{-4Kt}\frac{\sum_{1\le i\le m}w_{i}a_{i}^{2}}{\sum_{1\le i\le m}w_{i}} \\
& =\e^{-4Kt}t^{2} [\vthe^{\dagger} ]^{\top}\nabla^{2}\mathcal{L}_{n}^{m}(\vthe^{*})\vthe^{\dagger},
\end{align*}
where the last step is from the reformulation of $\mathcal{L}_{n}^{m}(\vthe^{*})$ in \eqref{2-20}. Therefore, we have
\begin{equation*}
    t [\vthe^{\dagger}]^{ \top} \big[ \nabla\mathcal{L}_{n}^{m}(\vthe^{*}+t\vthe^{\dagger})-\nabla\mathcal{L}_{n}^{m}(\vthe^{*}) \big] \ge
\e^{-4Kt}t^{2}[\vthe^{\dagger}]^{\top}\nabla^{2}\mathcal{L}_{n}^{m}(\vthe^{*})\vthe^{\dagger}.
\end{equation*}
\noindent By the definition of $\overline{F}(\zeta,T)$ and \eqref{A.1-2}, we have
\begin{align*}
t \e^{-4Kt} \overline{F}(\zeta,T)\|\vthe_{T}^{\dagger}\|_{1}^{2}
&\le t \e^{-4Kt} s[\vthe^{\dagger}]^{\top}\nabla^{2}\mathcal{L}_{n}^{m}(\vthe^{*})\vthe^{\dagger}\notag\\
&\le s[\vthe^{\dagger}]^{\top}\big[\nabla\mathcal{L}_{n}^{m}(\vthe^{*}+t\vthe^{\dagger})-\nabla\mathcal{L}_{n}^{m}(\vthe^{*}) \big]\\
&\le \frac{2\zeta}{\zeta+1}s\lambda_{1}\|\vthe^{\dagger}_{T}\|_{1}-\frac{2}{\zeta+1}s\lambda_{1}\|\vthe^{\dagger}_{T^{c}}\|_{1}\notag\\
&\le \frac{\zeta+1}{2}s\lambda_{1}\|\vthe_{T}^{\dagger}\|_{1}^{2},
\end{align*}
which means 
$
    t \e^{-4Kt} \le \frac{(\zeta+1)s\lambda_{1}}{2\overline{F}(\zeta,T)}${ for any } $ t\in (0,\widetilde{t}).
$
 It is easy to verify that $y(t)=t \e^{-4Kt}$ achieves its unique maximum $y_{\max}=\frac{1}{4K \e}$ at $t_{\max}=\frac{1}{4K}$. Based on the event
$
    \mathcal{P} \, \cap \, \Omega_{1}=\left\{\frac{(\zeta+1)s\lambda_{1}}{2\overline{F}(\zeta,T)}\le\frac{1}{4K \e} \right\} \, \cap \, \Omega_{1},
$
we have $\|\widecheck{\vthe}\|_{1}\le \widetilde{t}\le \frac{1}{4K}$, 
and therefore $\|\widecheck{\vthe}\|_{1} \e^{-1}\le \widetilde{t} \e^{-4K\widetilde{t}} \le \frac{(\zeta+1)s\lambda_{1}}{2\overline{F}(\zeta,T)}.$
\end{proof}
\vspace{6ex}

\noindent \textbf{Proof of Lemma \ref{3-9}:}\label{App-A.2}

\begin{proof}
Consider $j$-th element in $\nabla\mathcal{L}_{n}^{m}(\vthe^{*})$, it can be written as
\begin{align*}
\nabla\mathcal{L}_{n,j}^{m}(\vthe^{*})&=-\frac{1}{n}\sum_{i=1}^{n}\varphi_{j}(\Y_{i})+\frac{\sum_{i=1}^{m}w
_{i}(\vthe^{*})\varphi_{j}(\Y_{i})}{\sum_{i=1}^{m}w
_{i}(\vthe^{*})} =-\frac{1}{n}\sum_{i=1}^{n}(\varphi_{j}(\Y_{i})-\operatorname{E}_{\vthe^{*}}\varphi_{j}(\X))+\frac{\sum_{i=1}^{m}w
_{i}(\vthe^{*})(\varphi_{j}(\Y_{i})-\operatorname{E}_{\vthe^{*}}\varphi_{j}(\X))}{\sum_{i=1}^{m}w
_{i}(\vthe^{*})}.
\end{align*}
Define events
\begin{align}
&A_{j}=\left\{\left|\frac{1}{n}\sum_{i=1}^{n}\varphi_{j}(\X_{i})-\operatorname{E}_{\vthe^{*}}\varphi_{j}(\X)\right|\ge\frac{t}{2}\right\}, \quad B_{j}=\left\{\left|\frac{\sum_{i=1}^{m}w
_{i}(\vthe^{*})(\varphi_{j}(\Y_{i})-\operatorname{E}_{\vthe^{*}}\varphi_{j}(\X))}{\sum_{i=1}^{m}w
_{i}(\vthe^{*})}\right|\ge\frac{t}{2}\right\}, \notag \\
&C_{j}=\left\{\left| \frac{1}{m}\sum_{i=1}^{m}w
_{i}(\vthe^{*})(\varphi_{j}(\Y_{i})-\operatorname{E}_{\vthe^{*}}\varphi_{j}(\X))\right|\ge \frac{t}{4}C(\vthe^{*}) \right\} 
\quad\text{and}\quad
D=\left\{\frac{1}{m}\sum_{i=1}^{m}w
_{i}(\vthe^{*})\le \frac{1}{2}C(\vthe^{*}) \right\}.\notag
\end{align}
Note that we have the following relationships:
\begin{itemize}
    \item[(i)] $\big\{\|\nabla\mathcal{L}_{n}^{m}(\vthe^{*})\|_{\infty}\le t\big\}=\bigcap_{j = 1}^p \{|\nabla\mathcal{L}_{n,j}^{m}(\vthe^{*})|\le t\}$;
    \item[(ii)] $\{|\nabla\mathcal{L}_{n,j}^{m}(\vthe^{*})|\ge t\}\subseteq A_{j}\cup B_{j}$;
    \item[(iii)] $B_{j} \backslash D\subseteq C_{j} \, \Longrightarrow \, B_{j}\subset C_{j}\cup D$,
\end{itemize}
which implies $\big\{\|\nabla\mathcal{L}_{n}^{m}(\vthe^{*})\|_{\infty}\ge t\big\} \, \subseteq \bigcup_{j = 1}^p A_{j} \, \cup \,  C_{j} \, \cup \,  D.$
By Hoeffding inequality and $|\varphi_{j}(Y_{i})|\le K$, we have
\begin{equation}\label{A.2-2}
   \pr (A_{j})\le 2\exp\left(-\frac{2\left(\frac{n t}{2}\right)^{2}}{\sum_{i=1}^{n}(2K)^{2}}\right) = 2\e^{-\frac{n t^{2}}{8K^{2}}}.
\end{equation}
Let
$
    g_j(\Y) :=\frac{\exp \{[\vthe^{*}]^{\top}\varphi(\Y)\}(\varphi_{j}(\Y)-\operatorname{E}_{\vthe^{*}}\varphi_{j}(\X))}{h(\Y)}.
$
Note that $|\varphi_{j}(\Y)-\operatorname{E}_{\vthe^{*}}\varphi_{j}(\X) \le |\varphi_{j}(\Y)|+\operatorname{E}_{\vthe^{*}}|\varphi_{j}(\X)| \le 2K.$ Since
\begin{equation*}
    \operatorname{E}_{\Y\sim h}g(\Y) = \int_{\y \in \mathcal{X}} \e^{[\vthe^{*}]^{\top}\varphi(\y)}\varphi_{j}(\y) \, \mathrm{d} \y  -C(\vthe^{*})\operatorname{E}_{\vthe^{*}}\varphi_{j}(\X)=0,
\end{equation*}
and  
\begin{equation*}
   \operatorname{E}_{\Y\sim h} g_j^{2}(\Y)\le 4K^{2}\operatorname{E}_{\Y\sim h}\left( \frac{\e^{(\vthe^{*})^{\top}\varphi(\Y)}}{h(\Y)} \right)^{2}=(2KC(\vthe^{*}))^{2}M_{1},
\end{equation*} 
 by Lemma \ref{2-17} with $\|g\|_{\infty}\le 2K M_{2} C(\vthe^{*})$ and $M=2KM_{2}C(\vthe^{*})$,
we obtain
\begin{align}
\pr(C_{j})
&\le 2\exp\left(-\frac{m(\frac{t}{4}C(\vthe^{*}))^{2}}{2(\beta_{1}(2KC(\vthe^{*}))^{2}M_{1}+\beta_{2}2KC(\vthe^{*})M_{2}(\frac{t}{4}C(\vthe^{*}))}\right)= 2\exp\left(-\frac{mt^{2}}{16K(8K\beta_{1}M_{1}+t\beta_{2}M_{2})}\right).  \label{A.2-5}
\end{align}

We define the function
$
q(\Y)=\frac{\e^{[\vthe^{*}]^{\top}\varphi(\Y)}}{h(\Y)}-C(\vthe^{*}).
$
Since $\E q(\Y) = 0$, $\|q\|_{\infty}\le  M_{2} C(\vthe^{*})$,
\begin{equation*}
    \operatorname{E}_{\Y\sim h} g^{2}(\Y) \le \operatorname{E}_{\Y\sim h} \left( \frac{e^{(\vthe^{*})^{\top}\varphi(\Y)}}{h(\Y)} \right)^{2}=C^{2}(\vthe^{*})M_{1}
\end{equation*}
thus by applying Lemma \ref{2-17} again, 
we obtain
\begin{align}\label{A.2-7}
\pr(D)
&\le \exp\left( -\frac{m(\frac{C(\vthe^{*})}{2})^{2}}{2(\beta_{1}C^{2}(\vthe^{*})M_{1}+\beta_{2}C(\vthe^{*})M_{2}\frac{C(\vthe^{*})}{2})} \right)=\exp\left( -\frac{m}{4(2\beta_{1}M_{1}+\beta_{2}M_{2})} \right).
\end{align}
Combining \eqref{A.2-2}, \eqref{A.2-5} and \eqref{A.2-7}, we conclude
\begin{align*}
\pr\big(\|\nabla\mathcal{L}_{n}^{m}(\vthe^{*})\|_{\infty}\le t \big)
& =1-\pr \big(\|\nabla\mathcal{L}_{n}^{m}(\vthe^{*})\|_{\infty}\ge t \big) \notag  \ge 1-2p \e^{-\frac{n t^{2}}{8K^{2}}}-2p e^{-\frac{mt^{2}}{16K(8K\beta_{1}M_{1}+t\beta_{2}M_{2})}}-\e^{-\frac{m}{4(2\beta_{1}M_{1}+\beta_{2}M_{2})}},
\end{align*}  
which gives the first result in the lemma. For the second inequality in the lemma, according to the reformulation of $\nabla^{2}\mathcal{L}_{n}^{m}(\vthe^{*})$ in \eqref{2-20} and the definition of $H^{*}=\operatorname{cov}_{\vthe^{*}}\varphi(\X)$, we have
\begin{align*}
\nabla\mathcal{L}_{n}^{m}(\vthe^{*})-H^{*}&=\left(\frac{\sum_{i=1}^{m}w_{i}(\vthe)\varphi(\Y_{i})^{\otimes 2}}{\sum_{i=1}^{m}w_{i}(\vthe)}-\operatorname{E}_{\vthe^{*}}\varphi(\X)^{\otimes 2}\right)-\left[\left(\frac{\sum_{i=i}^{m}w_{i}(\vthe)\varphi(\Y_{i})}{\sum_{i=i}^{m}w_{i}(\vthe)}\right)^{\otimes 2}-\left(\operatorname{E}_{\vthe^{*}}\varphi(\X)\right)^{\otimes 2}\right]\\
&=:F+E,
\end{align*}
where
$
F_{k l}=\frac{\frac{1}{m}\sum_{i=i}^{m}w_{i}(\vthe)\left(\varphi_{k}(\Y_{i})\varphi_{l}(\Y_{i})-\operatorname{E}_{\vthe^{*}}\varphi_{k}(\X)\varphi_{l}(\X)\right)}{\frac{1}{m}\sum_{i=i}^{m}w_{i}(\vthe)} $ for $1\le k,l\le m.
$
Similar as the proof in Lemma \ref{3-9}, we define the events
\begin{align*}
&A_{k l}:=\left\{\left|F_{k l}\right|\ge \frac{t}{2}\right\},
B_{k l}:=\left\{\left|\frac{1}{m}\sum_{i=i}^{m}w_{i}(\vthe)\big[ \varphi_{k}(\Y_{i})\varphi_{l}(\Y_{i})-\operatorname{E}_{\vthe^{*}}\varphi_{k}(\X)\varphi_{l}(\X)\big]\right|\ge \frac{t}{4}C(\vthe^{*})\right\}, D:=\left\{\frac{1}{m}\sum_{i=1}^{m}w
_{i}(\vthe^{*})\le \frac{1}{2}C(\vthe^{*}) \right\},
\end{align*}
then $A_{kl}\backslash D\subseteq B_{k l} \, \Longrightarrow \, A_{k l}\subseteq B_{k l}\cup D$. Let
$
    g_{kl}(\Y)=\frac{\e^{[\vthe^{*}]^{\top}\varphi(\Y)}(\varphi_{k}(\Y)\varphi_{l}(\Y)-\operatorname{E}_{\vthe^{*}}\varphi_{k}(\X)\varphi_{l}(\X))}{h(\Y)}.
$
Note that 
\begin{equation*}
    |\varphi_{k}(\Y)\varphi_{l}(\Y)-\operatorname{E}_{\vthe^{*}}\varphi_{k}(\X)\varphi_{l}(\X)|
    \le |\varphi_{k}(\Y)|\cdot|\varphi_{l}(\Y)|+\operatorname{E}_{\vthe^{*}}|\varphi_{k}(\Y)|\cdot|\varphi_{l}(\Y)|\le 2K^{2}.
\end{equation*}
We have
\[
    \operatorname{E}_{\Y\sim h}g_{kl}(\Y) = \int_{\y \in \mathcal{X}} \e^{[\vthe^{*}]^{\top}\varphi(\y)}\varphi_{k}(\y)\varphi_{l}(\y) \, \mathrm{d} \y -C(\vthe^{*})\operatorname{E}_{\vthe^{*}}\varphi_{k}(\x)\varphi_{l}(\x)=0,
\]
and
\[
    \operatorname{E}_{\Y\sim h}g_{kl}^{2}(\Y)\le 4K^{4}\operatorname{E}_{\Y\sim h}\left( \frac{\e^{[\vthe^{*}]^{\top}\varphi(\Y)}}{h(\Y)} \right)^{2}=(2K^{2}C(\vthe^{*}))^{2}M_{1}.
\]
By Lemma \ref{2-17} with $\|g_{kl}\|_{\infty}\le 2K^{2}C(\vthe^{*})M_{2}$, we have
\begin{align*}
\pr \left( B_{kl}\right)
&\le 2  \exp \left( -\frac{m(\frac{t}{4}C(\vthe^{*}))^{2}}{2(\beta_{1}(2K^{2}C(\vthe^{*}))^{2}M_{1}+\beta_{2}2K^{2}C(\vthe^{*})M_{2}(\frac{t}{4}C(\vthe^{*})) )} \right) = 2\exp \left( -\frac{m t^{2}}{16 K^{2}( 8K^{2}\beta_{1}M_{1}+ t\beta_{2}M_{2} )} \right).
\end{align*}
Thus,
\begin{align}\label{A.3-5}
    \pr \left(\|F\|_{\infty}\ge \frac{t}{2}\right) = \pr\left( \bigcup_{1\le k,l\le p} \big[ B_{kl} \, \cup \,  D \big] \right)\le 2p^{2} \e^{-\frac{m t^{2}}{16 K^{2}( 8K^{2}\beta_{1}M_{1}+ t\beta_{2}M_{2} )}} + \e^{-\frac{m}{4(2\beta_{1}M_{1}+\beta_{2}M_{2})}},
\end{align}
where we use the upper bound for $\pr(D)$ in the proof of Lemma \ref{3-9} in the last step. Now we turn to bound $\|E\|_{\infty}$. We rewrite its elements as
\[
\begin{aligned}
E_{kl}
&=\left(\frac{\sum_{i=1}^{m}w
_{i}(\vthe^{*})\varphi_{k}(\Y_{i})}{\sum_{i=1}^{m}w
_{i}(\vthe^{*})}\right)\left(\frac{\sum_{i=1}^{m}w
_{i}(\vthe^{*})\varphi_{l}(\Y_{i})}{\sum_{i=1}^{m}w
_{i}(\vthe^{*})}\right)-\operatorname{E}_{\vthe^{*}}\varphi_{k}(\X)\cdot \operatorname{E}_{\vthe^{*}}\varphi_{l}(\X)\notag\\
&=\left(\frac{\sum_{i=1}^{m}w
_{i}(\vthe^{*})\varphi_{k}(\Y_{i})}{\sum_{i=1}^{m}w
_{i}(\vthe^{*})}-\operatorname{E}_{\vthe^{*}}\varphi_{k}(\X)\right)\left(\frac{\sum_{i=1}^{m}w
_{i}(\vthe^{*})\varphi_{l}(\Y_{i})}{\sum_{i=1}^{m}w
_{i}(\vthe^{*})}\right)+\left(\frac{\sum_{i=1}^{m}w
_{i}(\vthe^{*})\varphi_{l}(\Y_{i})}{\sum_{i=1}^{m}w
_{i}(\vthe^{*})}-\operatorname{E}_{\vthe^{*}}\varphi_{l}(\X)\right)\operatorname{E}_{\vthe^{*}}\varphi_{k}(\X).
\end{aligned}
\]
Therefore,
\begin{align*}
\| E\|_{\infty} & = \max_{k, l \in [p]}|E_{kl}| \le K \max_{k \in [p]} \left| \frac{\sum_{i=1}^{m}w
_{i}(\vthe^{*})\varphi_{k}(\Y_{i})}{\sum_{i=1}^{m}w
_{i}(\vthe^{*})}-\operatorname{E}_{\vthe^{*}}\varphi_{k}(\X) \right|+K \max_{l \in [p]} \left| \frac{\sum_{i=1}^{m}w
_{i}(\vthe^{*})\varphi_{l}(\Y_{i})}{\sum_{i=1}^{m}w
_{i}(\vthe^{*})}-\operatorname{E}_{\vthe^{*}}\varphi_{l}(\X) \right|  \notag\\
&\le 2K \left\| \frac{\sum_{i=1}^{m}w
_{i}(\vthe^{*})\varphi(\Y_{i})}{\sum_{i=1}^{m}w
_{i}(\vthe^{*})}-\operatorname{E}_{\vthe^{*}}\varphi(\X) \right\|_{\infty}.
\end{align*}
Similar to the proof in Lemma \ref{3-9}, we establish the upper bound for $\|E\|_{\infty}$
\begin{align}\label{A.3-7}
\pr \big(\|E\|_{\infty} \ge {t} / {2} \big)
&\le \pr \left( \left\| \frac{\sum_{i=1}^{m}w
_{i}(\vthe^{*})\varphi(\Y_{i})}{\sum_{i=1}^{m}w
_{i}(\vthe^{*})}-\operatorname{E}_{\vthe^{*}}\varphi(\X) \right\|_{\infty} \ge \frac{t}{4K} \right) \le 2p \e^{-\frac{mt^{2}}{64K^{3}(8K\beta_{1}M_{1}+t\beta_{2}M_{2})}} + \e^{-\frac{m}{4(2\beta_{1}M_{1}+\beta_{2}M_{2})}}.
\end{align}
Finally, by using \eqref{A.3-5} and \eqref{A.3-7}, we have
\begin{align*}
    \pr(\|\nabla\mathcal{L}_{n}^{m}(\vthe^{*})-H^{*}\|_{\infty}\le t) &\ge 1-2p^{2}\e^{-\frac{m t^{2}}{16 K^{2}( 8K^{2}\beta_{1}M_{1}+ t\beta_{2}M_{2} )}} - \e^{-\frac{m}{4(2\beta_{1}M_{1}+\beta_{2}M_{2})}}  -2p \e^{-\frac{mt^{2}}{64K^{3}(8K\beta_{1}M_{1}+t\beta_{2}M_{2})}} - \e^{-\frac{m}{4(2\beta_{1}M_{1}+\beta_{2}M_{2})}},
\end{align*}
which concludes this result.
\end{proof}
\vspace{6ex}

\noindent \textbf{Proof of Theorem \ref{3-14}}\label{App-A.5}
\vspace{2ex}

\begin{proof}
Define the events
$
    \Omega_{2}:=\left\{ s(\zeta+1)^{2}\|\nabla^{2}\mathcal{L}_{n}^{m}(\vthe^{*})-H^{*}\|_{\infty}\le \frac{C_{\min}}{2} \right\}$ and $ \Omega_{3}:=\left\{ \frac{(\zeta+1)s\lambda_{1}}{C_{\min}}\le \frac{1}{4K \e} \right\}$.
Conditional on event $\Omega_{2}$, we have
\begin{equation*}
    \overline{F}(\zeta,T)\ge F(\zeta,T)-s(\zeta+1)^{2}\|\nabla^{2}\mathcal{L}_{n}^{m}(\vthe^{*}) - H^{*}\|_{\infty}\ge C_{\min}-\frac{C_{\min}}{2}=\frac{C_{\min}}{2},
\end{equation*}
and on the event $\Omega_{2} \, \cap \, \Omega_{3}$, we have
\begin{equation*}
    \frac{(\zeta+1)s\lambda_{1}}{2\overline{F}(\zeta,T)}\le \frac{(\zeta+1)s\lambda_{1}}{C_{\min}}\le \frac{1}{4K \e}.
\end{equation*}
Therefore, by Lemma \ref{3-8},  we have 
$
    \|\widehat{\vthe} - \vthe^* \|_{1}\le \frac{\e (\zeta+1)s\lambda_{1}}{2\overline{F}(\zeta,T)}\le 
    \frac{\e(\zeta+1)s\lambda_{1}}{C_{\min}},
$
on the event $\Omega_{1} \, \cap \, \Omega_{2} \, \cap \, \Omega_{3}$. Under Assumption \ref{3-6}, $z^{*}\le\|\nabla\mathcal{L}(\vthe^{*})\|_{\infty}+2\lambda_{2}B$. Thus by Lemma \ref{3-9} with $\tau_{1}=\frac{\zeta-1}{\zeta+1}\lambda_{1} -2\lambda_{2}B$, we have
\begin{align*}
\pr \left(\Omega_{1}^{c}\right)
& = \pr \left(\left\{z^{*}\ge \frac{\zeta-1}{\zeta+1}\lambda_{1}\right\}\right)\le
\pr \left( \|\nabla\mathcal{L}(\vthe^{*})\|_{\infty}\ge \frac{\zeta-1}{\zeta+1}\lambda_{1} -2\lambda_{2}B \right) \\
&\le 2p \e^{-\frac{n \tau_{1}^{2}}{8K^{2}}} + 2p \e^{-\frac{m\tau_{1}^{2}}{16K(8K\beta_{1}M_{1}+\tau_{1}\beta_{2}M_{2})}} + \e^{-\frac{m}{4(2\beta_{1}M_{1}+\beta_{2}M_{2})}}=:\delta_{1}.
\end{align*}
For $\Omega_{2}$, note that $\tau_{2}=\frac{C_{\min}}{2s(\zeta+1)^{2}}$, by Lemma \ref{3-10},
\begin{align*}
    \pr (\Omega_{2}^{c})
    &=\pr \left(s(\zeta+1)^{2}\|\nabla^{2}\mathcal{L}_{n}^{m}-H^{*}\|_{\infty}\ge \frac{C_{\min}}{2}\right) = \pr (\|\nabla^{2}\mathcal{L}_{n}^{m}-H^{*}\|_{\infty}\ge\tau_{2}) \\
    &\le 2 \e^{-\frac{m}{4(2\beta_{1}M_{1}+\beta_{2}M_{2})}}+2p^{2}\e^{-\frac{m \tau_{2}^{2}}{16 K^{2}( 8K^{2}\beta_{1}M_{1}+ \tau_{2}\beta_{2}M_{2} )}}+2p \e^{-\frac{m\tau_{2}^{2}}{64K^{3}(8K\beta_{1}M_{1}+\tau_{2}\beta_{2}M_{2})}}=:\delta_{2}.
\end{align*}
By $\frac{(\zeta+1)s\lambda_{1}}{C_{\min}}\le\frac{1}{4K e}$ in the theorem, we have $\pr(\Omega_{3})=1$. Thus,
\begin{equation*}
\pr \left( \|\widehat{\vthe} - \vthe^* \|_{1} \le 
    \frac{\e(\zeta+1)s\lambda_{1}}{C_{\min}} \right) \ge \pr \left(  \Omega_{1}\cap\Omega_{2}\cap\Omega_{3}\right) = \pr \left(  \Omega_{1}\cap\Omega_{2}\right)\ge 1-\pr(\Omega_{1}^{c}) - \pr(\Omega_{2}^{c}) \geq 1 - \delta_1 - \delta_2,
\end{equation*}
which completes the proof.
\end{proof}

\section{Proofs of Theorem and Lemmas in Section \ref{sec-4}}\label{App-B}
\textbf{Proof of Lemma \ref{4-1}:}\label{App-B.1}
 
\begin{proof}
Since $\var\left( \frac{1}{n}\sum_{i=1}^{n}\bm{v}^{\top}\varphi(\X_{i}) \right) = \frac{1}{n}\bm{v}^{\top}H^{*}\bm{v}$.
According to the reformulation of $\nabla\mathcal{L}_{n}^{m}(\vthe^{*})$ in \eqref{2-19},
\begin{align*}
\frac{\sqrt{n}\bm{v}^{\top}\nabla\mathcal{L}_{n}^{m}(\vthe^{*})}{\sqrt{\bm{v}^{\top}H^{*}\bm{v}}}
&=\frac{-\frac{1}{n}\sum_{i=1}^{n}\bm{v}^{\top}(\varphi(\Y_{i})-\operatorname{E}_{\vthe^{*}}\varphi(\Y))}{\sqrt{\operatorname{var}(\frac{1}{n}\sum_{i=1}^{n}\bm{v}^{\top}\varphi(\X_{i}))}}+\frac{\sqrt{n}}{\sqrt{\bm{v}^{\top}H^{*}\bm{v}}}\frac{\sum_{i=1}^{m}w_{i}(\vthe)\bm{v}^{\top}(\varphi(\Y_{i})-\operatorname{E}_{\vthe^{*}}\varphi(\Y))}{\sum_{i=1}^{m}w_{i}(\vthe)}\notag\\
&=:A+B.
\end{align*}
By the central limit theorem, we have $A \, \rightsquigarrow \, \mathcal{N}(0,1)$.
If we can show $B=o_{p}(1)$, we will get the desired result due to Slutsky's theorem.
Indeed, using the fact $\|\bm{v}\|_{1}\le \sqrt{\|\bm{v}\|_{0}} \|\bm{v}\|_{2}$, we have
\begin{equation*}
|B|\le \frac{\sqrt{n}\sqrt{\|\bm{v}\|_{0}}}{\sqrt{\lambda_{\min}}} \left\| \frac{\sum_{i=1}^{m}w_{i}(\vthe)(\varphi(\Y_{i})-\operatorname{E}_{\vthe^{*}}\varphi(\X))}{\sum_{i=1}^{m}w_{i}(\vthe)} \right\|_{\infty}
\end{equation*}
Similar to the proof Lemma \ref{3-9}, we can show that
\begin{equation*}
\pr \left( \left\| \frac{\sum_{i=1}^{m}w_{i}(\vthe)(\varphi(\Y_{i})-\operatorname{E}_{\vthe^{*}}\varphi(\X))}{\sum_{i=1}^{m}w_{i}(\vthe)} \right\|_{\infty}\ge t\right) \le 2p\e^{-\frac{mt^{2}}{8K(4K\beta_{1}M_{1}+t\beta_{2}M_{2})}}+\e^{-\frac{m}{4(2\beta_{1}M_{1}+\beta_{2}M_{2})}}.
\end{equation*}
Therefore, for any $\epsilon > 0$, we have
\begin{align*}
\pr (|B|\ge \epsilon)
& \le \pr \left( \left\| \frac{\sum_{i=1}^{m}w_{i}(\vthe)(\varphi(\Y_{i})-\operatorname{E}_{\vthe^{*}}\varphi(\X))}{\sum_{i=1}^{m}w_{i}(\vthe)} \right\|_{\infty}\ge \frac{\epsilon\sqrt{\lambda_{\min}}}{\sqrt{n}\sqrt{\|\bm{v}\|_{0}}} \right) \\
&\le 2p\exp\left(-\frac{\lambda_{\min}\epsilon^{2}m}{8K\|\bm{v}\|_{0}n\left( 4K\beta_{1}M_{1}+\beta_{2}M_{2}\left(\frac{\epsilon\sqrt{\lambda_{\min}}}{\sqrt{n\|\bm{v}\|_{0}}}\right)\right)} \right)+\exp\left(-\frac{m}{4(2\beta_{1}M_{1}+\beta_{2}M_{2})}\right) \\
&=2pI_{1}+I_{2}.
\end{align*}
With the assumption $\|\bm{v}\|_{0}=O(1)$ and $\frac{m}{n} \gtrsim \log p$, there exists a constant $r_{3}>1$ such that
\[
    \frac{\lambda_{\min}\epsilon^{2}m}{8K\|\bm{v}\|_{0}n\left( 4K\beta_{1}M_{1}+\beta_{2}M_{2} \right)}>r_{3}\log p,
\]
for sufficiently large $n$.
Thus, we have $2p I_{1}\le 2p^{1-r_{3}} \to 0$ and $B=o_{p}(1)$. 
\end{proof}
\vspace{6ex}

\noindent \textbf{Proof of Lemma \ref{4-4}}\label{App-B.2}
\begin{proof}
Define the function $M :\bm{w} \, \mapsto \, \frac{1}{2}\bm{w}^{\top}\nabla^{2}_{\vbeta\vbeta}\mathcal{L}_{n}^{m}(\widehat{\vthe})\bm{w}-\bm{w}^{\top}\nabla^{2}_{\alpha\vbeta}\mathcal{L}_{n}^{m}(\widehat{\vthe})+\lambda^{\prime}\|\bm{w}\|_{1}$ mapping from $\mathbb{R}^p$ to $\mathbb{R}$. Let $\widehat{\boldsymbol{\Delta}}:=\widehat{\bm{w}}-\bm{w}^{*}$. Note that $\widehat{\bm{w}}$ minimizes $M(\bm{w})$, and thus we have $M(\widehat{\bm{w}})\le M(\bm{w}^{*})$, which implies 
\[
    \begin{aligned}
        \frac{1}{2}\widehat{\boldsymbol{\Delta}}^{\top}\nabla^{2}_{\vbeta \vbeta}\mathcal{L}_{n}^{m}(\widehat{\vthe})\widehat{\boldsymbol{\Delta}} & \leq \widehat{\boldsymbol{\Delta}}^{\top}\nabla^{2}_{\alpha\vbeta}\mathcal{L}_{n}^{m}(\widehat{\vthe})-\widehat{\boldsymbol{\Delta}}^{\top}\nabla^{2}_{\vbeta\vbeta}\mathcal{L}_{n}^{m}(\widehat{\vthe})\bm{w}^{*}+\lambda^{\prime}\|\bm{w}^{*}\|_{1}-\lambda^{\prime}\|\widehat{\bm{w}}\|_{1} \\
        & \leq \widehat{\boldsymbol{\Delta}}^{\top}\big(\nabla^{2}_{\alpha\vbeta}\mathcal{L}_{n}^{m}(\vthe^{*})-\nabla^{2}_{\vbeta\vbeta}\mathcal{L}_{n}^{m}(\vthe^{*})\bm{w}^{*}\big)+\lambda^{\prime}\|\bm{w}^{*}\|_{1}-\lambda^{\prime}\|\widehat{\bm{w}}\|_{1} \\
        & ~~~~~~~ + \widehat{\boldsymbol{\Delta}}^{\top}\big(\nabla^{2}_{\alpha\vbeta}\mathcal{L}_{n}^{m}(\widehat{\vthe})-\nabla^{2}_{\alpha\vbeta}\mathcal{L}_{n}^{m}(\vthe^{*}) \big) - \widehat{\boldsymbol{\Delta}}^{\top}(\nabla^{2}_{{\vbeta\vbeta}}\mathcal{L}_{n}^{m}(\widehat{\vthe})-\nabla^{2}_{\alpha\vbeta}\mathcal{L}_{n}^{m}(\vthe^{*}))\bm{w}^{*}\\
        &:=I_{1}+I_{2}+I_{3}-I_{4},
    \end{aligned}
\]
where $I_{1}=\widehat{\boldsymbol{\Delta}}^{\top}\big(\nabla^{2}_{\alpha\vbeta}\mathcal{L}_{n}^{m}(\vthe^{*})-\nabla^{2}_{\vbeta\vbeta}\mathcal{L}_{n}^{m}(\vthe^{*})\bm{w}^{*}\big)$, $I_{2}=\lambda^{\prime}\|\bm{w}^{*}\|_{1}-\lambda^{\prime}\|\widehat{\bm{w}}\|_{1}$, $I_{3} = \widehat{\boldsymbol{\Delta}}^{\top}\big(\nabla^{2}_{\alpha\vbeta}\mathcal{L}_{n}^{m}(\widehat{\vthe})-\nabla^{2}_{\alpha\vbeta}\mathcal{L}_{n}^{m}(\vthe^{*}) \big)$, and $I_{4}=\widehat{\boldsymbol{\Delta}}^{\top}(\nabla^{2}_{{\vbeta\vbeta}}\mathcal{L}_{n}^{m}(\widehat{\vthe})-\nabla^{2}_{\alpha\vbeta}\mathcal{L}_{n}^{m}(\vthe^{*}))\bm{w}^{*}$.
For $I_{1}$, we have
\begin{align}\label{B.2-1}
|I_{1}|
&\le \|\widehat{\boldsymbol{\Delta}}\|_{1}\cdot\big\|\nabla^{2}_{\alpha\vbeta}\mathcal{L}_{n}^{m}(\vthe^{*})-\nabla^{2}_{\vbeta\vbeta}\mathcal{L}_{n}^{m}(\vthe^{*})\bm{w}^{*} \big\|_{\infty} \notag \\
&= \big\|\widehat{\boldsymbol{\Delta}}\|_{1}\cdot\|(\nabla^{2}_{\alpha\vbeta}\mathcal{L}_{n}^{m}(\vthe^{*})-H^{*}_{\alpha\vbeta})-(\nabla^{2}_{\vbeta\vbeta}\mathcal{L}_{n}^{m}(\vthe^{*})-H^{*}_{\vbeta\vbeta})\bm{w}^{*} \big\|_{\infty} \notag \\
&\le \|\widehat{\boldsymbol{\Delta}}\|_{1}\cdot\big(\|\nabla^{2}_{\alpha\vbeta}\mathcal{L}_{n}^{m}(\vthe^{*})-H^{*}_{\alpha\vbeta} \|_{\infty}+s^{\prime} B \| \nabla^{2}_{\vbeta\vbeta}\mathcal{L}_{n}^{m}(\vthe^{*})-H^{*}_{\vbeta\vbeta}\|_{\infty}\big) \notag\\
& \le C \|\widehat{\boldsymbol{\Delta}}\|_{1} \cdot \| \nabla^{2}\mathcal{L}_{n}^{m}(\vthe^{*})-H^{*}    \|_{\infty} \notag\\
& \lesssim \left(\sqrt{\frac{\log p}{n}}+\sqrt{\frac{\log p}{m}} \right)\|\widehat{\boldsymbol{\Delta}}\|_{1},
\end{align}
where the last inequality holds by  \eqref{3-11}. For $I_{2}$, let $S \subseteq [p]$ be the support of $\bm{w}^{*}$. By the fact $\|\bm{w}^{*}_{S} \|_{1}-\|\widehat{\bm{w}}_{S}\|_{1}\le \|\widehat{\boldsymbol{\Delta}}_{S}\|_{1}$ and $\widehat{\boldsymbol{\Delta}}_{S^{c}}=\widehat{\bm{w}}_{S^{c}}$, we have
\begin{align}\label{B.2-2}
|I_{2}| &=\lambda^{\prime}\|\bm{w}^{*}\|_{1}-\lambda^{\prime}\|\widehat{\bm{w}}\|_{1} =\lambda^{\prime}\|\bm{w}^{*}_{S} \|_{1}-\lambda^{\prime}\|\widehat{\bm{w}}_{S}\|_{1}-\lambda^{\prime}\|\widehat{\bm{w}}_{S^{c}}\|_{1}\le \lambda^{\prime}\|\widehat{\boldsymbol{\Delta}}_{S}\|_{1}-\lambda^{\prime}\|\widehat{\boldsymbol{\Delta}}_{S^{c}}\|_{1}.
\end{align}

Next for $I_{3}$,
according to the reformulation of $\nabla^{2}\mathcal{L}_{n}^{m}(\vthe)$ in \eqref{2-20}, we rewrite 
$\nabla^{2}_{\alpha\vbeta}\mathcal{L}_{n}^{m}(\widehat{\vthe})$ as
\begin{equation*}
\nabla^{2}_{\alpha\vbeta}\mathcal{L}_{n}^{m}(\widehat{\vthe})=\frac{\sum_{i=1}^{m}w_{i}(\widehat{\vthe})(\varphi_{\alpha}(\Y_{i})-\overline{\varphi}_{\alpha})(\varphi_{\vbeta}(\Y_{i})-\overline{\varphi}_{\vbeta})}{\sum_{i=1}^{m}w_{i}(\widehat{\vthe})},
\end{equation*}
where $\varphi(\Y_{i})=(\varphi_{\alpha}(\Y_{i}),\varphi_{\vbeta}(\Y_{i})^{\top})^{\top}$ and $\overline{\varphi}=(\overline{\varphi}_{\alpha},\overline{\varphi}_{\vbeta}^{\top})^{\top}$.
Define $g_{i}=\widecheck{\vthe}^{\top}(\varphi(\Y_i) -\overline{\varphi})$, $w_{i}=w_{i}(\vthe^{*}) $, $b_{i}=\varphi_{\alpha}(\Y_{i})-\overline{\varphi}_{\alpha}$, and $ h_{i}:=\widehat{\boldsymbol{\Delta}}^{\top}(\varphi_{\vbeta}(\Y_{i})-\overline{\varphi}_{\vbeta})$.
Note that
$
w_{i}(\widehat{\vthe})=\frac{\e^{\widecheck{\vthe}^{\top}\varphi(\Y_{i})}\e^{{\vthe}^{*\top}\varphi(\Y_{i})}}{h(\Y_{i})}= \e^{\widecheck{\vthe}^{\top}\varphi(\Y_{i})}w_{i}(\vthe^{*}),
$
then we reformulate $\widehat{\boldsymbol{\Delta}}^{\top}\nabla^{2}_{\alpha\vbeta}\mathcal{L}_{n}^{m}(\widehat{\vthe})$ and $\widehat{\boldsymbol{\Delta}}^{\top}\nabla^{2}_{\alpha\vbeta}\mathcal{L}_{n}^{m}(\vthe^{*})$ as
$
\widehat{\boldsymbol{\Delta}}^{\top}\nabla^{2}_{\alpha\vbeta}\mathcal{L}_{n}^{m}(\widehat{\vthe})=\frac{\sum_{i=1}^{m}w_{i}b_{i}h_{i}\e^{g_{i}}}{\sum_{i=1}^{m}w_{i}\e^{g_{i}}}$, and $
\widehat{\boldsymbol{\Delta}}^{\top}\nabla^{2}_{\alpha\vbeta}\mathcal{L}_{n}^{m}(\vthe^{*})=\frac{\sum_{i=1}^{m}w_{i}b_{i}h_{i}}{\sum_{i=1}^{m}w_{i}},
$
respectively. Thus, for $I_{3}$, we have
\begin{align*}
|I_{3}|
&\le \left| \frac{\sum_{i=1}^{m}w_{i}b_{i}h_{i}(\e^{g_{i}}-1)}{\sum_{i=1}^{m}w_{i}} \right|+\left| \left(\sum_{i=1}^{m}w_{i}b_{i}h_{i} \e^{g_{i}}\right)\left(\frac{1}{\sum_{i=1}^{m}w_{i}\e^{g_{i}}}-\frac{1}{\sum_{i=1}^{m}w_{i}}\right) \right|=:I_{31}+I_{32}.
\end{align*}
Since $\|\varphi \|_{\infty} \leq K$, by Cauchy' inequality
\[
    \begin{aligned}
        |I_{31}|&\le 2K \left| \frac{\sum_{i=1}^{m}w_{i}h_{i}(\e^{g_{i}}-1)}{\sum_{i=1}^{m}w_{i}} \right| \le 2K \frac{\sqrt{\sum_{i=1}^{m}w_{i}h_{i}^{2}}}{\sqrt{\sum_{i=1}^{m}w_{i}}}\cdot\frac{\sqrt{\sum_{i=1}^{m}w_{i}(\e^{g_{i}}-1)^{2}}}{\sqrt{\sum_{i=1}^{m}w_{i}}} \lesssim \frac{\sqrt{\sum_{i=1}^{m}w_{i}h_{i}^{2}}}{\sqrt{\sum_{i=1}^{m}w_{i}}}\cdot\frac{\sqrt{\sum_{i=1}^{m}w_{i}g_{i}^{2}}}{\sqrt{\sum_{i=1}^{m}w_{i}}},
    \end{aligned}
\]
where last inequality holds by the fact that $e^{g_{i}}-1\asymp g_{i}$, $g_{i} \asymp \widecheck{\vthe} = o_p(1)$. For the above bound, we rewrite
\begin{align*}
\frac{\sum_{i=1}^{m}w_{i}g_{i}^{2}}{\sum_{i=1}^{m}w_{i}}
&=\widecheck{\vthe}^{\top}\nabla^{2}\mathcal{L}_{n}^{m}(\vthe^{*})\widecheck{\vthe} \le \widecheck{\vthe}^{\top} H^{*}\widecheck{\vthe} + \big|\widecheck{\vthe}^{\top}(\nabla^{2}\mathcal{L}_{n}^{m}(\vthe^{*})-H^{*})\widecheck{\vthe}\big| \le \lambda_{\max}\|\widecheck{\vthe}\|_{1}^{2}+\|\nabla^{2}\mathcal{L}_{n}^{m}(\vthe^{*})-H^{*}\|_{\infty} \|\widecheck{\vthe}\|_{1}^{2} \lesssim \|\widecheck{\vthe}\|_{1}^{2}.
\end{align*}
By Corollary \ref{3-14}, 
\begin{equation}\label{B.2-3}
    \begin{aligned}
        |I_{31}| & \lesssim \|\widecheck{\vthe}\|_{1} \sqrt{\frac{\sum_{i=1}^{m}w_{i}h_{i}^{2}}{\sum_{i=1}^{m}w_{i}}}  \lesssim s\left(\sqrt{\frac{\log p}{n}}+\sqrt{\frac{\log p}{m}}+ \frac{\log p}{m}\right)\sqrt{\widehat{\boldsymbol{\Delta}}^{\top}\nabla^{2}_{\vbeta\vbeta}\mathcal{L}_{n}^{m}(\vthe^{*})\widehat{\boldsymbol{\Delta}}}.
    \end{aligned}
\end{equation}
Similarly for $I_{32}$, we have
\begin{equation}\label{B.2-4}
    \begin{aligned}
        I_{32} &=\left| \frac{\sum_{i=1}^{m}w_{i}b_{i}h_{i} \e^{g_{i}}}{\sum_{i=1}^{m}w_{i}\e^{g_{i}}} \right|\cdot\left| \frac{\sum_{i=1}^{m}w_{i}(\e^{g_{i}}-1)}{\sum_{i=1}^{m}w_{i}} \right|  \lesssim \sqrt{\frac{\sum_{i=1}^{m}w_{i}h_{i}^{2}}{\sum_{i=1}^{m}w_{i}}}\cdot\sqrt{\frac{\sum_{i=1}^{m}w_{i}g_{i}^{2}}{\sum_{i=1}^{m}w_{i}}} \lesssim s\left(\sqrt{\frac{\log p}{n}}+\sqrt{\frac{\log p}{m}} + \frac{\log p}{m}\right)\sqrt{\widehat{\boldsymbol{\Delta}}^{\top}\nabla^{2}_{\vbeta\vbeta}\mathcal{L}_{n}^{m}(\vthe^{*})\widehat{\boldsymbol{\Delta}}}.
    \end{aligned}
\end{equation}
Therefore, according to \eqref{B.2-3} and \eqref{B.2-4}, we have
\begin{equation}\label{B.2-5}
|I_{3}|\lesssim s\left(\sqrt{\frac{\log p}{n}}+\sqrt{\frac{\log p}{m}} + \frac{\log p}{m}\right)\sqrt{\widehat{\boldsymbol{\Delta}}^{\top}\nabla^{2}_{\vbeta\vbeta}\mathcal{L}_{n}^{m}(\vthe^{*})\widehat{\boldsymbol{\Delta}}}.    
\end{equation}
Combining \eqref{B.2-1}, \eqref{B.2-2} and \eqref{B.2-5},  we get the conclusion with a similar proof of Lemma 1 in \cite{fang2017testing}.
\end{proof}
\vspace{6ex}


\begin{lemma}\label{B.3-1}
Suppose Assumptions \ref{3-4}-\ref{3-6}, \ref{4-5}-\ref{4-6} hold. Let $\lambda_{1} \asymp \left(\sqrt{\frac{\log p}{n}}+\sqrt{\frac{\log p}{m}} \right)$, $\lambda_{1}\asymp\lambda_{2}\asymp\lambda^{\prime}=o(1)$ and $s=s^{\prime}=O(1)$. We have
\begin{align*}
& \big\| \nabla^{2}_{\alpha\vbeta}\mathcal{L}_{n}^{m}(\widehat{\vthe})-\bm{w}^{*\top}\nabla^{2}_{\vbeta\vbeta}\mathcal{L}_{n}^{m}(\widehat{\vthe}) \big\|_{\infty}=O_{p}\left( s \left( \sqrt{\frac{\log p}{n}}+\sqrt{\frac{\log p}{m}} + \frac{\log p}{m}\right) \right), \\
& \big\| \nabla^{2}_{\alpha\vbeta}\mathcal{L}_{n}^{m}(\widehat{\vthe}) - \widehat{\bm{w}}^{\top}\nabla^{2}_{\vbeta\vbeta}\mathcal{L}_{n}^{m}(\widehat{\vthe}) \big\|_{\infty}=O_{p}\left( (s+s^{\prime}) \left(\sqrt{\frac{\log p}{n}}+\sqrt{\frac{\log p}{m}} + \frac{\log p}{m}\right) \right).
\end{align*}
\end{lemma}
\begin{proof}
We only prove the first claim, the same techniques can be applied to prove the second claim.
\begin{equation*}
	\big\|\nabla^{2}_{\alpha\vbeta}\mathcal{L}_{n}^{m}(\widehat{\vthe}) -\widehat{\bm{w}}^{\top}\nabla^{2}_{\vbeta\vbeta}\mathcal{L}_{n}^{m}(\widehat{\vthe}) \big\|_{\infty}
	\le \big\|\nabla^{2}_{\alpha\vbeta}\mathcal{L}_{n}^{m}(\widehat{\vthe})-\bm{w}^{*\top}\nabla^{2}_{\vbeta\vbeta}\mathcal{L}_{n}^{m}(\widehat{\vthe}) \big\|_{\infty}+\|\widehat{\bm{w}}-\bm{w}^{*}\|_{1}\|\nabla^{2}_{\vbeta\vbeta}\mathcal{L}_{n}^{m}(\widehat{\vthe}) \|_{\infty}.
\end{equation*}
Using the fact $H^{*}_{\alpha\vbeta}=\bm{w}^{*\top}H^{*}_{\vbeta\vbeta}$ and by the triangle inequality, we have
\begin{align*}
\big\| \nabla^{2}_{\alpha\vbeta}\mathcal{L}_{n}^{m}(\widehat{\vthe})-\bm{w}^{*\top}\nabla^{2}_{\vbeta\vbeta}\mathcal{L}_{n}^{m}(\widehat{\vthe}) \big\|_{\infty}
\le&  \big\| \nabla^{2}_{\alpha\vbeta}\mathcal{L}_{n}^{m}(\widehat{\vthe})-\nabla^{2}_{\alpha\vbeta}\mathcal{L}_{n}^{m}(\vthe^{*}) \big\|_{\infty} + \big\| \nabla^{2}_{\alpha\vbeta}\mathcal{L}_{n}^{m}(\vthe^{*})-H^{*}_{\alpha\vbeta} \big\|_{\infty}+ \\
&+ \big\| \bm{w}^{*\top}(\nabla^{2}_{\vbeta\vbeta}\mathcal{L}_{n}^{m}(\widehat{\vthe})-\nabla^{2}_{\vbeta\vbeta}\mathcal{L}_{n}^{m}(\vthe^{*})) \big\|_{\infty} +
\big\| \bm{w}^{*\top}(\nabla^{2}_{\vbeta\vbeta}\mathcal{L}_{n}^{m}(\vthe^{*})-H^{*}_{\vbeta}) \big\|_{\infty}\\
=:& J_{1}+J_{2}+J_{3}+J_{4}.
\end{align*}
Since $\big\| \nabla^{2}_{\alpha\vbeta}\mathcal{L}_{n}^{m}(\vthe^{*})-H^{*}_{\alpha\vbeta} \big\|_{\infty}\le \| \nabla^{2}\mathcal{L}_{n}^{m}(\vthe^{*})-H^{*} \|_{\infty}$,
we have $ J_{2}=O_{p}\left( \sqrt{\frac{\log p}{m}} + \frac{\log p}{m}\right)$
by \eqref{3-11}.
Since $\|\bm{w}^{*}\|_{\infty}\le D$ and $\|\bm{w}^{*}\|_{0}=s^{\prime}=O(1)$,
\begin{equation*}
\big\| \bm{w}^{*\top}(\nabla^{2}_{\vbeta\vbeta}\mathcal{L}_{n}^{m}(\vthe^{*})-H^{*}_{\vbeta}) \big\|_{\infty}\le s^{\prime}D \| \nabla^{2}\mathcal{L}_{n}^{m}(\vthe^{*})-H^{*} \|_{\infty}\lesssim \| \nabla^{2}\mathcal{L}_{n}^{m}(\vthe^{*})-H^{*} \|_{\infty},
\end{equation*} 
we have $J_{4}=O_{p}\left( s \left(\sqrt{\frac{\log p}{m}} + \frac{\log p}{m} \right) \right)$
by \eqref{3-11} and $s=O(1)$. It remains to prove the rate of $J_1$ and $J_3$.
Note that
\begin{align*}
\nabla^{2}_{\alpha\vbeta,j}\mathcal{L}_{n}^{m}(\widehat{\vthe})&=\frac{\sum_{i=1}^{m}w_{i}(\widehat{\vthe})(\varphi_{\alpha}(\Y_{i})-\overline{\varphi}_{\alpha})(\varphi_{\vbeta,j}(\Y_{i})-\overline{\varphi}_{\vbeta,j})}{\sum_{i=1}^{m}w_{i}(\widehat{\vthe})}=\frac{\sum_{i=1}^{m}w_{i}\e^{g_{i}}a_{i}b_{i j}}{\sum_{i=1}^{m}w_{i}\e^{g_{i}}},\\
\nabla^{2}_{\alpha\vbeta,j}\mathcal{L}_{n}^{m}(\vthe^{*})&=\frac{\sum_{i=1}^{m}w_{i}(\vthe^{*})(\varphi_{\alpha}(\Y_{i})-\overline{\varphi}_{\alpha})(\varphi_{\vbeta,j}(\Y_{i})-\overline{\varphi}_{\vbeta,j})}{\sum_{i=1}^{m}w_{i}(\vthe^{*})}=\frac{\sum_{i=1}^{m}w_{i}a_{i}b_{i j}}{\sum_{i=1}^{m}w_{i}},
\end{align*}
where $a_{i}=\varphi_{\alpha}(\Y_{i})-\overline{\varphi}_{\alpha}$, $b_{i j}=\varphi_{\vbeta,j}(\Y_{i})-\overline{\varphi}_{\vbeta,j}$, $w_{i}=w_{i}(\vthe^{*})$, and $g_{i}=\widecheck{\vthe}^{\top}\varphi(\Y_{i})$ as in the proof of Lemma \ref{4-4}.
Then, by the similar proof in Lemma \ref{4-4}, we have
\begin{align*}
\big| \nabla^{2}_{\alpha\vbeta,j}\mathcal{L}_{n}^{m}(\widehat{\vthe})-\nabla^{2}_{\alpha\vbeta,j}\mathcal{L}_{n}^{m}(\vthe^{*}) \big|
& =\left| \frac{\sum_{i=1}^{m}w_{i} \e^{g_{i}}a_{i}b_{i j}}{\sum_{i=1}^{m}w_{i} \e^{g_{i}}}-\frac{\sum_{i=1}^{m}w_{i}a_{i}b_{i j}}{\sum_{i=1}^{m}w_{i}} \right| \\
&\le \left| \frac{\sum_{i=1}^{m}w_{i}(\e^{g_{i}}-1)a_{i}b_{i j}}{\sum_{i=1}^{m}w_{i} \e^{g_{i}}} \right| + \left| \left(\sum_{i=1}^{m}w_{i}a_{i}b_{i j} \right)\left( \frac{1}{\sum_{i=1}^{m}w_{i}\e^{g_{i}}}-\frac{1}{\sum_{i=1}^{m}w_{i}} \right) \right|  \\
&= \left| \frac{\sum_{i=1}^{m}w_{i}(\e^{g_{i}}-1)a_{i}b_{i j}}{\sum_{i=1}^{m}w_{i}\e^{g_{i}}} \right| + \left| \frac{\left( \sum_{i=1}^{m}w_{i}a_{i}b_{i j} \right) \left( \sum_{i=1}^{m}w_{i}(\e^{g_{i}}-1) \right)}{\left( \sum_{i=1}^{m}w_{i}\e^{g_{i}} \right) \left( \sum_{i=1}^{m}w_{i} \right)} \right| \\
&\lesssim \|\widecheck{\vthe} \|_{1} = O_{p}\left( s\left(\sqrt{\frac{\log p}{n}}+\sqrt{\frac{\log p}{m}} + \frac{\log p}{m}\right) \right),
\end{align*}
which implies $J_1 = O_{p}\left( s\left(\sqrt{\frac{\log p}{n}}+\sqrt{\frac{\log p}{m}} + \frac{\log p}{m}\right) \right)$. Similarly, we have $J_3 = O_{p}\bigg( s\bigg(\sqrt{\frac{\log p}{n}}+\sqrt{\frac{\log p}{m}} + \frac{\log p}{m}\bigg) \bigg)$.
\end{proof}
\vspace{6ex}


\noindent \textbf{Proof of Theorem \ref{4-11}:}\label{App-B.4}

\begin{proof}
We begin the proof by decomposing the $\widehat{U}(\alpha_{0},\widehat{\vbeta})$ as
\begin{align*}
& \widehat{U}(\alpha_{0},\widehat{\vbeta}) =\nabla_{\alpha}\mathcal{L}_{n}^{m}(\alpha_{0},\widehat{\vbeta})-\widehat{\bm{w}}^{\top}\nabla_{\vbeta}\mathcal{L}_{n}^{m}(\alpha_{0},\widehat{\vbeta})  \\
= & \nabla_{\alpha}\mathcal{L}_{n}^{m}(\alpha_{0},\vbeta^{*}) + \big[\nabla^{2}_{\alpha\vbeta}\mathcal{L}_{n}^{m}(\alpha_{0},\vbeta_{1}) \big]^{\top}(\widehat{\vbeta}-\vbeta^{*})-\widehat{\bm{w}}^{\top}\nabla_{\vbeta}\mathcal{L}_{n}^{m}(\alpha_{0},\vbeta^{*})-\widehat{\bm{w}}^{\top}\nabla^{2}_{\vbeta\vbeta}\mathcal{L}_{n}^{m}(\alpha_{0},\vbeta_{2})(\widehat{\vbeta}-\vbeta^{*}) \\
= & \nabla_{\alpha}\mathcal{L}_{n}^{m}(\alpha_{0},\vbeta^{*})-\bm{w}^{*\top}\nabla_{\vbeta}\mathcal{L}_{n}^{m}(\alpha_{0},\vbeta^{*})  + (\bm{w}^{*}-\widehat{\bm{w}})^{\top}\nabla_{\vbeta}\mathcal{L}_{n}^{m}(\alpha_{0},\vbeta^{*}) \\
&~~~~ + (\widehat{\vbeta}-\vbeta^{*})^{\top}\big(\nabla^{2}_{\alpha\vbeta}\mathcal{L}_{n}^{m}(\alpha_{0},\vbeta_{1})-\nabla^{2}_{\vbeta\vbeta}\mathcal{L}_{n}^{m}(\alpha_{0},\vbeta_{2})\bm{w}^{*} \big) + (\bm{w}^{*}-\widehat{\bm{w}})^{\top}\nabla^{2}_{\vbeta\vbeta}\mathcal{L}_{n}^{m}(\alpha_{0},\vbeta_{2})(\widehat{\vbeta}-\vbeta^{*}) \\
=: & E_{1}+E_{2}+E_{3}+E_{4},
\end{align*}
where  $\vbeta_{1}=u_{1}(\widehat{\vbeta}-\vbeta^{*})+\vbeta^{*}$ and $\vbeta_{2}=u_{2}(\widehat{\vbeta}-\vbeta^{*})+\vbeta^{*}$ for $u_{1},u_{2}\in [0,1]$ by Taylor's expansion. by Lemma \ref{4-1}, we have $\bm{v}^{\top}H^{*}\bm{v}=H^{*}_{\alpha|\vbeta}$ and thus $\sqrt{n}E_{1} \, \rightsquigarrow \, N(0,H^{*}_{\alpha|\vbeta})$ for term $E_{1}$, where $\bm{v}=(1,-\bm{w}^{*\top})^{\top}$. 
For the term $E_{2}$, by the Lemma \ref{4-4} and \eqref{3-11}, we have
\begin{align*}
\sqrt{n}|E_{2}|\le \sqrt{n}\|\bm{w}^{*}-\widehat{\bm{w}}\|_{1}\|\nabla_{\vbeta}\mathcal{L}_{n}^{m}(\alpha_{0},\vbeta^{*})\|_{\infty}  =O_{p}\left((s+s^{\prime})\sqrt{n}\left(\sqrt{\frac{\log p}{n}}+\sqrt{\frac{\log p}{m}} + \frac{\log p}{m}\right)^{2}\right),
\end{align*}
By assumptions about the relations among $m$, $n$ and $p$ the assumption $\frac{m}{n} \gtrsim \log p$, we have
\begin{equation}\label{B.5-1}
    \begin{aligned}
        \sqrt{n}\left(\sqrt{\frac{\log p}{n}}+\sqrt{\frac{\log p}{m}} + \frac{\log p}{m}\right) & = \sqrt{\log p}+\sqrt{\frac{n\log p}{m}} + \frac{\sqrt{n} \log p}{m}  \lesssim \sqrt{\log p} + \sqrt{\frac{n\log p}{n\log p}} +  \frac{\sqrt{n} \log p}{n\log p} \lesssim \sqrt{\log p},
    \end{aligned}
\end{equation}
and then,
\begin{equation*}
   (s+s^{\prime})\sqrt{n}\left(\sqrt{\frac{\log p}{n}}+\sqrt{\frac{\log p}{m}} + \frac{\log p}{m}\right)^{2} \lesssim \frac{\log p}{\sqrt{n}}+\frac{\log p}{\sqrt{m}} + \frac{\log^{3 / 2} p}{m}.
\end{equation*}
Therefore, $\sqrt{n}|E_{2}|=O_{p}\left( \frac{\log p}{\sqrt{n}}+\frac{\log p}{\sqrt{m}} + \frac{\log^{3 / 2} p}{m} \right) =o_{p}(1)$.
For the term $E_{3}$, we have
\begin{equation*}
    \begin{aligned}
       \sqrt{n} |E_{3}| & \le \sqrt{n} \big\| \widehat{\vbeta}-\vbeta^{*} \big\|_{1} \big\| \nabla^{2}_{\alpha\vbeta}\mathcal{L}_{n}^{m}(\alpha_{0},\vbeta_{1}) -\nabla^{2}_{\vbeta\vbeta}\mathcal{L}_{n}^{m}(\alpha_{0},\vbeta_{2})\bm{w}^{*} \big\|_{\infty} \\
        & = \sqrt{n} O_{p}\left(s\left(\sqrt{\frac{\log p}{n}}+\sqrt{\frac{\log p}{m}} + \frac{\log p}{m} \right)\right) \cdot O_{p}\left( s\left( \sqrt{\frac{\log p}{n}}+\sqrt{\frac{\log p}{m}} + \frac{\log p}{m}\right) \right) = o_p(1)
    \end{aligned}
\end{equation*}
with the same proof for $E_{2}$ and the results in Corollary \ref{3-14}.
Finally, for the term $E_{4}$, we have
\begin{align*}
    \big\| \nabla^{2}_{\vbeta\vbeta}\mathcal{L}_{n}^{m}(\alpha_{0},\vbeta_{2}) \big\|_{\infty}
    &\le \|H^{*}_{\vbeta\vbeta}\|_{\infty}+\| \nabla^{2}_{\vbeta\vbeta}\mathcal{L}_{n}^{m}(\alpha_{0},\vbeta_{2})-H^{*}_{\vbeta\vbeta} \|_{\infty} \\
    &\le \lambda_{\max} + \| \nabla^{2}_{\vbeta\vbeta}\mathcal{L}_{n}^{m}(\alpha_{0},\vbeta_{2})-H^{*}_{\vbeta\vbeta} \|_{\infty} \\
    &\le \lambda_{\max} +\|\nabla^{2}_{\vbeta\vbeta}\mathcal{L}_{n}^{m}(\alpha_{0},\vbeta_{2})-\nabla^{2}_{\vbeta\vbeta}\mathcal{L}_{n}^{m}(\vthe^{*})\|_{\infty}+\|\nabla^{2}_{\vbeta\vbeta}\mathcal{L}_{n}^{m}(\vthe^{*})-H^{*}_{\vbeta\vbeta}\|_{\infty}.
\end{align*}
By the same proofs for the terms $J_{1}$ and $J_{2}$ in the proof of Lemma \ref{B.3-1}, we have
\begin{align*}
    \| \nabla^{2}_{\vbeta\vbeta}\mathcal{L}_{n}^{m}(\alpha_{0},\vbeta_{2})-\nabla^{2}_{\vbeta\vbeta}\mathcal{L}_{n}^{m}(\vthe^{*}) \|_{\infty}\lesssim \|\widehat{\vbeta}-\vbeta^{*}\|_{1} &=o_{p}(1),\\
    \|\nabla^{2}_{\vbeta\vbeta}\mathcal{L}_{n}^{m}(\vthe^{*})-H^{*}_{\vbeta\vbeta}\|_{\infty} = O_{p}\left( s\left( \sqrt{\frac{\log p}{n}}+\sqrt{\frac{\log p}{m}} + \frac{\log p}{m}\right) \right)&=o_{p}(1).
\end{align*}
Finally, $|E_{4}|\le\|\nabla^{2}_{\vbeta\vbeta}\mathcal{L}_{n}^{m}(\alpha_{0},\vbeta_{2})\|_{\infty}\|\bm{w}^{*}-\widehat{\bm{w}}\|_{1}\|\widehat{\vbeta}-\vbeta^{*}\|_{1}$ and thus by the similar techniques in $E_{2}$,
\[
\begin{aligned}
\sqrt{n}|E_{4}|
&\le \sqrt{n}\|\nabla^{2}_{\vbeta\vbeta}\mathcal{L}_{n}^{m}(\alpha_{0},\vbeta_{2})\|_{\infty}\|\bm{w}^{*}-\widehat{\bm{w}}\|_{1}\|\widehat{\vbeta}-\vbeta^{*}\|_{1} \lesssim \sqrt{n}\|\bm{w}^{*}-\widehat{\bm{w}}\|_{1}\|\widehat{\vbeta}-\vbeta^{*}\|_{1}=o_{p}(1).
\end{aligned}
\]
Combining the above results, we get the desired asymptotic normality $\sqrt{n}\widehat{U}(\alpha_{0},\widehat{\vbeta}) = \sqrt{n} E_1 + o_p(1) \, \rightsquigarrow \, \mathcal{N}(0, H_{\alpha \mid \vbeta}^*).$
\end{proof}
\vspace{6ex}

\noindent \textbf{Proof of Lemma \ref{4-10}:}\label{App-B.5}

\begin{proof}
By the definition of $\widehat{H}_{\alpha|\vbeta}$ and $H^{*}_{\alpha|\vbeta}$, 
\begin{align*}
|\widehat{H}_{\alpha|\vbeta}-H^{*}_{\alpha|\vbeta}|
&\le |\nabla^{2}_{\alpha\alpha}\mathcal{L}_{n}^{m}(\widehat{\vthe})-H^{*}_{\alpha\alpha}| + |(\widehat{\bm{w}}-\bm{w}^{*})^{\top}H^{*}_{\alpha\vbeta}| + \big|\widehat{\bm{w}}^{\top}\big(\nabla^{2}_{\alpha\vbeta}\mathcal{L}_{n}^{m}(\widehat{\vthe})-H^{*}_{\alpha\vbeta}\big)\big|  \\
&=:F_{1}+F_{2}+F_{3}.
\end{align*}
For $F_{1}$, by the same argument for the term $J_{1}$ in the proof of Lemma \ref{B.3-1} and \eqref{3-11}, we have
\[
    \begin{aligned}
        F_{1} & \le | \nabla^{2}_{\alpha\alpha}\mathcal{L}_{n}^{m}(\widehat{\vthe})-\nabla^{2}_{\alpha\alpha}\mathcal{L}_{n}^{m}(\vthe^{*})| + |\nabla^{2}_{\alpha\alpha}\mathcal{L}_{n}^{m}(\vthe^{*})-H^{*}_{\alpha\alpha} | \\
        & \lesssim \| \widecheck{\vthe}\|_1 + O_p \left( s \left( \sqrt{\frac{\log p}{n}}+\sqrt{\frac{\log p}{m}} + \frac{\log p}{m}\right) \right) \asymp O_p \left( s \left( \sqrt{\frac{\log p}{n}}+\sqrt{\frac{\log p}{m}} + \frac{\log p}{m}\right) \right).
    \end{aligned}
\]
Next, for $F_{2}$, by the Assumption \ref{4-5} and Lemma \ref{4-4},
\[
     F_{2}\le \|\widehat{\bm{w}}-\bm{w}^{*}\|_{1}\|H^{*}\|_{\infty} = O_p \left( (s+s^{\prime})\left(\sqrt{\frac{\log p}{n}}+\sqrt{\frac{\log p}{m}} + \frac{\log p}{m}\right) \right).
\]
Finally, for $F_{3}$, note that 
$
    \|\widehat{\bm{w}}\|_{1} \le \|\widehat{\bm{w}}-\bm{w}^{*}\|_{1} + \|\bm{w}^{*}\|_{1} \leq s' D + o_p(1),
$
we have
\[
    \begin{aligned}
        F_{3} &\le \|\widehat{\bm{w}}\|_{1} \big\| \nabla^{2}_{\alpha\vbeta}\mathcal{L}_{n}^{m}(\widehat{\vthe})-H^{*}_{\alpha\vbeta} \big\|_{\infty} \\
        & \lesssim \big\| \nabla^{2}_{\alpha\vbeta}\mathcal{L}_{n}^{m}(\widehat{\vthe})-\nabla^{2}_{\alpha\vbeta}\mathcal{L}_{n}^{m}(\vthe^{*}) \big\|_{\infty} + \big\| \nabla^{2}_{\alpha\vbeta}\mathcal{L}_{n}^{m}(\vthe^{*})-H^{*}_{\alpha\vbeta} \big\|_{\infty} \\
        & = O_p \left( (s+s^{\prime})\left(\sqrt{\frac{\log p}{n}}+\sqrt{\frac{\log p}{m}} + \frac{\log p}{m}\right) \right).
    \end{aligned}
\]
\end{proof}
\vspace{6ex}

\noindent \textbf{Proof of the Theorem \ref{5-2}:}\label{App-C}
\begin{proof}
According to the definition of $\widetilde{\alpha}$ in \eqref{one_step_est}, we have
\begin{align*}
\widetilde{\alpha}-\alpha^{*}
&=\widehat{\alpha}-\alpha^{*}-H^{*-1}_{\alpha|\vbeta}\widehat{U}(\widehat{\alpha},\widehat{\vbeta})+\widehat{U}(\widehat{\alpha},\widehat{\vbeta})\left(H^{*-1}_{\alpha|\vbeta}-\left( \frac{\partial \widehat{U}(\widehat{\alpha},\widehat{\vbeta})}{\partial\alpha} \right)^{-1}\right) \\
&=\widehat{\alpha}-\alpha^{*}-H^{*-1}_{\alpha|\vbeta}\left(\widehat{U}( \alpha^{*},\widehat{\vbeta})+\left( \widehat{\alpha}-\alpha^{*} \right)\frac{\partial \widehat{U}(\overline{\alpha},\widehat{\vbeta})}{\partial\alpha} \right)+\widehat{U}(\widehat{\alpha},\widehat{\vbeta})\left(H^{*-1}_{\alpha|\vbeta}-\left( \frac{\partial \widehat{U}(\widehat{\alpha},\widehat{\vbeta})}{\partial\alpha} \right)^{-1}\right) \\
&=-H^{*-1}_{\alpha|\vbeta}\widehat{U}( \alpha^{*},\widehat{\vbeta})+(\widehat{\alpha}-\alpha^{*})H^{*-1}_{\alpha|\vbeta}\left(H^{*}_{\alpha|\vbeta}-\frac{\partial \widehat{U}(\overline{\alpha},\widehat{\vbeta})}{\partial\alpha} \right)+\\
&~~~~~~~~~+\widehat{U}(\alpha^{*},\widehat{\vbeta})\left(H^{*-1}_{\alpha|\vbeta}-\left( \frac{\partial \widehat{U}(\widehat{\alpha},\widehat{\vbeta})}{\partial\alpha} \right)^{-1}\right)+\left( \widehat{\alpha}-\alpha^{*} \right)\frac{\partial \widehat{U}(\overline{\alpha},\widehat{\vbeta})}{\partial\alpha}\left(H^{*-1}_{\alpha|\vbeta}-\left( \frac{\partial \widehat{U}(\widehat{\alpha},\widehat{\vbeta})}{\partial\alpha} \right)^{-1}\right)  \\
&=:E_{1}+E_{2}+E_{3}+E_{4},
\end{align*}
where $\overline{\alpha}=u(\widehat{\alpha}-\alpha^{*})+\alpha^{*}$ for some $u\in[0,1]$. By Theorem \ref{4-11}, it is easy to verify that $\sqrt{n}E_{1} \, \rightsquigarrow \, \mathcal{N}(0,H^{*-1}_{\alpha|\vbeta})$. It remains to show that $E_2 + E_3 + E_4 = o_p(1 / \sqrt{n})$. Indeed, by the Lemma \ref{4-10},
\[
    H^{*}_{\alpha|\vbeta}-\frac{\partial \widehat{U}(\widehat{\alpha},\widehat{\vbeta})}{\partial\alpha}= H^{*}_{\alpha|\vbeta} - \widehat{H}_{\alpha|\vbeta} = O_{p}\left((s+s^{\prime})\left(\sqrt{\frac{\log p}{n}}+\sqrt{\frac{\log p}{m}} + \frac{\log p}{m}\right)\right).
\]
Thus, 
\[
    \begin{aligned}
        E_3 & = \widehat{U}(\alpha^{*},\widehat{\vbeta})\left(H^{*-1}_{\alpha|\vbeta}-\left( \frac{\partial \widehat{U}(\widehat{\alpha},\widehat{\vbeta})}{\partial\alpha} \right)^{-1}\right) = O_p \left( \frac{1}{\sqrt{n}}\right) O_{p}\left((s+s^{\prime})\left(\sqrt{\frac{\log p}{n}}+\sqrt{\frac{\log p}{m}} + \frac{\log p}{m}\right)\right) = o_p \big(1 / \sqrt{n} \big).
    \end{aligned}
\]
Similarly, by Corollary \ref{3-14},
\[
    \begin{aligned}
        |E_{2}| &\lesssim |\widehat{\alpha}-\alpha^{*}|\left| H^{*}_{\alpha|\vbeta}-\frac{\partial \widehat{U}(\overline{\alpha},\widehat{\vbeta})}{\partial\alpha} \right| \le \|\widehat{\vthe}-\vthe^{*}\|_{1} \left| H^{*}_{\alpha|\vbeta}-\frac{\partial \widehat{U}(\overline{\alpha},\widehat{\vbeta})}{\partial\alpha} \right|  = O_{p}\left(s(s+s^{\prime})\left(\sqrt{\frac{\log p}{n}}+\sqrt{\frac{\log p}{m}} + \frac{\log p}{m}\right)^{2}\right) = o_p \big(1 / \sqrt{n} \big),
    \end{aligned}
\]
by the assumptions for $m,n$, and $p$.
Finally, for the term $E_{4}$, with the same proof of the Lemma \ref{4-10}, one can show that 
\[
    \frac{\partial \widehat{U}(\overline{\alpha},\widehat{\vbeta})}{\partial\alpha}-\widehat{H}_{\alpha|\vbeta}=O_{p}((1-u)(\widehat{\alpha}-\alpha^{*}))=o_{p}(1)
\]
and then $\frac{\partial \widehat{U}(\overline{\alpha},\widehat{\vbeta})}{\partial\alpha} = O_p(1)$. Thus,
\[
    \begin{aligned}
        |E_{4}| &\lesssim |\widehat{\alpha}-\alpha^{*}|\left| H^{*-1}_{\alpha|\vbeta}-\left(\frac{\partial \widehat{U}(\widehat{\alpha},\widehat{\vbeta})}{\partial\alpha}\right)^{-1} \right| \le  \|\widehat{\vthe}-\vthe^{*}\|_{1} \left| H^{*-1}_{\alpha|\vbeta}-\left(\frac{\partial \widehat{U}(\widehat{\alpha},\widehat{\vbeta})}{\partial\alpha}\right)^{-1} \right|\\
        &=O_{p}\left(s(s+s^{\prime})\left(\sqrt{\frac{\log p}{n}}+\sqrt{\frac{\log p}{m}} + \frac{\log p}{m}\right)^{2}\right).
    \end{aligned}
\]
which leads to the desired result.
\end{proof}

\section{Proofs of Theorem and Lemmas in Section \ref{sec-6}}\label{App-D}

We first need the following lemma stating the rate of correlation between any elements in $\widehat{\vthe}$. Denote $S_0 := \{ j \in [p] : \theta_j^* = 0\}$.

\begin{lemma}\label{lem6-1}

    Suppose the assumptions in Theorem \ref{5-2} hold. For any $j ,k \in S_0$, we have $\cov(\widehat{\theta}_j, \widehat{\theta}_k) = O(n^{-1})$, where  $\widehat{\vthe}$ is the elastic-net estimator in \eqref{2-8}.  
\end{lemma}

\begin{proof}
    By Cauchy-Schwartz inequality, it suffices to show that $\operatorname{var}(\widehat{\theta}_j) = O(n^{-1})$ for any $j \in S_0$. Define event $\mathscr{A} := \big\{ \|\widehat{\vthe} - \vthe^* \|_1 \leq \e(\zeta + 1) s \lambda_1 / C_{\min}\big\}$.
    By $\theta^*_j = 0$ and the exchangeable property in Assumption \ref{ass6}, on the event $\mathscr{A}$, we have
    \begin{align*}
             \E \big[\widehat{\theta}_j \mid \mathscr{A}\big] &\leq  \E \big[ |\widehat{\theta}_j - \theta^*_j| \, | \, \mathscr{A} \big]  \leq \frac{1}{p} \E \big[ \| \widehat{\vthe} - \vthe\|_1 \, | \, \mathscr{A}\big] \leq \frac{\e (\zeta + 1) s \lambda_1}{p C_{\min}},\\
            \var \big( \widehat{\theta}_j \, | \, \mathscr{A}\big) &\leq \E \big[ \widehat{\theta}_j^2 \, | \, \mathscr{A}\big]  = \E \big[ (\widehat{\theta}_j - \theta^*_j)^2 \, | \, \mathscr{A}\big] \leq \frac{1}{p} \E \big[ \| \widehat{\vthe} - \vthe^* \|_1^2 \, | \, \mathscr{A}\big] \leq \frac{\e^2 (\zeta + 1)^2 s^2 \lambda_1}{p C^2_{\min}}.
    \end{align*}
    By Theorem \ref{thm-1}, we know that $1 - \pr (\mathscr{A}) \leq \delta_1 + \delta_2 \to 0$. Note that $\| \vthe^*\|_{\infty} \leq B$, then 
    \begin{equation*}
        \begin{aligned}
            \var \big(\widehat{\theta}_j\big) & 
            \leq \frac{\e^2 (\zeta + 1)^2 s^2 \lambda_1^2}{p C_{\min}^2} + \left( \frac{\e (\zeta + 1) s \lambda_1}{p C_{\min}} \pr (\mathscr{A}) + \| \vthe^*\|_{\infty} \big( 1 - \pr (\mathscr{A}) \big) \right)^2\\
            & \lesssim \frac{s^2 \lambda_1^2}{p} \asymp \frac{1}{p}\left(\sqrt{\frac{\log p}{n}}+\sqrt{\frac{\log p}{m}} + \frac{\log p}{m}\right)^2.
        \end{aligned}
    \end{equation*}
    Finally, by using the same argument in \eqref{B.5-1}, we conclude that 
    \[
        \begin{aligned}
            \var \big(\widehat{\theta}_j\big) & \lesssim \frac{1}{p}\left(\sqrt{\frac{\log p}{n}}+\sqrt{\frac{\log p}{m}} + \frac{\log p}{m}\right)^2 \\
            & \lesssim \frac{1}{p} \sqrt{\frac{\log p}{n}} \left(\sqrt{\frac{\log p}{n}}+\sqrt{\frac{\log p}{m}} + \frac{\log p}{m}\right) \\
            &\lesssim \frac{1}{n} + \frac{1}{p} \sqrt{\frac{\log p}{n}} \left(\sqrt{\frac{\log p}{m}} + \frac{\log p}{m}\right) \lesssim  \frac{1}{n},
        \end{aligned}
    \]
    where the last step is by $\frac{m}{n} \gtrsim \log p$. Thus, we complete the proof of the lemma.
\end{proof}
\vspace{6ex}

\noindent \textbf{Proof of the Theorem \ref{thm6-1}}\label{App-D.2}

\begin{proof}
    Since $M_j$ is naturally symmetric about $0$ for any $j \in S_0$, from Proposition 1 in \cite{dai2022false}, we only need to verify that there exists some constants $C > 0$ and $\alpha \in (0, 2)$ such that
    $
        \operatorname{var} \left( \sum_{j \in S_0} \mathds{1}(M_j > t)\right) \leq C p_0^{\alpha},
    $
    for any $t \in \mathbb{R}$, where $p_0 = |S_0|$. 
    It is suffices to show
    \begin{equation}\label{D-1}
        \sup_{t \in \mathbb{R}} \operatorname{var} \left( \frac{1}{p_0}\sum_{j \in S_0} \mathds{1}(M_j > t)\right) \to 0.
    \end{equation}
    Indeed, we decompose this quantity as
    \begin{equation}\label{D-2}
        \begin{aligned}
            \var \left( \frac{1}{p_0}\sum_{j \in S_0} \mathds{1}(M_j > t)\right) & \leq \max_{(j, k) \in S_0 \times S_0 } \big| \mathrm{P} (M_j > t) \mathrm{P} (M_k > t) - H^2(t) \big| + \frac{1}{p_0^2} \sum_{j, k \in  S_0 \times S_0}\big| \mathrm{P} (M_j > t, M_k > t) - H^2(t) \big|,
        \end{aligned}
    \end{equation}
    where $H(t) = \mathrm{P} \big( \operatorname{sgn} (Z_1 Z_2) f(Z_1, Z_2) > t\big)$ with $Z_1$ and $Z_2$ are independent standard normal distribution. 
    From Theorem \ref{5-2}, for whole-data based statistic $T_j$, denoted by $\widetilde{T}_j := \widetilde{\theta}_j \sqrt{n H_{j | -j}^*}$, we have
    \begin{equation*}
        \begin{aligned}
            \sup_{t \in \mathbb{R}} \big| \mathrm{P} (T_j \leq t) - \Phi(t)\big| & \leq \sup_{t \in \mathbb{R}}\big| \mathrm{P} (T_j \leq t) - \mathrm{P} (\widetilde{T}_j \leq t) \big| +  \sup_{t \in \mathbb{R}}\big| \mathrm{P} (\widetilde{T}_j \leq t) - \Phi(t)\big| \\
            & = \sup_{t \in \mathbb{R}} \mathrm{E} \big[ \mathrm{P} (T_j \leq t \, | \, \widehat{H}_{j | -j}) - \mathrm{P} (\widetilde{T}_j \leq t \, | \, \widehat{H}_{j | -j})\big] + o(1)\\
            & = \sup_{t \in \mathbb{R}} \left| \mathrm{E} \int_t^{t\sqrt{{H_{j | -j}^*}/{\widehat{H}_{j | -j}}}} \, \phi(x) \, dx \right| + o(1)\\
            & \leq \sup_{t \in \mathbb{R}} \mathrm{E} \left| t \big( \sqrt{{H_{j | -j}^*}/{\widehat{H}_{j | -j}}} - 1\big)\right| + o(1) \\
            & = \lim_{n \rightarrow \infty} \sup_{t \in [-n, n]} |t| \mathrm{E} \left|\big( \sqrt{{H_{j | -j}^*}/{\widehat{H}_{j | -j}}} - 1\big)\right| + o(1) = 0,
        \end{aligned}
    \end{equation*}
    where $\phi(\cdot)$ is the density of standard normal density and the last equality holds by the uniform integrability of $\widehat{H}_{j | -j}$. By Lemma A.5 in \cite{dai2022false}, we have $\sup_{t \in \mathbb{R}, \, j \in S_0} |\mathrm{P} (M_j > t) - H(t)| \longrightarrow 0$, which implies
    \begin{equation*}
        \sup_{t \in \mathbb{R}} \max_{(j, k) \in S_0 \times S_0} \big| \mathrm{P} (M_j > t) \mathrm{P} (M_k > t) - H^2(t) \big| \, \to \, 0.
    \end{equation*}
    For the last term in \eqref{D-2}, we make the following decomposition
    \begin{equation*}
        \begin{aligned} 
            \mathrm{P}(M_{j}>t, M_{k}>t) &=\mathrm{P} \left(T_{j}^{(2)}>I_{t}(T_{j}^{(1)}), \,  T_{k}^{(2)}>I_{t}(T_{k}^{(1)}), \, T_{i}^{(1)}>0, \,  T_{j}^{(1)}>0\right) \\ & ~~~~~~ +\mathrm{P}\left(T_{j}^{(2)}>I_{t}(T_{j}^{(1)}), \,  T_{k}^{(1)}<-I_{t}(T_{j}^{(1)}), \, T_{j}^{(1)}>0, \,  T_{k}^{(1)}<0\right) \\ &~~~~~~ +\mathrm{P}\left(T_{j}^{(2)}<-I_{t}(T_{j}^{(1)}), \,  T_{k}^{(2)}>I_{t}(T_{k}^{(1)}), \, T_{j}^{(1)}<0, \, T_{k}^{(1)}>0\right) \\ &~~~~~~ +\mathrm{P}\left(T_{j}^{(2)}<-I_{t}(T_{i}^{(1)}), \,  T_{j}^{(2)}<-I_{t}(T_{j}^{(1)}), \, T_{j}^{(1)}<0, \,  T_{k}^{(1)}<0\right) \\ 
            &:=D_{1}+D_{2}+D_{3}+D_{4},
        \end{aligned}
    \end{equation*}
    where $I_t(v) = \inf\{u \geq 0 \, : \, f(u, v) > t\}$. For $D_1$, 
    \begin{equation*}
        \begin{aligned}
            D_1 & = \mathrm{E} \left[ \mathrm{P} \big( T_j^{(2)} > I_t(x), \, T_k^{(2)} > I_t(y)\big) \, \big| \, T_j^{(1)} = x, \, T_k^{(1)} = y \right] \\
            & = \mathrm{E} \left[ \mathrm{P} \big( \widetilde{T}_j^{(2)} > I_t(x), \, \widetilde{T}_k^{(2)} > I_t(y)\big) \, \big| \, T_j^{(1)} = x, \, T_k^{(1)} = y \right] + o(1) \\
            & \leq \mathrm{E} \left[ Q(I_t(x)) Q(I_t(y)) \, \big| \, T_j^{(1)} = x, \, T_k^{(1)} = y \right] + c_1 \big|\operatorname{cov} (\widetilde{T}_j^{(2)}, \widetilde{T}_k^{(2)})\big| + o(1),
        \end{aligned}
    \end{equation*}
    where $Q(t) = 1 - \Phi(t)$, $c_1$ is some positive number, and the last ``$\leq$" follows from Mehler's identity and Lemma 1 in \cite{azriel2015}. For the zero mean bivariate normal distribution function $\Phi_{\rho}(t_1, t_2)$ with covariance matrix $\left[ \begin{matrix} 1 & \rho \\ \rho & 1 \end{matrix} \right]$, Mehler's identity ensures that
    \begin{equation*}
        \begin{aligned}
            \Phi_{\rho}(t_1, t_2) & = \Phi(t_1) \Phi(t_2) + \sum_{n = 1}^{\infty} \frac{\rho^n}{n!} \phi^{(n - 1)} (t_1) \phi^{(n - 1)} (t_2) \\
            & = \Phi(t_1) \Phi(t_2) + \rho \phi(t_1) \phi(t_2) + \sum_{n = 2}^{\infty} \frac{\rho^n}{n!} \phi^{(n - 1)} (t_1) \phi^{(n - 1)} (t_2) \\
            & \leq \Phi(t_1) \Phi(t_2) + \rho \phi(t_1) \phi(t_2) + c_2 \rho \leq \Phi(t_1) \Phi(t_2) + (1 / (2 \pi) + c_2) \rho.
        \end{aligned} 
    \end{equation*}
    where the last inequality holds due to 
   $
        \sum_{n = 2}^{\infty} \frac{[\sup_{t \in \mathbb{R}} \phi^{(n - 1)} (t)]^2}{n!} < \infty
   $
    by Lemma 1 in \cite{azriel2015}. 
    We first deal with the covariance $\operatorname{cov} (\widetilde{T}_j^{(2)}, \widetilde{T}_k^{(2)})$. Denote $\nabla_j = \nabla_{\theta_j}$,
    by the definition of the one-step estimator 
    \begin{equation}\label{D-3}
        \begin{aligned}
            \operatorname{cov} (\widetilde{\theta}_j, \widetilde{\theta}_k) & = \operatorname{cov} \Big( \widehat{\theta}_j - \big( \nabla_j \, \widehat{U} (\widehat{\theta}_j, \widehat{\vthe}_{-j})\big)^{-1} \widehat{U} (\widehat{\vthe}), \ \widehat{\theta}_k - \big( \nabla_k \, \widehat{U} (\widehat{\theta}_k, \widehat{\vthe}_{-k})\big)^{-1} \widehat{U} (\widehat{\vthe}) \Big) \\
            & = \cov(\widehat{\theta}_j, \widehat{\theta}_k) - \operatorname{cov} \Big( \widehat{\theta}_j, \ \big( \nabla_k \, \widehat{U} (\widehat{\theta}_k, \widehat{\vthe}_{-k})\big)^{-1} \widehat{U} (\widehat{\vthe}) \Big) - \operatorname{cov} \Big( \big( \nabla_j \, \widehat{U} (\widehat{\theta}_j, \widehat{\vthe}_{-j})\big)^{-1} \widehat{U} (\widehat{\vthe}), \ \widehat{\theta}_k \Big) \\
            & \qquad + \operatorname{cov} \Big( \big( \nabla_j \, \widehat{U} (\widehat{\theta}_j, \widehat{\vthe}_{-j})\big)^{-1} \widehat{U} (\widehat{\vthe}), \ \big( \nabla_k \, \widehat{U} (\widehat{\theta}_k, \widehat{\vthe}_{-k})\big)^{-1} \widehat{U} (\widehat{\vthe}) \Big).
        \end{aligned}
    \end{equation}
    In the above decomposition, the first term is $O(n^{-1})$ by Lemma \ref{lem6-1}, thus it is sufficient to show the last term above is $O(n^{-1})$. Note that 
    $
        \widehat{U} (\widehat{\vthe}) = \widehat{U} (\theta_j^*, \widehat{\vthe}_{-j}) + \nabla_j \widehat{U} (\overline{\theta}_j, \widehat{\vthe}_{-j}) (\widehat{\theta}_j - \theta_j^*),
    $
    where $\overline{\theta}_j$ lies between $\theta_j$ and $\widehat{\theta}_j$. By the similar proof in Theorem \ref{5-2}, one can show that $\nabla_j \widehat{U} (\widehat{\theta}_j, \widehat{\vthe}_{-j}) = H_{j | -j}^* + o_p(1)$ and $\nabla_j \widehat{U} ({\theta}_j^*, \widehat{\vthe}_{-j}) = \nabla_j \widehat{U} (\widehat{\theta}_j, \widehat{\vthe}_{-j}) + O_p (|\widehat{\theta}_j - \theta_j^* |) = H_{j | -j}^* + o_p(1)$. Thus, by the continuity and Assumption \ref{4-5}, we know that
    $
        \nabla_j \widehat{U} (\cdot, \widehat{\vthe}_{-j}) \in [c_2, c_3]$,

    where $c_2$ and $c_3$ are some positive number independent with $n$, $m$, and $p$.
    Then, 
    \begin{equation*}
        \begin{aligned}
            & \operatorname{cov} \Big( \big( \nabla_j \, \widehat{U} (\widehat{\theta}_j, \widehat{\vthe}_{-j})\big)^{-1} \widehat{U} (\widehat{\vthe}), \ \big( \nabla_k \, \widehat{U} (\widehat{\theta}_k, \widehat{\vthe}_{-k})\big)^{-1} \widehat{U} (\widehat{\vthe}) \Big) \\
            \leq &  c_2^{-2} \operatorname{cov} \big( \widehat{U} (\widehat{\theta}_j, \widehat{\vthe}_{-j}), \ \widehat{U} (\widehat{\theta}_k, \widehat{\vthe}_{-k}) \big) \\
            = & c_2^{-2} \operatorname{cov} \big( \widehat{U} (\theta_j^*, \widehat{\vthe}_{-j}) + \nabla_j \widehat{U}(\overline{\theta}_j, \widehat{\vthe}_{-j}) (\widehat{\theta}_j - \theta_j^*), \ \widehat{U} (\theta_k^*, \widehat{\vthe}_{-k}) + \nabla_k \widehat{U}(\overline{\theta}_k, \widehat{\vthe}_{-k}) (\widehat{\theta}_k - \theta_k^*) \big) \\
            \lesssim & \operatorname{cov} \big( \widehat{U}(\theta_j^*, \widehat{\vthe}_{-j}), \ \widehat{U}(\theta_k^*, \widehat{\vthe}_{-k}) \big) + \operatorname{cov} (\widehat{\theta}_j, \widehat{\theta}_k) + \operatorname{cov} \big( \widehat{U}(\theta_j^*, \widehat{\vthe}_{-j}),  \widehat{\theta}_k \big).
        \end{aligned}
    \end{equation*}
    Theorem \ref{4-11} implies that the first term is $O\left( n^{-1} \sqrt{H_{j | -j}^*  H_{k | -k}^*}\right)$,  Lemma \ref{lem6-1} ensures the second term above is $O(n^{-1})$, and the Cauchy's inequality gives the third term above is also $O(n^{-1})$.
    Hence, for $D_1$, we have
    \begin{equation*}
        D_1 \leq \mathrm{E} \left[ Q(I_t(x)) Q(I_t(y)) \, \big| \, T_j^{(1)} = x, \, T_k^{(1)} = y \right] + c_4 \sqrt{H_{j | -j}^*  H_{k | -k}^*} + o(1),
    \end{equation*}
    where $c_4$ is some positive number free of $n, m, p$. 
    By applying the similar techniques for $D_2$, $D_3$, and $D_4$, we get an upper bound on $\mathrm{P} (M_i > t, M_j > t)$ specified as below
    \begin{equation*}
        \mathrm{P} \big( \operatorname{sgn} (Z_j^{(2)} T_j^{(1)}) f(Z_j^{(2)} T_j^{(1)}) > t, \operatorname{sgn} (Z_k^{(2)} T_k^{(1)}) f(Z_k^{(2)} T_k^{(1)}) > t \big) + c_5 \sqrt{H_{j | -j}^*  H_{k | -k}^*} + o(1)
    \end{equation*}
    with some positive $c_5$, in which $Z_j^{(2)}$ and $Z_k^{(2)}$ are two independent random variables following the standard normal distribution. By conditioning on the signs of $Z_j^{(2)}$ and $Z_k^{(2)}$ and decomposing the above display as previously, one can show 
    $
        \mathrm{P} (M_j > t, M_k > t) \leq H^2(t) + c_6 \sqrt{H_{j | -j}^*  H_{k | -k}^*} + o(1)
    $
    with $c_6 > 0$. 
    Therefore, 
    \begin{equation*}
        \begin{aligned}
            \frac{1}{p_0^2} \sum_{j, k \in S_0} \big|  \mathrm{P} (M_j > t, M_k > t) - H^2(t) \big| & \leq  \frac{c_6}{p_0^2} \sum_{j, k \in S_0} \sqrt{H_{j | -j}^*  H_{k | -k}^*} + o(1) \\
            & \leq \frac{c_6}{p_0^2} \sum_{j, k \in S_0} \sqrt{H_{jj}^*  H_{kk}^*} + o(1 ) \\
            & \leq \frac{c_6}{p_0^2} p_0 \lambda_{\max}(H^*) + o(1) 
        \end{aligned}
    \end{equation*}
    uniformly on $t \in \mathbb{R}$ by Assumption \ref{ass6}. As a result, \ref{D-1} is valid, and we complete the proof.
\end{proof}
\vspace{6ex}

\noindent \textbf{Proof of the Theorem \ref{thm6-2}:}\label{App-D.3}

\begin{proof}
    By Theorem \ref{thm6-1}, we know
         $\mathop{\lim \sup}_{n, p \rightarrow \infty} \sum_{j \in S_0} I_j \leq q$.
    Therefore, $\mathop{\lim \sup}_{n, p \rightarrow \infty} \sum_{k \in S_1} I_k \geq 1- q$, which implies
    \begin{equation*}
        \mathop{\lim \sup}_{n, p \rightarrow \infty} \max_{k \in S_1} I_k \geq \frac{1 - q}{s}.
    \end{equation*}
    And for any $\alpha \in (0, 1)$, note that $p_0 \rightarrow \infty$, we have $(1 - q) / s > \alpha / p_0 $. Then
    \begin{equation*}
        \begin{aligned}
            \mathop{\lim \sup}_{n, p \rightarrow \infty} \frac{1}{s} \sum_{k \in S_1} \mathds{1} \bigg( I_k \leq \frac{\alpha}{p_0}\bigg) &= \mathop{\lim \sup}_{n, p \rightarrow \infty} \frac{1}{s} \mathds{1} \Big( \max_{k \in S_1} I_k \leq \frac{\alpha}{p_0}\Big) \leq \mathop{\lim \sup}_{n, p \rightarrow \infty} \frac{1}{s} \mathds{1} \Big( \max_{k \in S_1} I_k < \frac{1 - q }{s}\Big) = 0.
        \end{aligned}
    \end{equation*}
    So {rank faithfulness} in \cite{dai2022false} holds. Finally, since for any $j \in S_0$, $M_j$ has the same limit distribution, then
    $
        j \in \widehat{S} \, \Leftrightarrow \, M_j > \tau_q
   $
    shares with the same law by Assumption \ref{ass6}, the {null exchangeability} in \cite{dai2022false} are also satisfied. By Proposition 2 in \cite{dai2022false}, we obtain this theorem. 
\end{proof}
\vspace{6ex}

\noindent \textbf{Proof of the Theorem \ref{6-3}:}\label{App-D}

\begin{proof}
Let $N_{1}(t)$ be the number of true null hypotheses with e-values larger than or equal to $t$ and $N_{2}(t)$ be the number of all hypotheses with e-values larger than or equal to $t$. Define $t_{\epsilon}:=\inf \left\{ t\ge 0: \frac{N_{2}(t)t}{p} \ge \frac{1}{\epsilon}\right\}.$
Note that
$
    \frac{k e_{(k)}}{p}=\frac{N_{2}(e_{(k)})e_{(k)}}{p},
$
we reject all hypotheses with e-values larger than or equal to $t_{\epsilon}$, from which the FDR can be written as
\begin{equation*}
    \fdr=\E\left( \frac{N_{1}(t_{\epsilon})}{N_{2}(t_{\epsilon})}\right)=  \frac{\epsilon}{p}\cdot \E(t_{\epsilon}N_{1}(t_{\epsilon})),
\end{equation*}
where $N_{2}(t_{\epsilon})=p/ \epsilon t_{\epsilon}$. Note 
    $N_{1}(t_{\epsilon})=\sum_{k\in \mathscr{N}} 1(E_{k} \ge t_{\epsilon}),
$
thus we have
$
    \E(t_{\epsilon}N_{1}(t_{\epsilon})) = \sum_{k\in \mathscr{N}} \E(t_{\epsilon} 1(E_{k}\ge t_{\epsilon}) ) \le \sum_{k\in \mathscr{N}} \E(E_{k}),
$
where the inequality hold due to $t_{\epsilon} 1(E_{k}\ge t_{\epsilon}) \le E_{k}$ a.s.. Now, we bound the FDR by
\begin{equation}\label{App-D-1}
    \fdr \le \epsilon \cdot \frac{|\mathscr{N}|}{p} \cdot\frac{1}{|\mathscr{N}|}\sum_{k\in \mathcal{N}} \E(E_{k}) \le \epsilon\cdot \frac{1}{|\mathscr{N}|}\sum_{k\in \mathscr{N}} \E(E_{k}).
\end{equation}
By taking the limit superior on both sides of \eqref{App-D-1} and Assumption \ref{6-1}, we conclude the result.
\end{proof}

\end{appendix}

\end{document}